%% file: Main.tex
\documentclass[10pt,twocolumn,letterpaper]{article}

\usepackage[review]{cvpr}      

\usepackage{graphicx}
\usepackage{amsmath}
\usepackage{amssymb}
\usepackage{booktabs}

\usepackage[pagebackref,breaklinks,colorlinks]{hyperref}

\usepackage[capitalize]{cleveref}
\crefname{section}{Sec.}{Secs.}
\Crefname{section}{Section}{Sections}
\Crefname{table}{Table}{Tables}
\crefname{table}{Tab.}{Tabs.}

\usepackage[symbols,xindy,toc,acronym,section=section]{glossaries}
\usepackage{glossary-longragged}
\usepackage[symbols]{glossaries-extra}
\usepackage{amsthm}
\usepackage[abs]{overpic}
\usepackage[utf8]{inputenc}
\usepackage[T1]{fontenc}
\usepackage{changepage}
\usepackage{ragged2e, lipsum}
\usepackage{tabularx, caption, booktabs}
\usepackage{float}
\usepackage{multirow}

\usepackage[dvipsnames]{xcolor}

\loadglsentries{acronyms}

\usepackage{enumitem}

\newlist{inlineroman}{enumerate*}{1}
\setlist[inlineroman]{itemjoin*={{, and }},afterlabel=~,label=\roman*)}

\newcommand{\y}[1]{\gls{#1}}

\DeclareMathOperator{\Tr}{tr}
\DeclareMathOperator*{\argmin}{\arg\!\min}

\newtheorem{lemma}{Lemma}

\newtheorem{prop}{Proposition}
\usepackage{mathtools}

\begin{document}

\title{Convex Relaxations for Isometric and Equiareal NRS\textit{f}M}

\author{Agniva Sengupta,  Adrien Bartoli\\
\footnotesize{EnCoV, TGI, Institut Pascal, UMR6602 CNRS, Universit\'e Clermont Auvergne, France}\\
{\textit{\footnotesize agniva.sengupta@uca.fr}}, {\textit{\footnotesize adrien.bartoli@gmail.com}}
}
\maketitle
\allowdisplaybreaks
\begin{abstract}
Extensible objects form a challenging case for \y{nrsfm}, owing to the lack of a sufficiently constrained extensible model of the point-cloud. We tackle the challenge by proposing 1) convex relaxations of the isometric model up to quasi-isometry, and 2) convex relaxations involving the equiareal deformation model, which preserves local area and has not been used in \y{nrsfm}.
The equiareal model is appealing because it is physically plausible and widely applicable. 
However, it has two main difficulties: first, when used on its own, it is ambiguous, and second, it involves quartic, hence highly nonconvex, constraints.
Our approach handles the first difficulty by mixing the equiareal with the isometric model and the second difficulty by new convex relaxations. 
We validate our methods on multiple real and synthetic data, including well-known benchmarks.
\end{abstract}

\section{Introduction}\label{sec:intro}
\input{sections/introduction}


\section{Background}\label{sec:background}
\input{sections/background}

\section{Methods}\label{sec:method}
\input{sections/method}

\section{Experimental Results}\label{sec:exp}
\input{sections/experiments}

\section{Conclusion}\label{sec:conclusion}
\input{sections/conclusion}

\section*{Acknowledgement}
We thank the authors of \cite{casillas2019equiareal} for helping us out by sharing and explaining their dataset to us. This project has received funding from the European Union’s Horizon 2020 research and innovation programme under grant agreement No 863146.

{\small
\bibliographystyle{ieee_fullname}
\bibliography{Main}
}

\end{document}


\title{Supplementary Materials}
\maketitle
\allowdisplaybreaks
\glsunset{sl} \glsunset{nrsfm} \glsunset{dsl} \glsunset{pp} \glsunset{mdh} \glsunset{snr} \glsunset{qnr} \glsunset{hnr} \glsunsetall


\begin{table*}[t]
\centering  
\begin{tabular}{>{\centering\arraybackslash}m{2cm} >{\centering\arraybackslash}m{6cm} >{\centering\arraybackslash}m{3cm} >{\centering\arraybackslash}m{4cm}}
    \toprule
  &   \textbf{Problem formulation}    &    \textbf{Assumptions} & \textbf{Existing convex methods}\\ \cline{2-4}
Isometric \y{nrsfm}  
        &   $\begin{gathered}
 \min \quad 0, \\ \text{s.t.:} \mathfrak{g}_{I}(i,j,q) = \mathcal{G}_{2}(j,q) \\
 \forall i \in [1,n], j \in [1,m], (j,q) \in \mathcal{E}_2
 \end{gathered}$
                        &  \center Euclidean distances approximate  geodesic distances       & -                \\ \hline
Isometric \y{nrsfm} with \y{mdh}  
        &   $\begin{gathered}
 \min \quad - \sum_{1,j=1}^{n,m}\delta_{i,j}, \\ \text{s.t.:} \mathfrak{g}_{I}(i,j,q) = \mathcal{G}_{2}(j,q) \\
 \forall i \in [1,n], j \in [1,m], (j,q) \in \mathcal{E}_2
 \end{gathered}$
                        &  \center Euclidean distances approximate  geodesic distances, \y{mdh}       & -                \\ \hline
Inextensible \y{nrsfm} with \y{mdh}  
        &   $\begin{gathered}
 \min \quad - \sum_{1,j=1}^{n,m}\delta_{i,j}, \\ \text{s.t.:} \mathfrak{g}_{I}(i,j,q) \leq \mathcal{G}_{2}(j,q) \\
 \forall i \in [1,n], j \in [1,m], (j,q) \in \mathcal{E}_2
 \end{gathered}$
                        &  \center Euclidean distances approximate  geodesic distances, \y{mdh}, inextensibility       & \cite{chhatkuli2017inextensible} \cite{ji2017maximizing}                \\
    \bottomrule
\end{tabular}
\caption{Hierarchy of isometric \y{nrsfm} problems. The solutions we provide to \y{snr} and \y{qnr} are reformulations of `isometric \y{nrsfm} with \y{mdh}', while purely `isometric \y{nrsfm}' is uninteresting, since it does not have an unique solution}
\label{tab_summary}
\end{table*}

We describe additional details of our article in this supplementary material. To recap, we had provided convex solutions to isometric and quasi-isometric \y{nrsfm} in the main article, as well as a convex solution for quasi-equiareal \y{nrsfm}. Table~(\ref{tab_summary}) summarizes the position of our contributions w.r.t existing \y{nrsfm} state-of-the-art.

Following up from our main article, we provide the following: section~(\ref{sec_1}) provides the proof of Lemma 1, section~(\ref{sec:quartic}) provides the details of the quartics involved in area computation, section~(\ref{sec_area_gram}) provides details of how to express these quartics as linear combination of elements of the Gram matrices $\{\mathfrak{T}_i\}$ and $\{\mathfrak{U}_i\}$, section~(\ref{sec_acc}) provides strategies for accelerating the \y{pp} solution for \y{hnr}, section~(\ref{sec_sim}) details the methodology for generating the synthetic data used in our experiments, and section~(\ref{sec_exp}) provides additional experimental results, both qualitative and quantitative.

\section{Proof of Lemma 1}\label{sec_1}
\begin{figure}[t]
    \centering
    \begin{overpic}[width=\columnwidth, trim=0 0 0 0,clip]{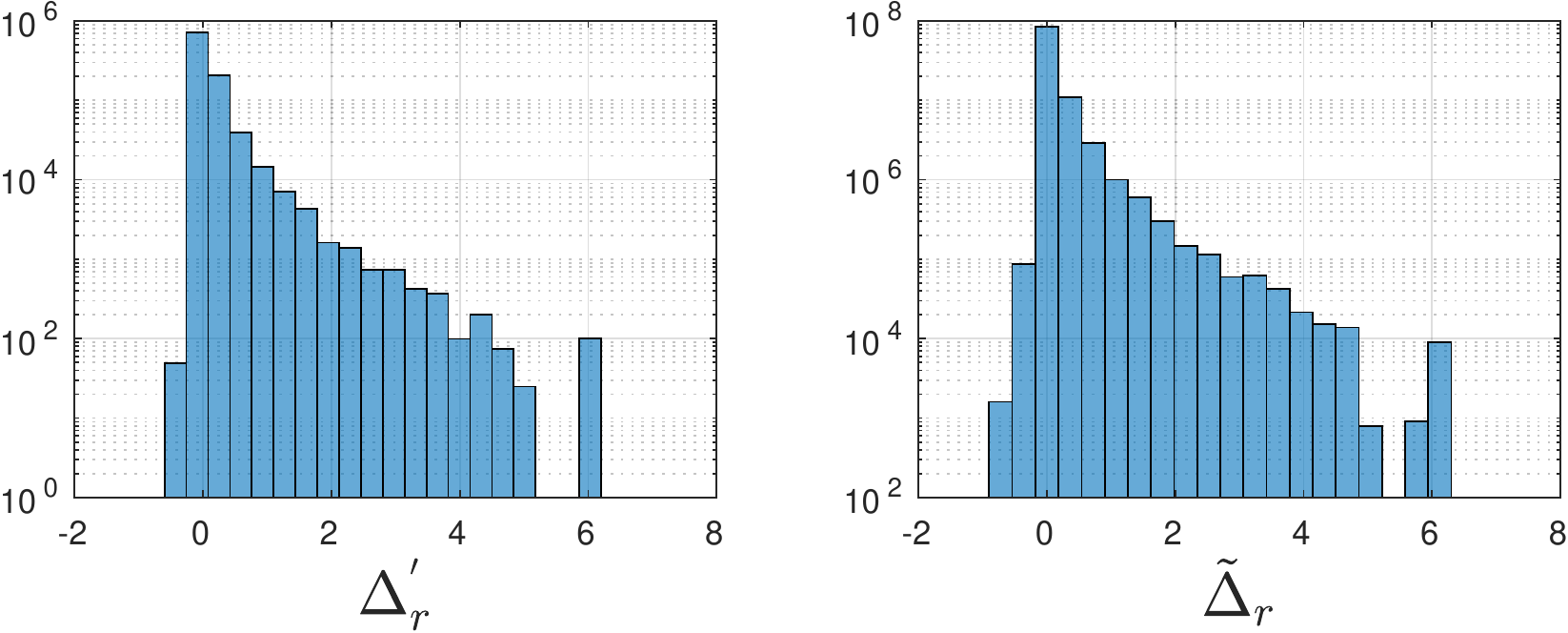}
    \end{overpic}
    \caption{\textit{left:} the discriminant $\Delta_r^{'}$ of the solution for $\delta_{i,r}^{'}$ from equation~(\ref{eqn:first_expr}), sampled over 1 million times with randomly drawn 3D vertices forming a triangle with $h_1 \in [0, 0.1]$, \textit{right}: the discriminant $\tilde{\Delta}_r$ of the solution for $\tilde{\delta}_{i,r}$ from equation~(\ref{eqn:second_expr}), sampled over 1 million times with similarly drawn 3D vertices with $(h_1, h_2) \in [0, 0.1]$, where the average edge length of all the triangles is $\sim 0.6$ units. $(\Delta_r^{'}, \tilde{\Delta}_r)$ are positive, ensuring real solution, for a clear majority of cases }
    \label{fig:lem1}
\end{figure}
A 2-simplex with vertices lying strictly on three non-degenerate \y{sl}s with an additional equiareality constraint has 2 degrees of freedom. Therefore, with 3 independent variables (using \y{dsl} for the simplest parameterisation), we have the following minimal case: 4 vertices forming a fully connected graph, i.e., $|\mathcal{E}_{2}| = 5$ and $|\mathcal{E}_{3}| = 3$ observed across 3 images gives us 12 equations for 12 variables, hence \textit{solvable} in theory. However, lemma~1 posits that such a solution, if it exists, is non-unique for the \y{nrsfm} problem. 

The squared area of a 2-simplex enclosed by $(\mathbf{P}_{i,j}, \mathbf{P}_{i,q}, \mathbf{P}_{i,r})$ is given by the element-wise expansion of the vector product of any two edges of the triangle formed by those three vertices, given by:
\begin{gather}\label{eqn:area_dsl}
    \mathfrak{h}_E^{\delta}(i,j,q,r) = \Big(\mathrm{Area}(\mathbf{P}_{i,j}, \mathbf{P}_{i,q}, \mathbf{P}_{i,r})\Big)^2 = \notag\\ \Big( \frac{1}{2} \| (\delta_{i,j}\mathfrak{d}_{i,j} - \delta_{i,q}\mathfrak{d}_{i,q}) \times (\delta_{i,r}\mathfrak{d}_{i,r} - \delta_{i,q}\mathfrak{d}_{i,q}) \| \Big)^2 \notag\\
    = \frac{1}{4}\bigg(\Big(G_1^{i,j,q,r} \delta_{i,q}^2 + G_2^{i,j,q,r} \delta_{i,r} \delta_{i,q} + G_3^{i,j,q,r} \delta_{i,r}^2\Big) \delta_{i,j}^2 + \notag\\ \Big( G_4^{i,j,q,r} \delta_{i,r} \delta_{i,q}^2 + G_5^{i,j,q,r} \delta_{i,r}^2 \delta_{i,q} \Big) \delta_{i,j} + G_6^{i,j,q,r} \delta_{i,q}^2 \delta_{i,r}^2\bigg),
\end{gather}
where $G_1^{i,j,q,r}, \hdots, G_6^{i,j,q,r}$ are data-terms derived from the \y{sl}s (check section~\ref{sec:quartic} for details).

The proof of lemma~1 follows.
\begin{proof}
Given the triangle enclosed by $(\mathbf{P}_{i,j}, \mathbf{P}_{i,q}, \mathbf{P}_{i,r})$, we claim the following:

\textit{Ansatz 1}: for some small displacement $h_1$ applied to $\delta_{i,j}$, there exists a corrective displacement $\delta_{i,r}^{'}$ such that the triangles at depth $(\delta_{i,j}, \delta_{i,q}, \delta_{i,r})$ and $(\delta_{i,j}+h_1, \delta_{i,q}, \delta_{i,r}^{'})$ of \y{sl}s $(\mathfrak{d}_{i,j},\mathfrak{d}_{i,q},\mathfrak{d}_{i,r})$ have the same area. We prove this by equating the squared area of the two triangles, giving us:
    \begin{gather}
        \bigg(\Big(G_1^{i,j,q,r} \delta_{i,q}^2 + G_2^{i,j,q,r} \delta_{i,r} \delta_{i,q} + G_3^{i,j,q,r} \delta_{i,r}^2\Big) \delta_{i,j}^2 + \notag\\ \Big( G_4^{i,j,q,r} \delta_{i,r} \delta_{i,q}^2 + G_5^{i,j,q,r} \delta_{i,r}^2 \delta_{i,q} \Big) \delta_{i,j} + G_6^{i,j,q,r} \delta_{i,q}^2 \delta_{i,r}^2\bigg) = \notag\\ \bigg(\Big(G_1^{i,j,q,r} \delta_{i,q}^2 + G_2^{i,j,q,r} \delta_{i,r}^{'} \delta_{i,q} + G_3^{i,j,q,r} (\delta_{i,r}^{'})^2\Big) (\delta_{i,j} + h_1)^2 \notag\\ + \Big( G_4^{i,j,q,r} \delta_{i,r}^{'} \delta_{i,q}^2 + G_5^{i,j,q,r} (\delta_{i,r}^{'})^2 \delta_{i,q} \Big) (\delta_{i,j} + h_1) \notag\\+ G_6^{i,j,q,r} \delta_{i,q}^2 (\delta_{i,r}^{'})^2\bigg)
        \label{eqn:first_expr}
    \end{gather}
Solving for $\delta_{i,r}^{'}$ leads to a hexic discriminant\footnote{Maple scripts for generating the full solution of $\delta_{i,r}^{'}$ (and $\tilde{\delta}_{i,r}$ later) has been provided with the supplementary}. Hence existence of a real solution cannot be ensured analytically. However, we show with extensive exhaustive randomized simulations that for all small values of $h_1$, a real $\delta_{i,r}^{'}$ is obtainable in most cases. Check figure~\ref{fig:lem1} for details. 

\textit{Ansatz 2}: there exist small displacements $h_1$ and $h_2$ which, when applied to $\delta_{i,j}$ and $\delta_{i,q}$, can change the area of the triangle in $(\mathbf{P}_{i,j}, \mathbf{P}_{i,q}, \mathbf{P}_{i,r})$ but can be exactly compensated by a corrective displacement $\tilde{\delta}_{i,r}$. This is derived by similarly equating the two squared areas:
\begin{gather}\label{eqn:second_expr}
    \bigg(\Big(G_1^{i,j,q,r} \delta_{i,q}^2 + G_2^{i,j,q,r} \delta_{i,r} \delta_{i,q} + G_3^{i,j,q,r} \delta_{i,r}^2\Big) \delta_{i,j}^2 + \notag\\ \Big( G_4^{i,j,q,r} \delta_{i,r} \delta_{i,q}^2 + G_5^{i,j,q,r} \delta_{i,r}^2 \delta_{i,q} \Big) \delta_{i,j} + G_6^{i,j,q,r} \delta_{i,q}^2 \delta_{i,r}^2\bigg) = \notag\\ \bigg(\Big(G_1^{i,j,q,r} (\delta_{i,q} + h_2)^2 + G_2^{i,j,q,r} \tilde{\delta}_{i,r} (\delta_{i,q} + h_2) + G_3^{i,j,q,r} \tilde{\delta}_{i,r}^2\Big) \notag\\ (\delta_{i,j} + h_1)^2 + \Big( G_4^{i,j,q,r} \tilde{\delta}_{i,r} (\delta_{i,q} + h_2)^2 + G_5^{i,j,q,r} \tilde{\delta}_{i,r}^2 \notag\\ (\delta_{i,q} + h_2) \Big) (\delta_{i,j} + h_1) + G_6^{i,j,q,r} (\delta_{i,q} + h_2)^2 \tilde{\delta}_{i,r}^2\bigg)
        \label{eqn:second_expr}
\end{gather}
Solving for $\tilde{\delta}_{i,r}$ leads to another solution with a hexic discriminant. Again using randomized simulations, we show that there exist many combinations of $(h_1, h_2)$ with $(\mathbf{P}_{i,j}, \mathbf{P}_{i,q}, \mathbf{P}_{i,r})$ which leads to a real solution for $\tilde{\delta}_{i,r}$. as shown in figure~(\ref{fig:lem1}).

Combining \textit{ansatz 1} with \textit{ansatz 2}, we can say the following: given vertices $\mathbf{P}_{i,1}, \hdots, \mathbf{P}_{i,m}$, there can exist a locally acyclic graph structure on these vertices such that the change of area due to arbitrary displacement of one of the vertices can be compensated by some other displacement of some other vertices. Hence, the vertices are non-unique. Given that in NRSfM the areas are also unknown, not just the depth of the points, this makes purely equiareal NRSfM highly ambiguous.
\end{proof}

Such a proof precludes the possibility of defining specially constrained graph structures (e.g.: locally cyclic Whitney triangulation etc.) on the vertices to enforce unique solution, but such graph structures are unavailable for \y{nrsfm}, hence not considered.

\section{Quartics for Area Computation}\label{sec:quartic}
We now describe the details of $\mathfrak{h}_E^{\delta}$ and $\mathfrak{h}_E^{P}$, the squared area of 2-simplices in \y{dsl} and \y{pp} respectively. Beginning with $\mathfrak{h}_E^{\delta}$, which can be expressed by equation~(\ref{eqn:area_dsl}), where:
\begin{subequations}
    \begin{gather}
        G_1^{i,j,q,r} = (y_{i,q}^2 + z_{i,q}^2)x_{i,j}^2 - 2(y_{i,q} x_{i,q} y_{i,j} + z_{i,q} x_{i,q} z_{i,j}) x_{i,j} \notag \\+ (z_{i,q}^2 + x_{i,q}^2) y_{i,j}^2- 2 z_{i,q} y_{i,q} z_{i,j} y_{i,j} + (y_{i,q}^2 + x_{i,q}^2) z_{i,j}^2, \\
        G_2^{i,j,q,r} = -2(y_{i,q}y_{i,r} + z_{i,r}z_{i,q})x_{i,j}^2 - 2\Big((-x_{i,r}y_{i,q} \notag\\- x_{i,q}y_{i,r})y_{i,j} + (-x_{i,r}z_{i,q} - z_{i,r}x_{i,q})z_{i,j}\Big)x_{i,j} - 2(z_{i,r}z_{i,q} \notag\\+ x_{i,q}x_{i,r})y_{i,j}^2 + 2(z_{i,r}y_{i,q} + z_{i,q}y_{i,r})z_{i,j}y_{i,j} \notag\\- 2(y_{i,q}y_{i,r} + x_{i,q}x_{i,r})z_{i,j}^2, \\
        G_3^{i,j,q,r} = (y_{i,r}^2 + z_{i,r}^2)x_{i,j}^2 - 2(y_{i,r}x_{i,r}y_{i,j} + z_{i,r}x_{i,r}z_{i,j})x_{i,j} \notag\\+ (x_{i,r}^2 + z_{i,r}^2)y_{i,j}^2 - 2y_{i,j}z_{i,j}y_{i,r}z_{i,r} + z_{i,j}^2(x_{i,r}^2 + y_{i,r}^2), \\
    G_4^{i,j,q,r} = 2\Big((y_{i,q}y_{i,r} + z_{i,r}z_{i,q})x_{i,q} - x_{i,r}(y_{i,q}^2 \notag\\+ (z_{i,q})^2)\Big)x_{i,j} - 2(-z_{i,r}z_{i,q}y_{i,q} - x_{i,r}y_{i,q}x_{i,q} + z_{i,q}^2y_{i,r} \notag\\+ x_{i,q}^2y_{i,r})y_{i,j}- 2(y_{i,q}^2z_{i,r} - z_{i,q}y_{i,r}y_{i,q} - x_{i,r}z_{i,q}x_{i,q} \notag\\+ x_{i,q}^2z_{i,r})z_{i,j}, \\
        G_5^{i,j,q,r} = 2\Big(-(y_{i,r}^2 + z_{i,r}^2)x_{i,q} - x_{i,r}(-y_{i,q}y_{i,r} \notag\\- z_{i,r}z_{i,q})\Big)x_{i,j} - 2\Big(-x_{i,r}y_{i,r}x_{i,q} + (x_{i,r}^2 + z_{i,r}^2)y_{i,q} \notag\\- z_{i,q}y_{i,r}z_{i,r})y_{i,j} - 2(-x_{i,r}z_{i,r}x_{i,q} - y_{i,r}z_{i,r}y_{i,q} + z_{i,q}(x_{i,r}^2 \notag\\+ y_{i,r}^2)\Big)z_{i,j}, \\
        G_6^{i,j,q,r} = (y_{i,r}^2 + z_{i,r}^2)x_{i,q}^2 - 2x_{i,r}(y_{i,q}y_{i,r} + z_{i,r}z_{i,q})x_{i,q} \notag\\+ (x_{i,r}^2 + z_{i,r}^2)y_{i,q}^2 - 2z_{i,q}y_{i,r}z_{i,r}y_{i,q} + z_{i,q}^2(x_{i,r}^2 + y_{i,r}^2).
    \end{gather}
\end{subequations}
Similarly expanding for $\mathfrak{h}_E^{P}$, we have:
    \begin{gather}
        \mathfrak{h}_E^{P}(i,j,q,r) = \Big( \mathrm{Area}(\mathbf{P}_{i,j}, \mathbf{P}_{i,q}, \mathbf{P}_{i,r}) \Big)^2 \notag\\= \frac{1}{4} \Big(X_{i,j}^2Y_{i,q}^2 - 2X_{i,j}^2Y_{i,q}Y_{i,r} + X_{i,j}^2Y_{i,r}^2 + X_{i,j}^2Z_{i,q}^2 \notag\\- 2X_{i,j}^2Z_{i,q}Z_{i,r} +
        X_{i,j}^2Z_{i,r}^2 - 2X_{i,j}X_{i,q}Y_{i,j}Y_{i,q} \notag\\+ 2X_{i,j}X_{i,q}Y_{i,j}Y_{i,r} + 2X_{i,j}X_{i,q}Y_{i,q}Y_{i,r} - 2X_{i,j}X_{i,q}Y_{i,r}^2 -\notag\\
        2X_{i,j}X_{i,q}Z_{i,j}Z_{i,q} + 2X_{i,j}X_{i,q}Z_{i,j}Z_{i,r} + 2X_{i,j}X_{i,q}Z_{i,q}Z_{i,r} \notag\\- 2X_{i,j}X_{i,q}Z_{i,r}^2 + 2X_{i,j}X_{i,r}Y_{i,j}Y_{i,q} - 
        2X_{i,j}X_{i,r}Y_{i,j}Y_{i,r} \notag\\- 2X_{i,j}X_{i,r}Y_{i,q}^2 + 2X_{i,j}X_{i,r}Y_{i,q}Y_{i,r} + 2X_{i,j}X_{i,r}Z_{i,j}Z_{i,q} \notag\\- 2X_{i,j}X_{i,r}Z_{i,j}Z_{i,r} - 
        2X_{i,j}X_{i,r}Z_{i,q}^2 + 2X_{i,j}X_{i,r}Z_{i,q}Z_{i,r} \notag\\+ X_{i,q}^2Y_{i,j}^2 - 2X_{i,q}^2Y_{i,j}Y_{i,r} + X_{i,q}^2Y_{i,r}^2 \notag\\+ X_{i,q}^2Z_{i,j}^2 -
        2X_{i,q}^2Z_{i,j}Z_{i,r} + X_{i,q}^2Z_{i,r}^2 \notag\\- 2X_{i,q}X_{i,r}Y_{i,j}^2 + 2X_{i,q}X_{i,r}Y_{i,j}Y_{i,q} + 2X_{i,q}X_{i,r}Y_{i,j}Y_{i,r} \notag\\-
        2X_{i,q}X_{i,r}Y_{i,q}Y_{i,r} - 2X_{i,q}X_{i,r}Z_{i,j}^2 + 2X_{i,q}X_{i,r}Z_{i,j}Z_{i,q} +\notag\\ 2X_{i,q}X_{i,r}Z_{i,j}Z_{i,r} - 2X_{i,q}X_{i,r}Z_{i,q}Z_{i,r} +
        X_{i,r}^2Y_{i,j}^2 \notag\\- 2X_{i,r}^2Y_{i,j}Y_{i,q} + X_{i,r}^2Y_{i,q}^2 + X_{i,r}^2Z_{i,j}^2 \notag\\- 2X_{i,r}^2Z_{i,j}Z_{i,q} + X_{i,r}^2Z_{i,q}^2 +
        Y_{i,j}^2Z_{i,q}^2 \notag\\- 2Y_{i,j}^2Z_{i,q}Z_{i,r} +  Y_{i,j}^2Z_{i,r}^2 - 2Y_{i,j}Y_{i,q}Z_{i,j}Z_{i,q} \notag\\+ 2Y_{i,j}Y_{i,q}Z_{i,j}Z_{i,r} +
        2Y_{i,j}Y_{i,q}Z_{i,q}Z_{i,r} - 2Y_{i,j}Y_{i,q}Z_{i,r}^2 \notag\\+ 2Y_{i,j}Y_{i,r}Z_{i,j}Z_{i,q} - 2Y_{i,j}Y_{i,r}Z_{i,j}Z_{i,r} - 2Y_{i,j}Y_{i,r}Z_{i,q}^2 \notag\\+
        2Y_{i,j}Y_{i,r}Z_{i,q}Z_{i,r} + Y_{i,q}^2Z_{i,j}^2 - 2Y_{i,q}^2Z_{i,j}Z_{i,r} \notag\\+ Y_{i,q}^2Z_{i,r}^2 - 2Y_{i,q}Y_{i,r}Z_{i,j}^2 + 2Y_{i,q}Y_{i,r}Z_{i,j}Z_{i,q} \notag\\+ 
        2Y_{i,q}Y_{i,r}Z_{i,j}Z_{i,r} - 2Y_{i,q}Y_{i,r}Z_{i,q}Z_{i,r} + Y_{i,r}^2Z_{i,j}^2 \notag\\- 2Y_{i,r}^2Z_{i,j}Z_{i,q} + Y_{i,r}^2Z_{i,q}^2\Big). \label{eqn:area_hp}
    \end{gather}
\section{Area from Gram Matrices}\label{sec_area_gram}
We now describe the functions $\mathfrak{g}_E^{\delta}$ and $\mathfrak{g}_E^{P}$ that compute squared area from linear combination of the elements of $\mathfrak{T}_i$ and $\mathfrak{U}_i$ respectively. Beginning with $\mathfrak{g}_E^{\delta}$:
\begin{equation}
    \begin{gathered}
        \mathfrak{g}_E^{\delta}(i,j,q,r) = \frac{1}{4}\begin{pmatrix}G_1^{i,j,q,r}, \hdots,G_6^{i,j,q,r} \end{pmatrix}^{\top}\begin{pmatrix}\mathfrak{T}_{i,a_1,a_1}\\ \mathfrak{T}_{i,a_1,c_1}\\ \mathfrak{T}_{i,c_1,c_1}\\ \mathfrak{T}_{i,a_1,b_1} \\ \mathfrak{T}_{i,b_1,c_1}\\ \mathfrak{T}_{i,b_1,b_1}\end{pmatrix},
    \end{gathered}
\end{equation}
where $\mathfrak{E}_{i,a_1} = \delta_{i,j}\delta_{i,q}$, $\mathfrak{E}_{i,b_1} = \delta_{i,q}\delta_{i,r}$ and, $\mathfrak{E}_{i,c_1} = \delta_{i,j}\delta_{i,r}$.

Next, we describe $\mathfrak{g}_E^{P}$ as a linear combination of elements of $\mathfrak{U}_i$. Deriving from equation~(\ref{eqn:area_hp}):
    \begin{gather}
     \mathfrak{g}_E^{P}(i,j,q,r) = \frac{1}{4}\bigg(\mathfrak{U}_{i,d_1,d_1}  - 2\mathfrak{U}_{i,d_1,f_1}  + \mathfrak{U}_{i,f_1,f_1}  + \mathfrak{U}_{i,e_1,e_1} \notag\\ - 2\mathfrak{U}_{i,e_1,g_1}  +  \mathfrak{U}_{i,g_1,g_1}  - 2\mathfrak{U}_{i,d_1,h_1}  + 2\mathfrak{U}_{i,f_1,h_1} \notag\\  + 2\mathfrak{U}_{i,d_1,p_1}  - 2\mathfrak{U}_{i,f_1,p_1}  -   2\mathfrak{U}_{i,e_1,l_1} 
 + 2\mathfrak{U}_{i,g_1,l_1} \notag\\ + 2\mathfrak{U}_{i,e_1,q_1}  - 2\mathfrak{U}_{i,g_1,q_1}   + 2\mathfrak{U}_{i,d_1,j_1}  - 2\mathfrak{U}_{i,f_1,j_1}  - \notag\\  2\mathfrak{U}_{i,d_1,r_1}  + 2\mathfrak{U}_{i,f_1,r_1}  + 2\mathfrak{U}_{i,e_1,n_1}   - 2\mathfrak{U}_{i,g_1,n_1} \notag\\  - 2\mathfrak{U}_{i,e_1,t_1}  +   2\mathfrak{U}_{i,g_1,t_1} 
 + \mathfrak{U}_{i,h_1,h_1}  - 2\mathfrak{U}_{i,h_1,p_1} \notag\\ + \mathfrak{U}_{i,p_1,p_1}   + \mathfrak{U}_{i,l_1,l_1}  - 2\mathfrak{U}_{i,l_1,q_1}  +  \mathfrak{U}_{i,q_1,q_1}  - 2\mathfrak{U}_{i,h_1,j_1}  \notag\\ + 2\mathfrak{U}_{i,h_1,r_1}   + 2\mathfrak{U}_{i,j_1,p_1}  - 2\mathfrak{U}_{i,p_1,r_1}  -  2\mathfrak{U}_{i,l_1,n_1} \notag\\ 
 + 2\mathfrak{U}_{i,l_1,t_1}  + 2\mathfrak{U}_{i,n_1,q_1}  -  2\mathfrak{U}_{i,q_1,t_1}   + \mathfrak{U}_{i,j_1,j_1}  - \notag\\ 2\mathfrak{U}_{i,j_1,r_1}  +   \mathfrak{U}_{i,r_1,r_1}  + \mathfrak{U}_{i,n_1,n_1}  - 2\mathfrak{U}_{i,n_1,t_1}   + \mathfrak{U}_{i,t_1,t_1} \notag\\  + \mathfrak{U}_{i,i_1,i_1}  -  2\mathfrak{U}_{i,i_1,k_1} 
 + \mathfrak{U}_{i,k_1,k_1}  - 2\mathfrak{U}_{i,i_1,m_1}  + 2\mathfrak{U}_{i,k_1,m_1}  \notag\\  + 2\mathfrak{U}_{i,i_1,s_1}  - 2\mathfrak{U}_{i,k_1,s_1}  +  2\mathfrak{U}_{i,i_1,o_1}  - 2\mathfrak{U}_{i,k_1,o_1}  -  \notag\\ 2\mathfrak{U}_{i,i_1,u_1}   + 2\mathfrak{U}_{i,k_1,u_1}  + \mathfrak{U}_{i,m_1,m_1}  -  2\mathfrak{U}_{i,m_1,s_1} \notag\\ 
 + \mathfrak{U}_{i,s_1,s_1}  - 2\mathfrak{U}_{i,m_1,o_1}  + 2\mathfrak{U}_{i,m_1,u_1}   + 2\mathfrak{U}_{i,o_1,s_1} \notag\\  - 2\mathfrak{U}_{i,s_1,u_1}  +  \mathfrak{U}_{i,o_1,o_1}  - 2\mathfrak{U}_{i,o_1,u_1}  + \mathfrak{U}_{i,u_1,u_1} \bigg),
    \end{gather}
where $\mathfrak{Q}_{i,d_1} = X_{i,j} Y_{i,q}$, $\mathfrak{Q}_{i,f_1} = X_{i,j} Y_{i,r}$, and so on. Clearly, $(d_1,\hdots, u_1)$ is exactly mappable from the definition of $\mathfrak{Q}_i$ in equation~(15) of the main paper using the map $\vartheta(i,j,q)$.

\glsunset{hnr} \glsunset{pp}
\section{Accelerating \y{pp} Solution for \y{hnr}}\label{sec_acc}
We now present a method for accelerating the \y{hnr} solution in equation~(18) of the main paper. 

\textbf{Edge based relaxation}. If $\mathcal{E}_{3}$ does not represent all 2-simplices from a fully connected graph in the vertices of $\mathbf{X}_i$, $\{\mathfrak{U}_i\}$ is a partial \y{psd} matrix with a sparsely connected structure and standard \y{sdp} solvers do not guarantee a rank-1 \y{psd} completion of $\{\mathfrak{U}_i\}$ while solving equations~(18) \cite{wang2008further}. However, for our \y{nrsfm} problem, an edge in $\mathcal{E}_2$ cannot be connected beyond its local neighborhood, hence an exact rank-1 completion of $\{\mathfrak{U}_i\}$ is not mandatory. However, there exist principal submatrices of $\{\mathfrak{U}_i\}$ (in $\mathbf{S}^{18}_+$) corresponding to the connected 2-simplices for which an exact rank-1 completion is desirable. Therefore we can relax the \y{psd} constraint on the full matrices of $\{\mathfrak{U}_i\}$ in equation~(18) into multiple \y{psd} constraints of $\mathbf{S}^{18}_+$ each. This approach roughly corresponds to the so-called `edge based relaxation' techniques for \y{sdp} and allows us to utilise \y{hnr} for problems with a moderate number of point correspondences. Since \y{sdp} has an arithmetic operation complexity of $\mathcal{O}(n^3)$, splitting into smaller matrices accelerates the solution significantly. Concretely, we do this in two steps.

\textbf{Auxiliary variable}. As the \textit{first} step, we introduce the auxiliary variable $\mathfrak{V}_{i} = \mathfrak{Q}_i^{'}(\mathfrak{Q}_i^{'})^{\top}$, such that $\mathfrak{Q}_i^{'} = \Big(\vartheta(i,j,q)^{\top}, \vartheta(i,q,r)^{\top}, \vartheta(i,j,r)^{\top}, \hdots\Big)^{\top}$. $\mathfrak{V}_{i}$ allows us to split $\mathfrak{U}_i$ into: 1) $\mathfrak{V}_{i}$, and 2) $\mathfrak{S}_i^{'}$ (once again). Therefore, equation~(18) modifies to:
\begin{subequations}
    \label{eqn:hnr_xyz_sol_2}
    \begin{gather}
        \min_{\{\mathfrak{S}_{i}^{'}\}, \{\mathfrak{V}_{i}\},\mathcal{G}_{2},\mathcal{G}_{3}}   \sum_{i=1}^n\bigg( \Big(f^{\Pi}_{i,j}(\mathfrak{S}_{i}^{'}) + \Tr(\mathfrak{S}_{i}^{'}) + \Tr(\mathfrak{V}_{i})\Big) + \notag \\ \sum_{j,q=1}^{m,\mathcal{E}_{2}(j)} \Big( \lambda_I |\mathfrak{g}_I^{P}(i,j,q)  - \mathcal{G}_{2}(j,q)| -  f_{i,j}^{\delta}(\mathfrak{S}_{i}^{'})   + \notag \\
        \sum_{r=1}^{\mathcal{E}_{3}(j,q)}  \lambda_E|\mathfrak{g}_E^{P}(i,j,q,r) - \mathcal{G}_{3}(j,q,r)|\Big) \bigg)
        \label{eqn:hxs_1}\\
        \text{s.t.~~} (\mathfrak{S}^{'}_{i,j,q})^2 \leq \mathfrak{V}_{i,b_t^{'},b_t^{'}}, \sum_{j,q=1}^{m,\mathcal{E}_{2}(j)} \mathcal{G}_{2}(j,q) = 1, \mathfrak{S}^{'}_{i,1,1} = 1,\notag \\  \mathfrak{S}_i^{'} \in \mathbf{S}^{3m + 1}_+, \quad \mathfrak{V}_i \in \mathbf{S}^{6\tilde{p}_2}_{+}, \quad b_t^{'} = b_t - (3m + 1) \notag \\ \quad \forall \quad b_t \in \rho^k(j,q,r), t \in [1,18], i \in [i,n], \notag \\
        j \in [1,m], (j,q) \in \mathcal{E}_{2}, (j,q, r) \in \mathcal{E}_{3}, \lambda_I > 0, \label{eqn:hxs_3}
    \end{gather}
\end{subequations}

\textbf{Principal sub-matrices}. The \textit{second} modification comes from the principal sub-matrices of $\mathfrak{V}_{i}$. Following our strategy to limit usage of \y{psd} constraints to smaller matrices, we extract those principal sub-matrices of $\mathfrak{V}_{i}$ which represents the $18 \times 18$ matrix of the edges of each 2-simplex in $\mathcal{E}_3$. For doing this, we define a map $\beta^k$ such that:
\begin{equation}
    \beta^k_{j,q,r} (\mathfrak{V}_{i}) = \begin{pmatrix} \vartheta(i,j,q)^{\top} \\ \vartheta(i,q,r)^{\top} \\ \vartheta(i,j,r)^{\top} \end{pmatrix}^{\top} \begin{pmatrix} \vartheta(i,j,q)^{\top} \\ \vartheta(i,q,r)^{\top} \\ \vartheta(i,j,r)^{\top} \end{pmatrix} \in \mathbb{R}^{18 \times 18}
\end{equation}
This allows us to re-write equation~(\ref{eqn:hnr_xyz_sol_2}) with smaller \y{psd} constraints, given by:
\begin{subequations}\label{eqn:hnr_xyz_sol_3}
    \begin{gather}
        \min_{\{\mathfrak{S}_{i}^{'}\}, \{\mathfrak{V}_{i}\},\mathcal{G}_{2},\mathcal{G}_{3}}   \sum_{i=1}^n\bigg( \Big(f^{\Pi}_{i,j}(\mathfrak{S}_{i}^{'}) + \Tr(\mathfrak{S}_{i}^{'}) + \Tr(\mathfrak{V}_{i})\Big) + \notag \\ \sum_{j,q=1}^{m,\mathcal{E}_{2}(j)} \Big( \lambda_I |\mathfrak{g}_I^{P}(i,j,q)  - \mathcal{G}_{2}(j,q)| -  f_{i,j}^{\delta}(\mathfrak{S}_{i}^{'})   + \notag \\
        \sum_{r=1}^{\mathcal{E}_{3}(j,q)}  \lambda_E|\mathfrak{g}_E^{P}(i,j,q,r) - \mathcal{G}_{3}(j,q,r)|\Big) \bigg)
        \label{eqn:hxs_1}\\
        \text{s.t.~~} (\mathfrak{S}^{'}_{i,j,q})^2 \leq \mathfrak{V}_{i,b_t^{'},b_t^{'}}, \sum_{j,q=1}^{m,\mathcal{E}_{2}(j)} \mathcal{G}_{2}(j,q) = 1, \mathfrak{S}^{'}_{i,1,1} = 1,\notag \\  \mathfrak{S}_i^{'} \in \mathbf{S}^{3m + 1}_+, \quad \beta^k_{j,q,r}(\mathfrak{V}_i) \in \mathbf{S}^{18}_{+}, \quad b_t^{'} = b_t - (3m + 1) \notag \\ \quad \forall \quad b_t \in \rho^k(j,q,r), t \in [1,18], i \in [i,n], \notag \\
        j \in [1,m], (j,q) \in \mathcal{E}_{2}, (j,q, r) \in \mathcal{E}_{3}, \lambda_I > 0. \label{eqn:hxs_3}
    \end{gather}
\end{subequations}

\glsunset{lrsb}
\section{Simulators}\label{sec_sim}
We use two simulators, one for generating randomly deforming isometric surfaces and another for generating randomly deforming equiareal surfaces. We describe these two simulator below.

\textbf{Isometric surfaces}. The \y{isg} proposed by us is based on the method of \cite{perriollat2013computational}. Given a sparse grid of $m_a \times m_b$ flat 3D points $\mathbf{T}_s \in \mathbb{R}^{3 \times m}$ (with $m = m_a m_b$), we use the model from \cite{perriollat2013computational} to create $n$ bounded developable surfaces, represented as a pointcloud $\{ \mathbf{P}_{S,i} \in \mathbb{R}^{3 \times m}, \forall i \in [1,n]\}$. This is done by assigning random bending angles to the rulings used to map $\mathbf{T}_s$ to $\{\mathbf{P}_{S,i}\}$ using \cite{perriollat2013computational}. To $\{\mathbf{P}_{S,i}\}$, we assign random roto-translations $\mathbf{R}_{S,i}, \mathbf{t}_{S,i}$ and perspectively project to $\{\mathbf{p}_{S,i}\}$ using random realistic camera intrinsics. Optionally, we modify each projected point $\mathbf{p}_{S,i,j}$ by adding noise $\mathbf{p}^{'}_{S,i,j} = \mathbf{p}_{S,i,j} + x_{\sigma} \sigma_1$, where $\sigma_1 \sim \mathrm{U}(0,1)$ is drawn
from a random uniform distribution and the parameter $x_{\sigma}$ is a scalar multiplier. $\{\mathbf{p}_{S,i}^{'}\}$ gives our \textit{isometric simulated} data from \y{isg}.

\textbf{Equiareal surfaces}. For our \y{qsg}, we follow the same framework as \y{isg}, but randomly deform the developable surfaces $\{\mathbf{P}_{S,i}\}$ and use non-linear least squares to enforce equal area between 2-simplices. Given $\{\mathbf{P}_{S,i}\}$ from $\mathbf{T}_s$ using \cite{perriollat2013computational}, we obtain each randomly deformed 3D point as $\mathbf{P}_{E,i,j} = \mathbf{P}_{S,i,j} + \mathbf{1} \chi_{e} \sigma_2 $, where $\sigma_2 \sim \mathrm{U}(-0.5,0.5)$ is drawn from a random uniform distribution and $\mathbf{1}  = (1,1,1)$. We already have $\mathcal{G}_3$ from $\mathbf{T}_s$. Therefore, the quasi-equiareal pointcloud is obtained as:
\begin{equation}
    \begin{gathered}
        \{\mathbf{P}_{Q,i}\} = \\ \argmin_{\{\mathbf{P}_{E,i,j}\}} \sum_{i,j,q,r=1}^{n,m,\mathcal{E}_2(j),\mathcal{E}_3(j,q)} \bigg(\mathcal{G}_3(j,q,r) - \mathfrak{g}_E(i,j,q,r)\bigg)^2.
    \end{gathered}
    \label{eqn:qeg}
\end{equation}
We solve equation~(\ref{eqn:qeg}) with the Levenberg–Marquardt algorithm.

\section{Additional Experimental Results}\label{sec_exp}

\glsunset{s1} \glsunset{s2} \glsunset{q1} \glsunset{q2} \glsunset{h1} \glsunset{h2} \glsunset{gt} \glsunset{sb} \glsunset{sl} \glsunset{kp} \glsunset{wc} \glsunset{st}
We now present some additional results that could not be included in the main article due to space limitations:
\subsection{Additional qualitative results}
We show some qualitative results from our experiments on the Hulk, \y{kp}, \y{wc} and \y{st} datasets with \y{s1}, \y{s2}, \y{q1} and \y{q2} in table~(\ref{tab_2}). We show one frame for each of the datasets, densely interpolated using \cite{amidror2002scattered}. Our proposed methods reconstruct the shape accurately in all cases. Among the state-of-the-art methods, the zeroth order \y{nrsfm} techniques (\cite{chhatkuli2017inextensible}, \cite{ji2017maximizing}) are very accurate with the inextensible objects (\y{kp}, \y{wc}, Hulk), but we manage to outperform them in \y{kp} and Hulk, while in \y{wc}, we are the close second. \y{lrsb} methods (\cite{hamsici2012learning}, \cite{gotardo2011kernel}, \cite{dai2014simple}) do not seem to be very accurate with inextensible objects. However, in the highly extensible data of \y{st}, the \y{lrsb} methods work very accurately, the zeroth order methods suffer to maintain accuracy and our proposed method maintains accuracy at par with \y{lrsb} methods. These results on these datasets fit the claim that quasi-isometry is a better model for handling both inextensible and extensible objects, which is not the case for either the previous zeroth-order \y{nrsfm} methods or the \y{lrsb} methods. In figure~(\ref{fig_qual_synth_large}), we show some additional qualitative results on synthetic data using \y{s1}, \y{s2}, \y{q1} and \y{q2} on \y{isg}.

\newcommand{\imgwidth}{1.21cm}
\newcommand{\imgheight}{1.0cm}
\begin{table}[]
\makegapedcells
\centering
\caption{Qualitative results for Hulk, KP, WC, ST}
\label{tab_2}
\begin{tabular}{ccccc}
\toprule
& Hulk & \y{kp} & \y{wc} & \y{st} \\ \hline
\multicolumn{1}{c|}{\scriptsize{Input}} &   \cincludegraphics[height=\imgheight,trim=80 0 30 20,clip]{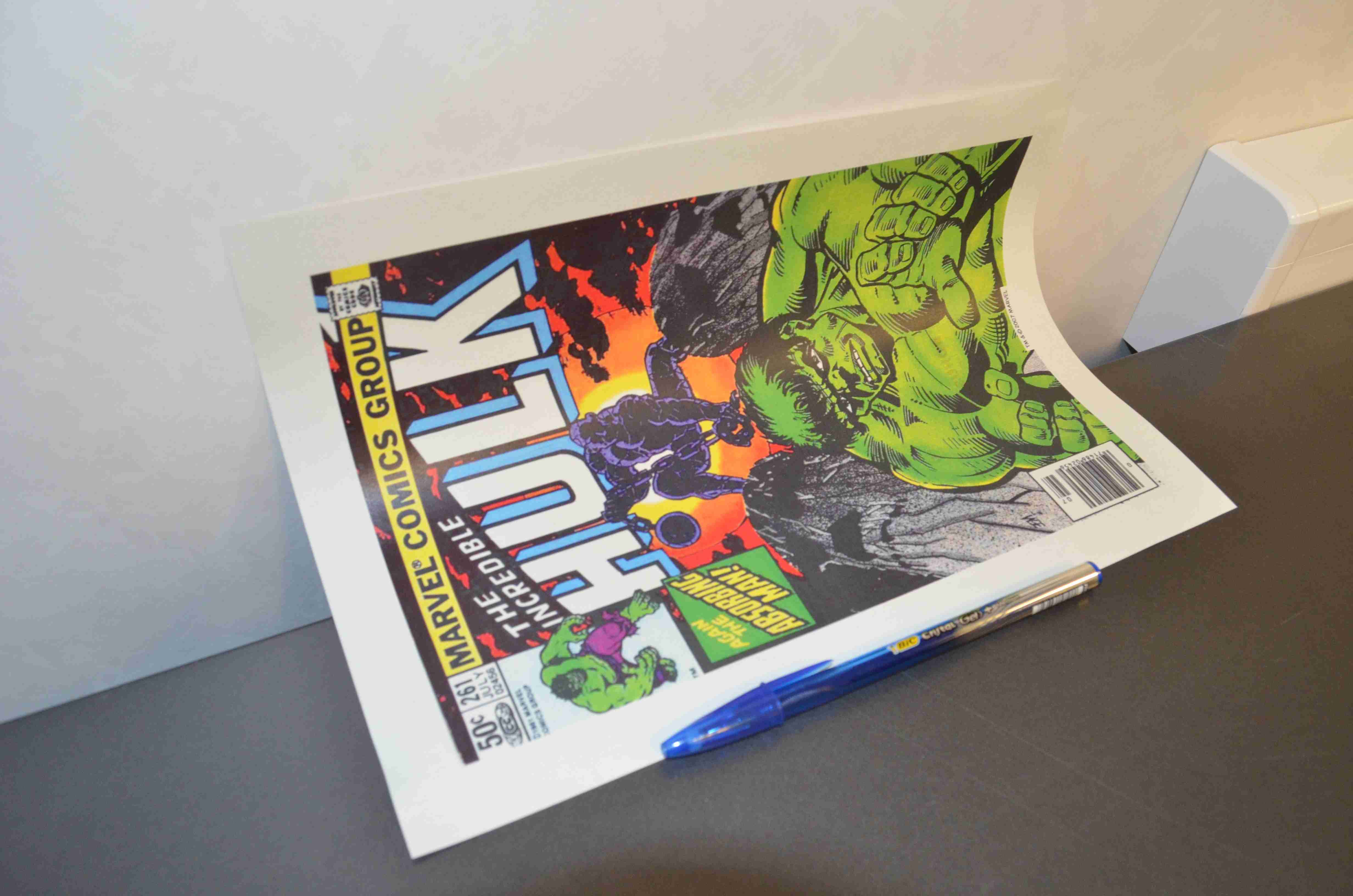}         &     \cincludegraphics[height=\imgheight,trim=120 50 30 40,clip]{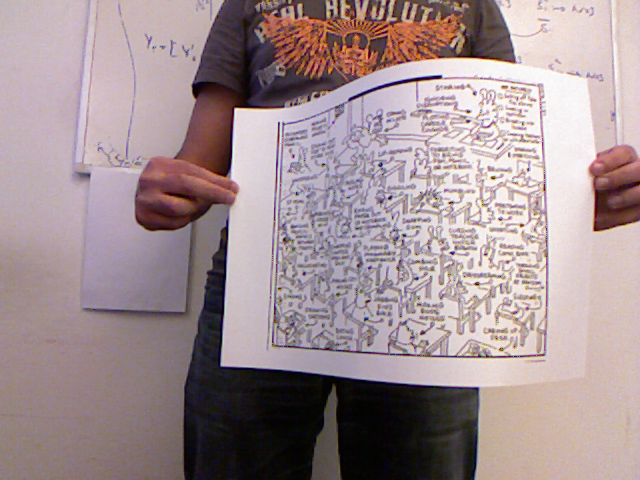}  &    \cincludegraphics[height=\imgheight,trim=30 20 100 20,clip]{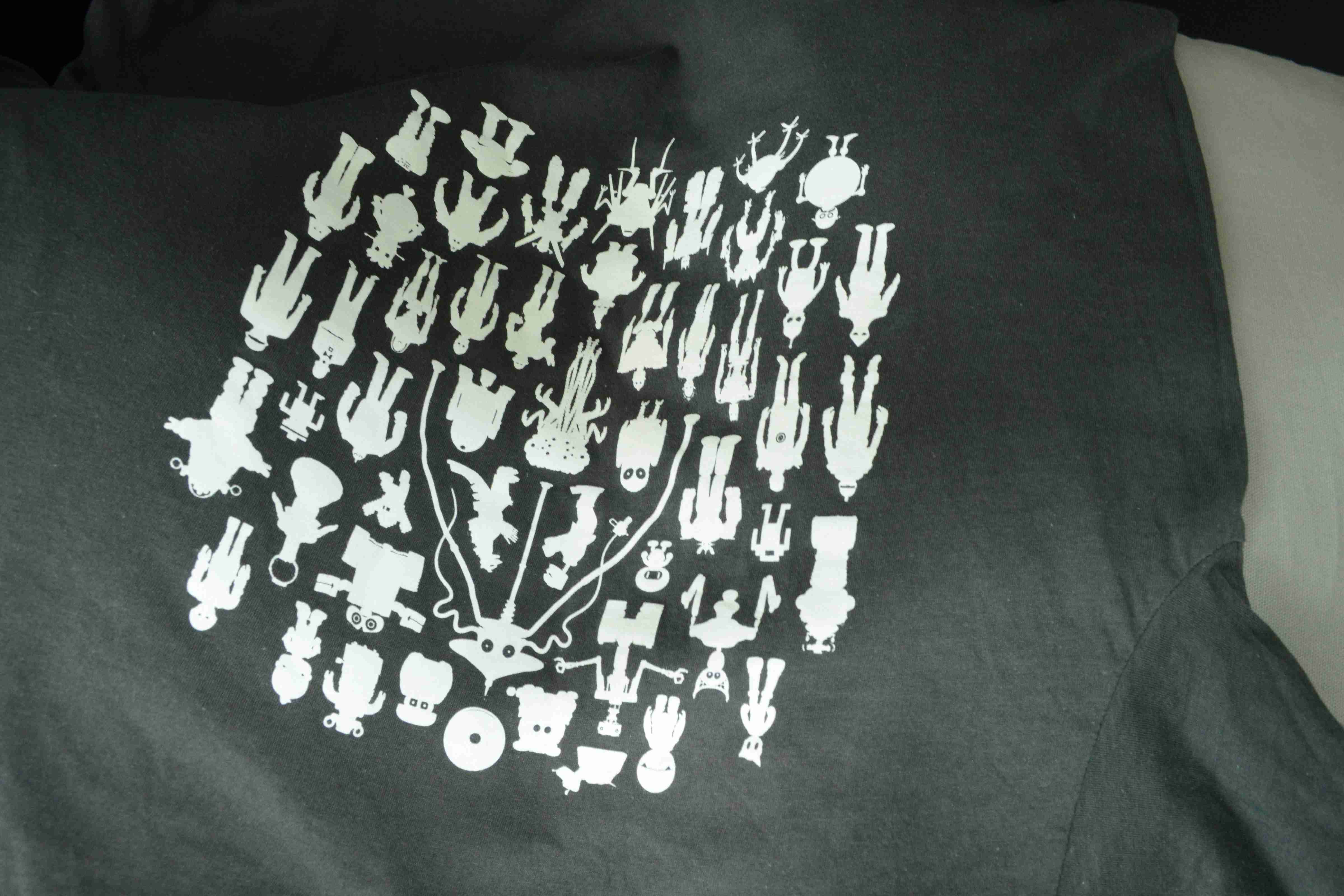}          &  \cincludegraphics[height=\imgheight,trim=0cm 70 20cm 5cm,clip]{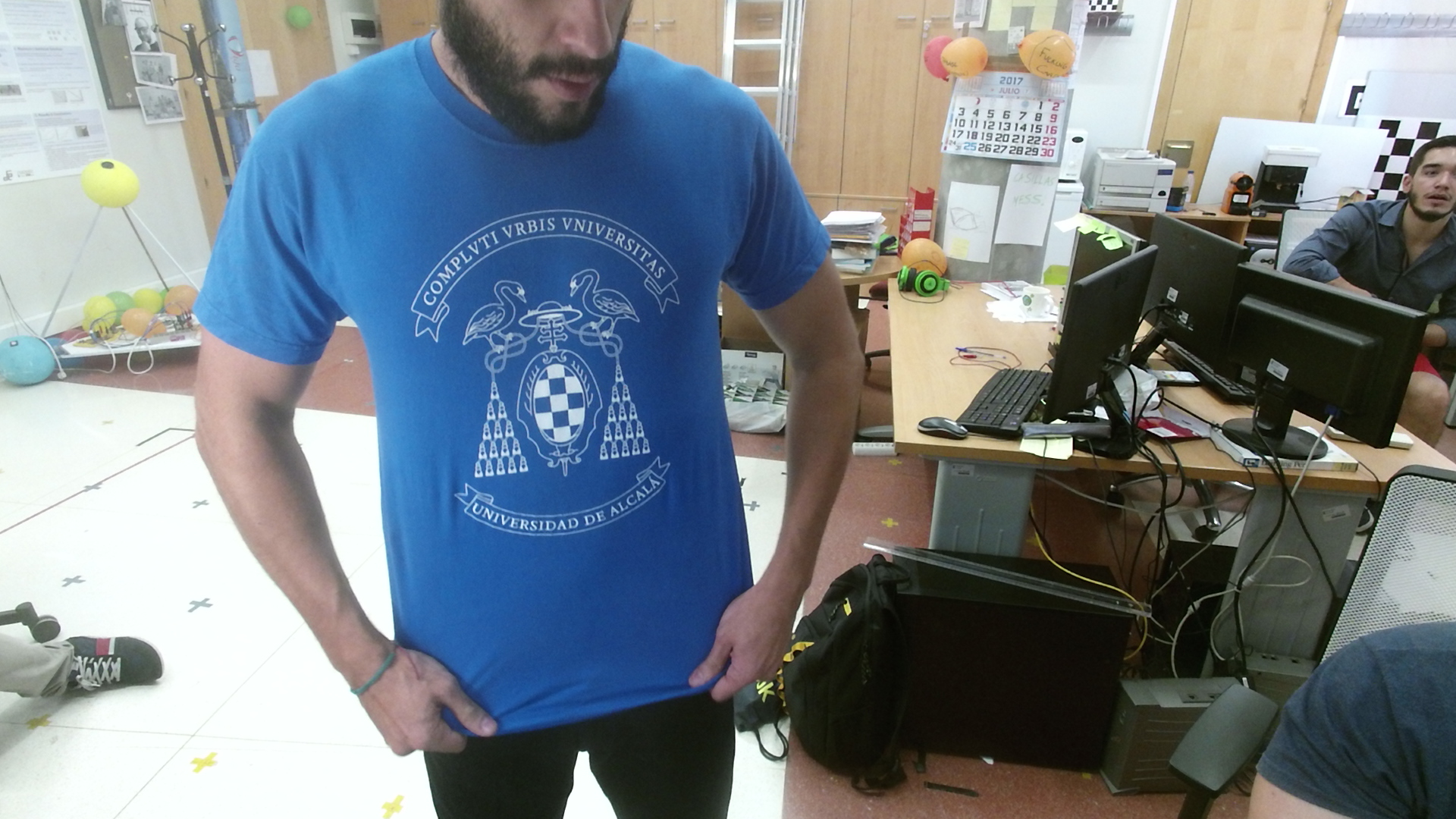}   \\ \hline           \multicolumn{1}{c|}{\y{gt}} &   \cincludegraphics[width=\imgwidth]{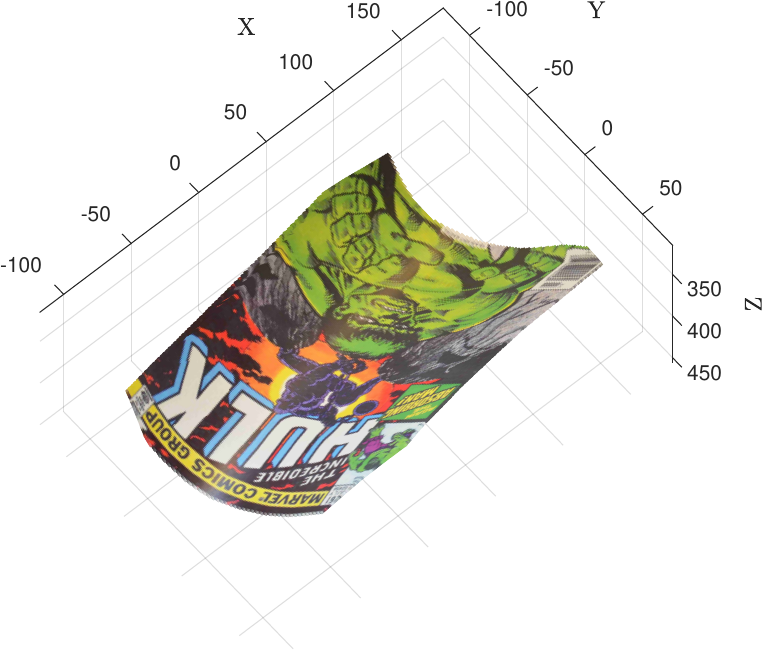}    &    \cincludegraphics[width=\imgwidth]{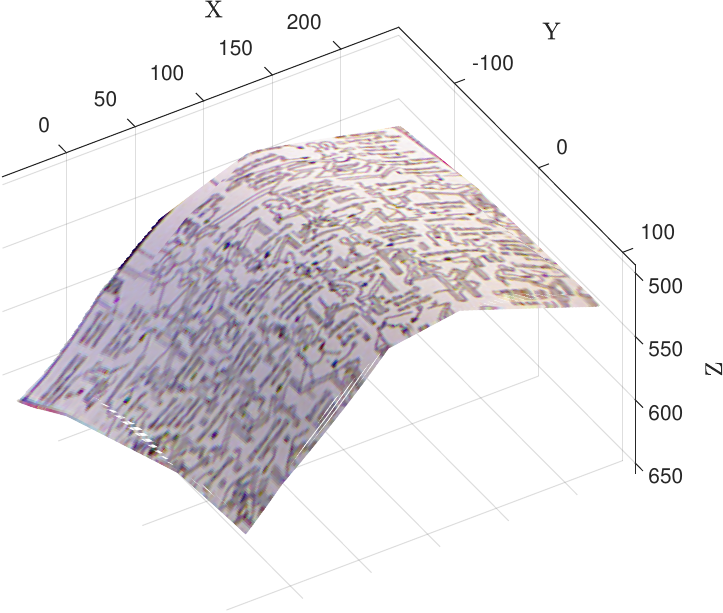}          &  \cincludegraphics[width=\imgwidth]{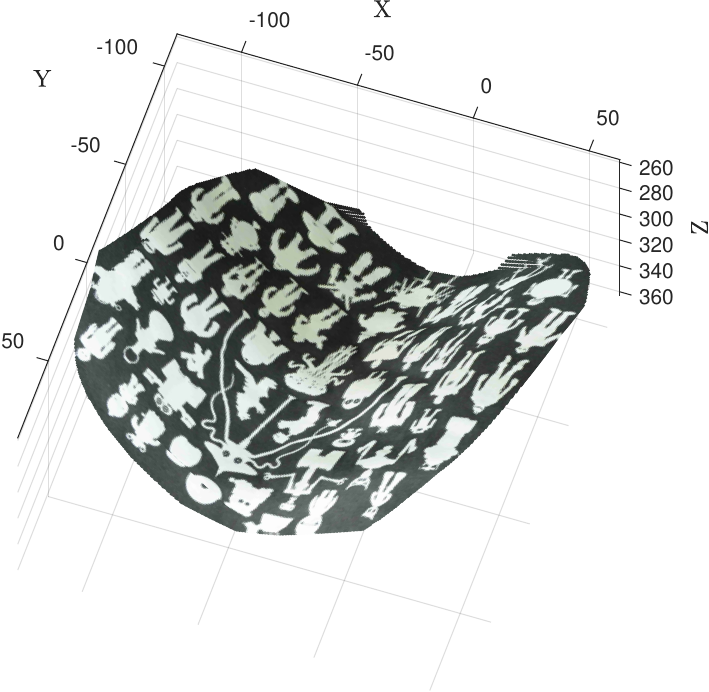}  & \cincludegraphics[width=\imgwidth]{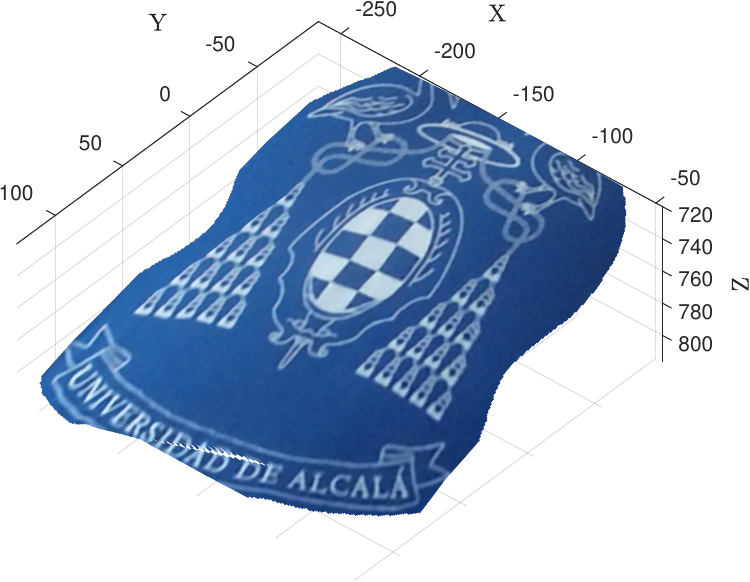}   \\ \hline
\multicolumn{1}{c|}{\y{s1}} & \cincludegraphics[width=\imgwidth]{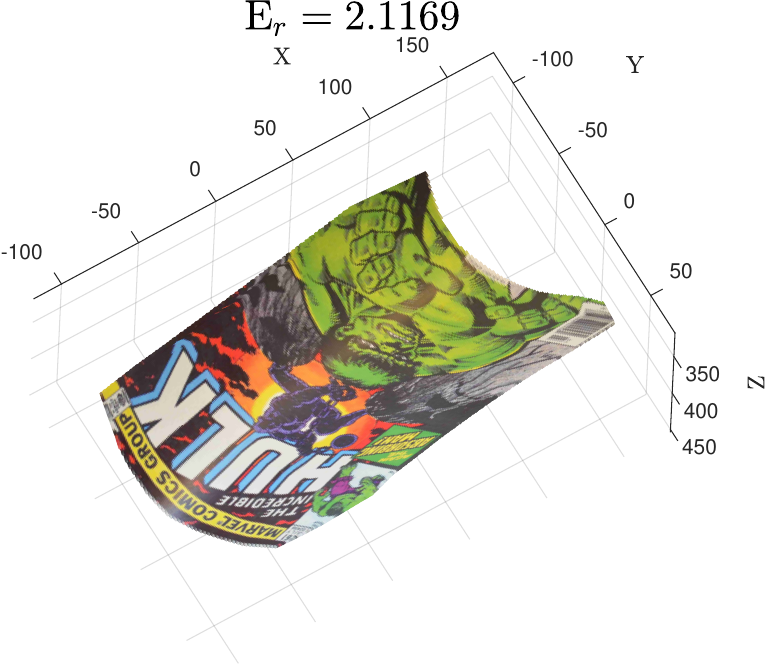}    &    \cincludegraphics[width=\imgwidth]{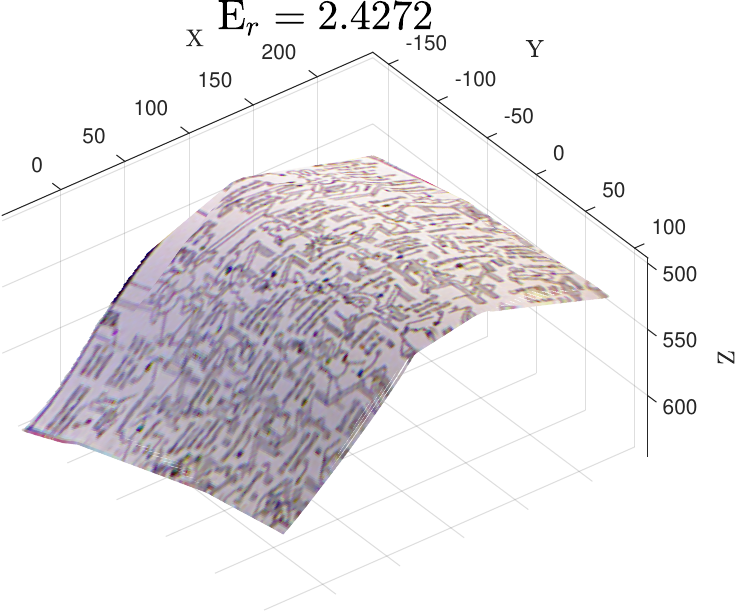}          &  \cincludegraphics[width=\imgwidth]{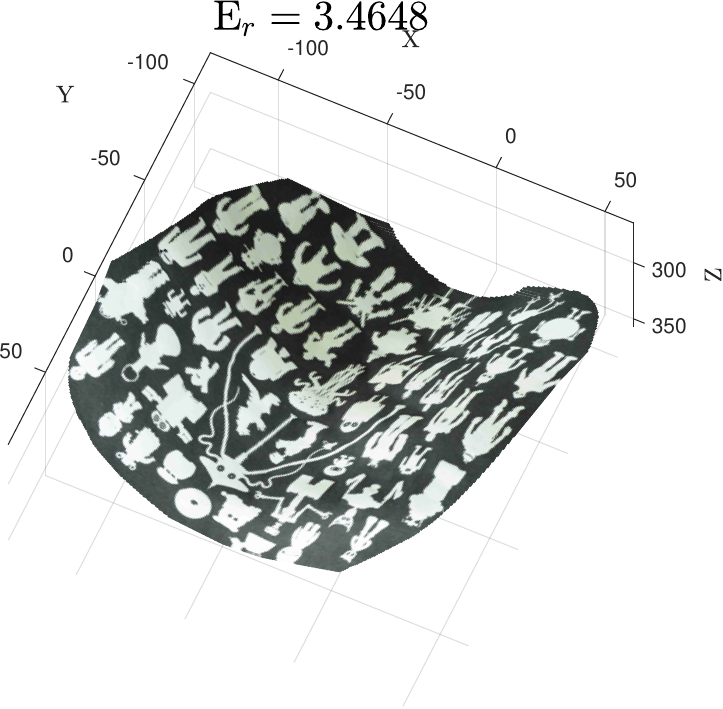}  & \cincludegraphics[width=\imgwidth]{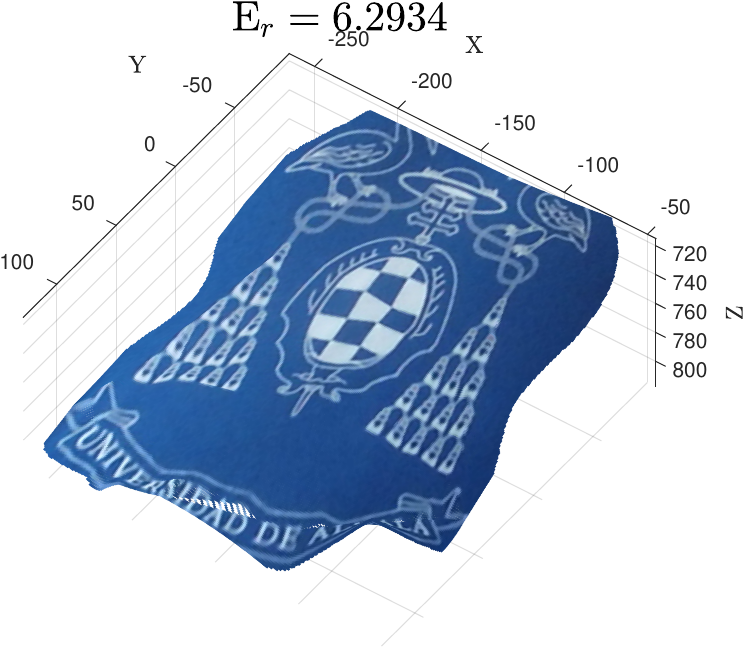}            \\ \hline
\multicolumn{1}{c|}{\y{s2}} & \cincludegraphics[width=\imgwidth]{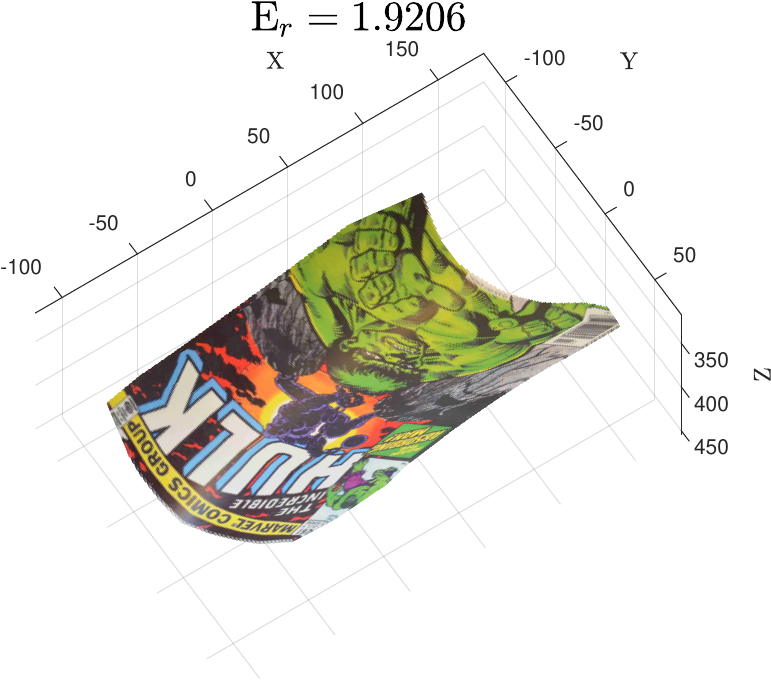}    &    \cincludegraphics[width=\imgwidth]{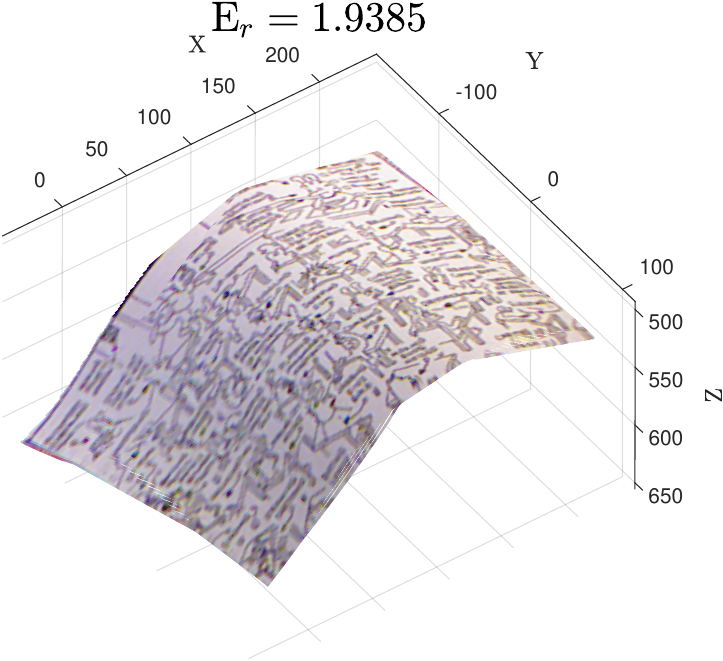}          &  \cincludegraphics[width=\imgwidth]{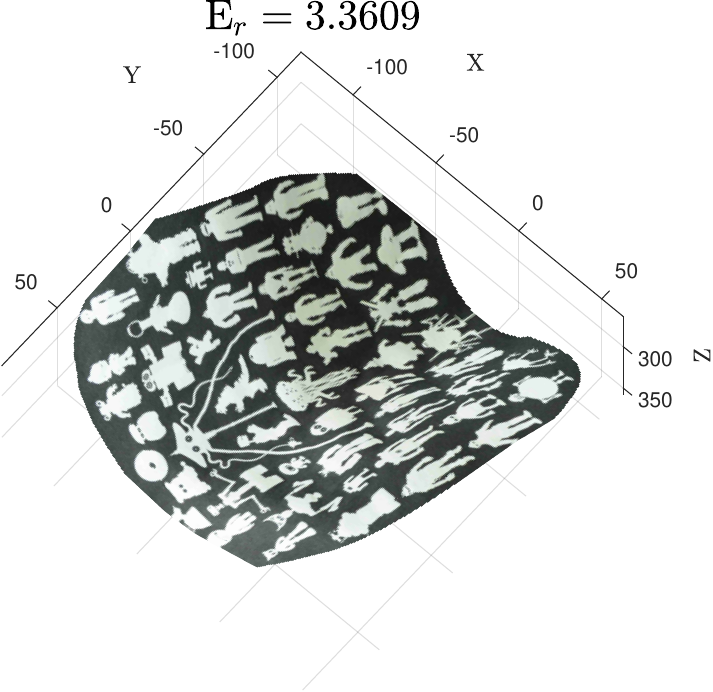}  & \cincludegraphics[width=\imgwidth]{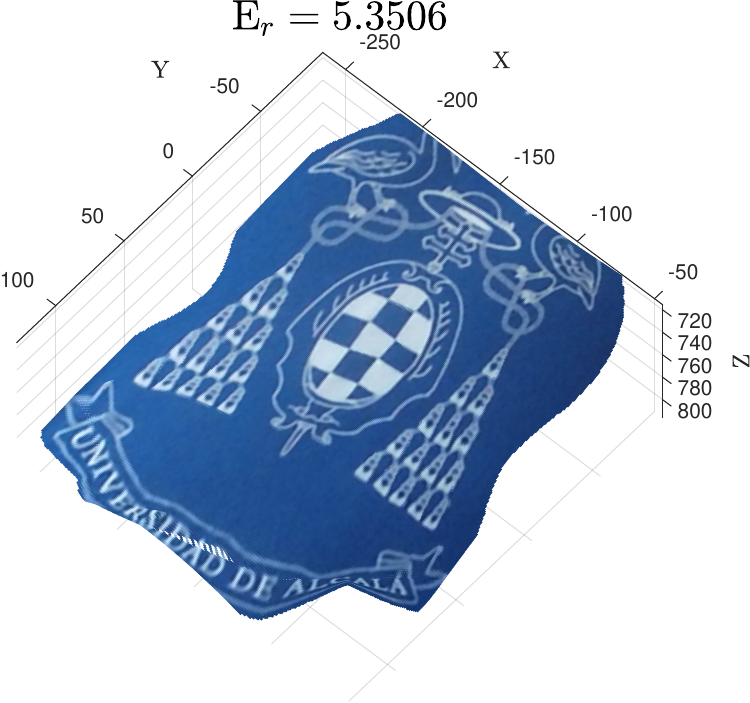}           \\ \hline
\multicolumn{1}{c|}{\y{q1}} & \cincludegraphics[width=\imgwidth]{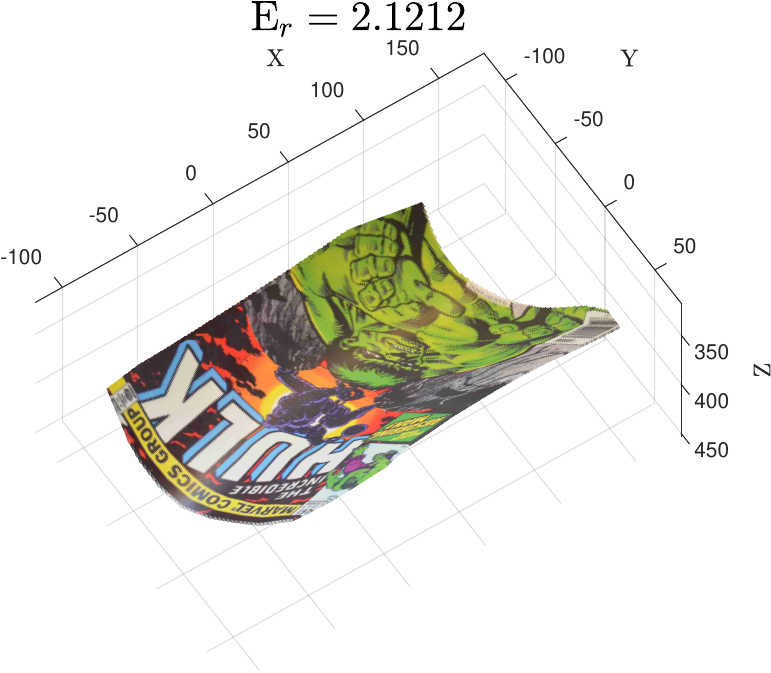}    &    \cincludegraphics[width=\imgwidth]{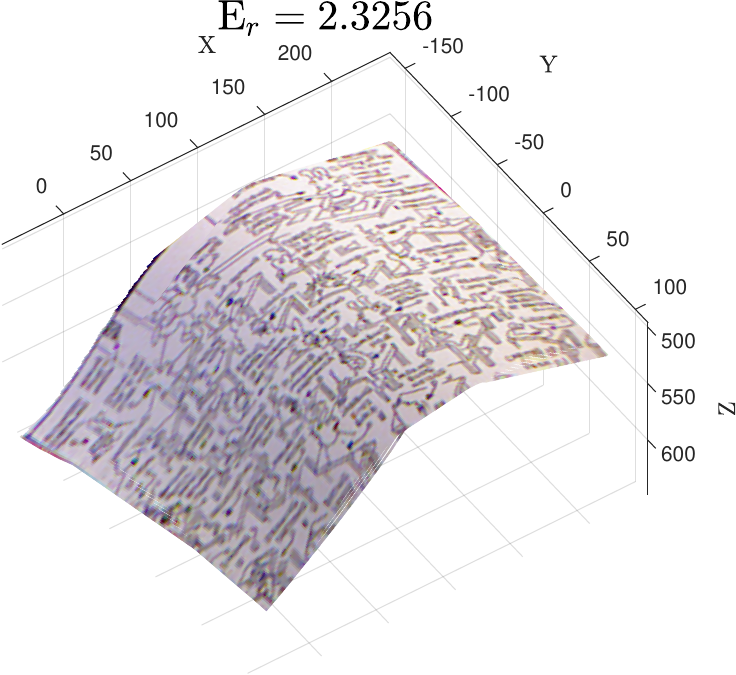}          &  \cincludegraphics[width=\imgwidth]{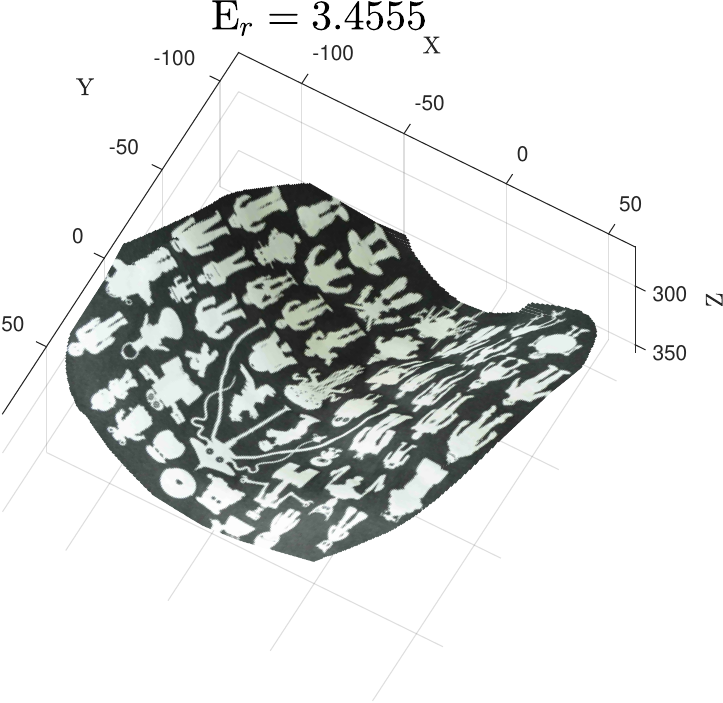}  & \cincludegraphics[width=\imgwidth]{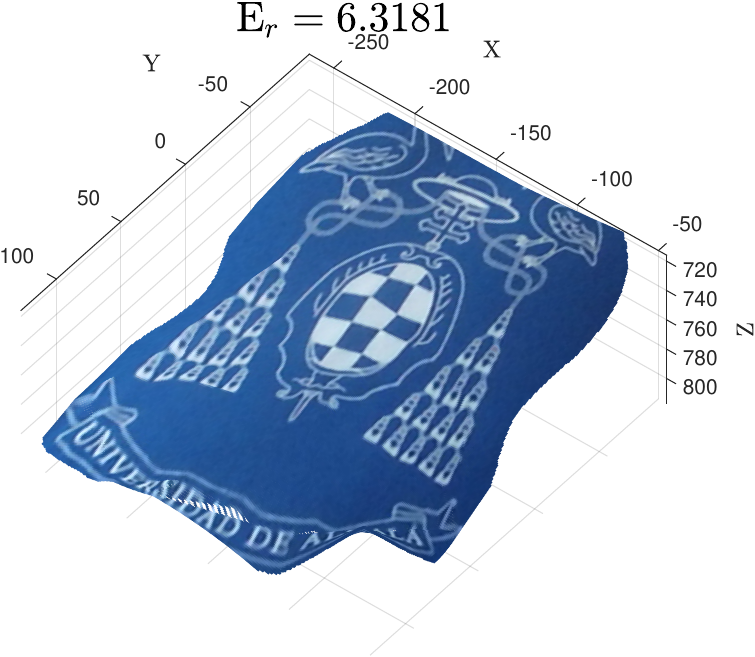}   \\ \hline
\multicolumn{1}{c|}{\y{q2}}& \cincludegraphics[width=\imgwidth]{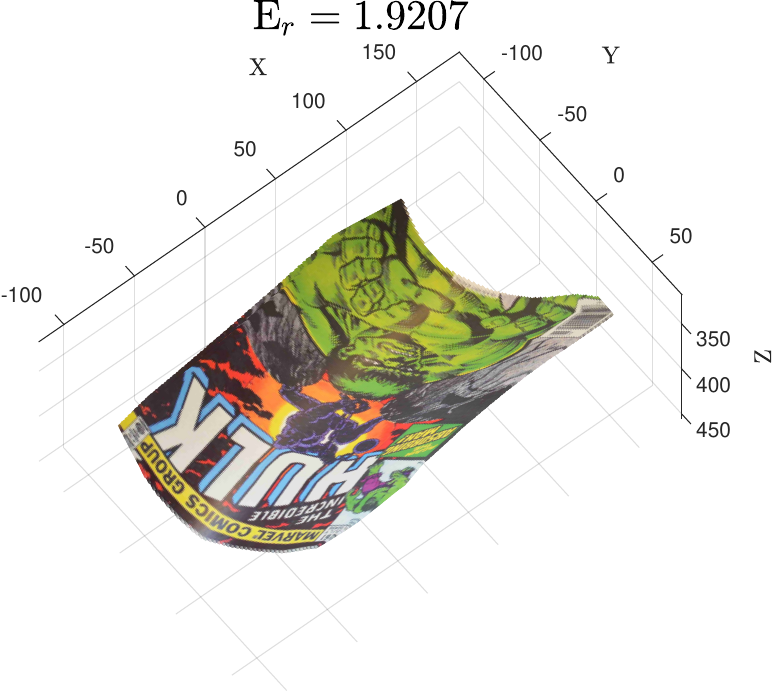}    &    \cincludegraphics[width=\imgwidth]{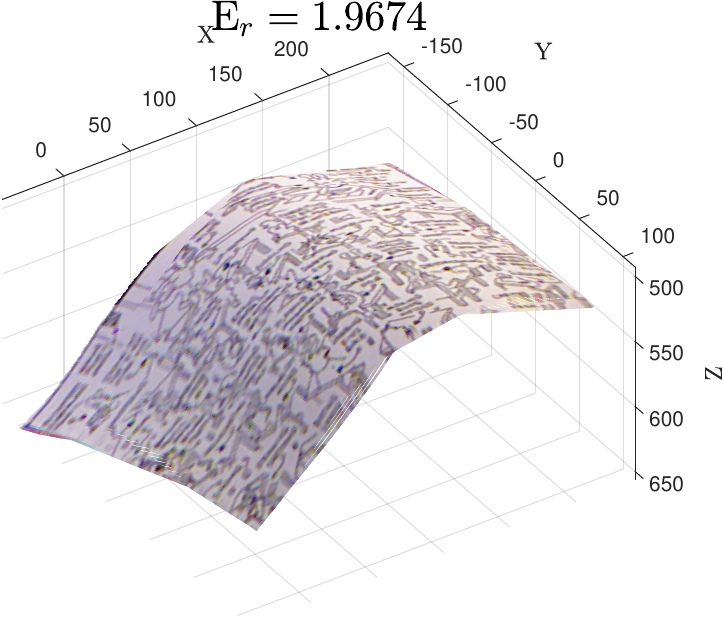}          &  \cincludegraphics[width=\imgwidth]{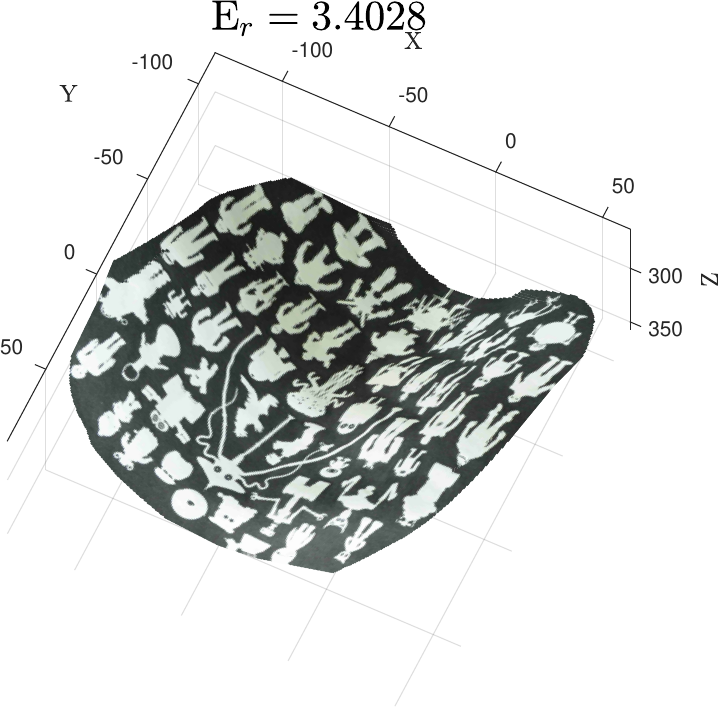}  & \cincludegraphics[width=\imgwidth]{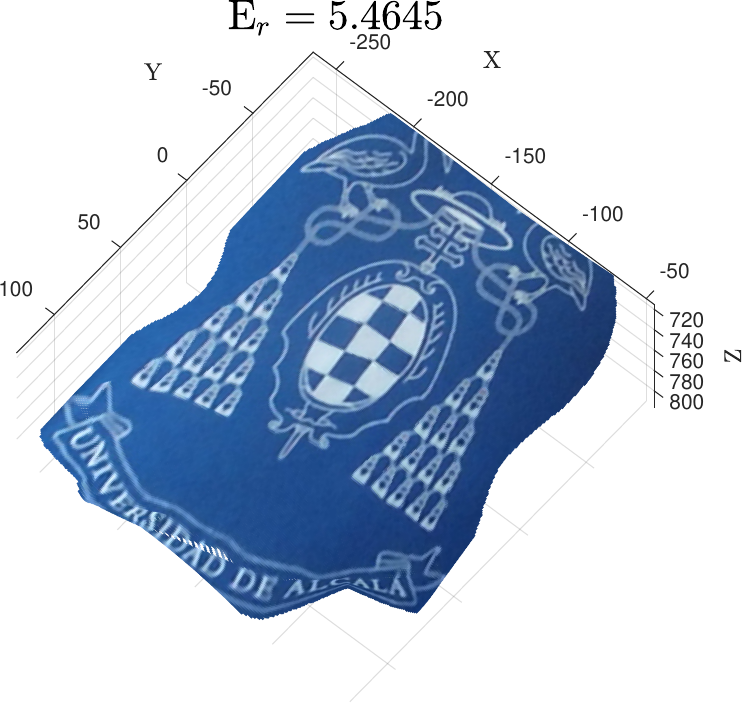}  \\ \hline
\multicolumn{1}{c|}{\cite{ji2017maximizing}} & \cincludegraphics[width=\imgwidth]{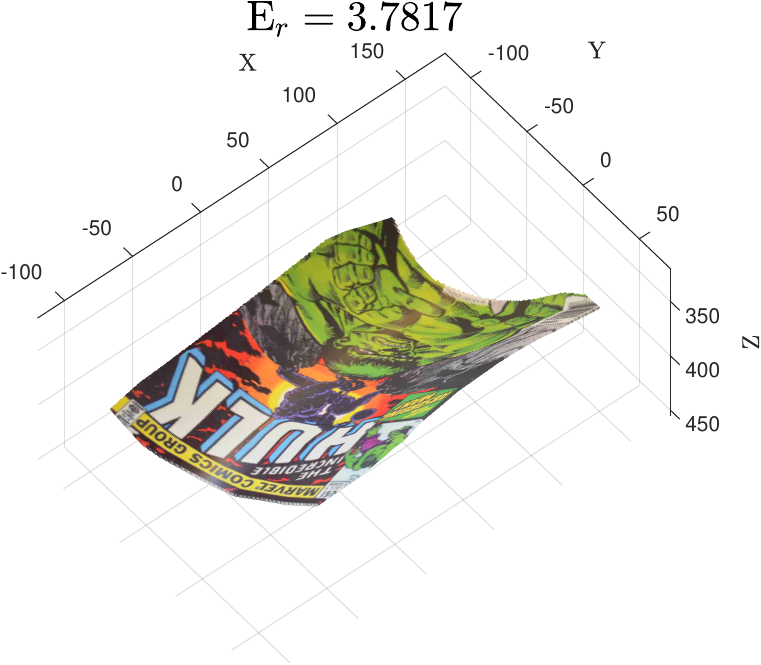}    &    \cincludegraphics[width=\imgwidth]{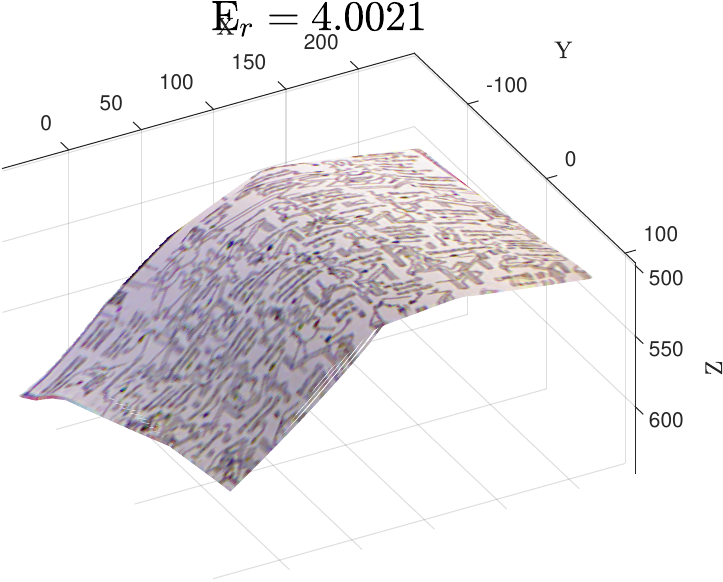}          &  \cincludegraphics[width=\imgwidth]{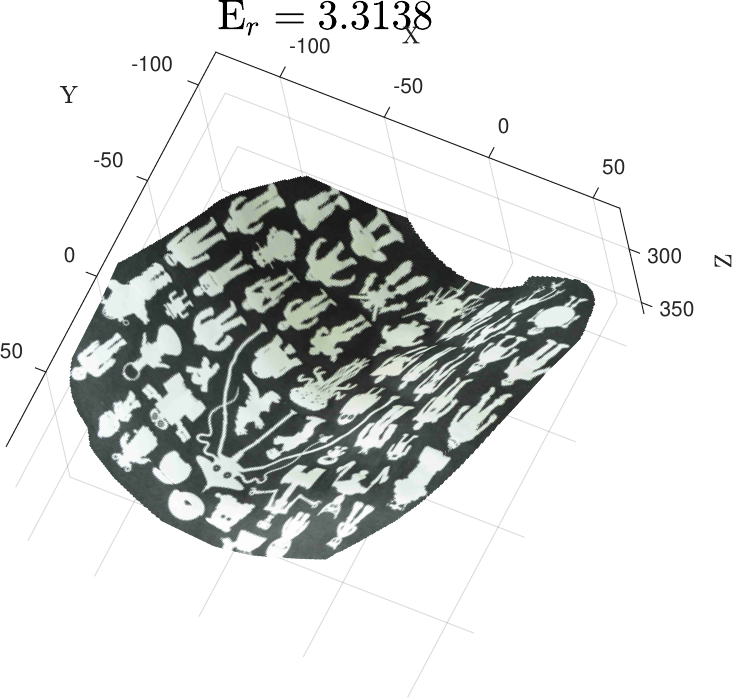}  & \cincludegraphics[width=\imgwidth]{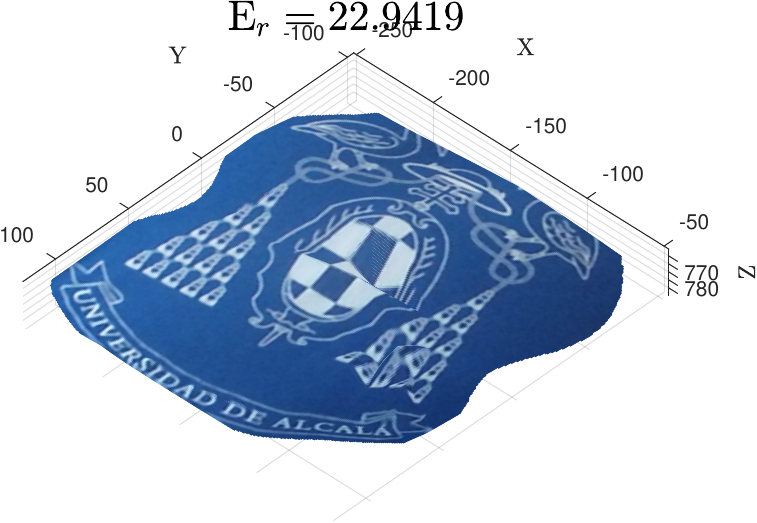}  \\ \hline
\multicolumn{1}{c|}{\cite{chhatkuli2017inextensible}}& \cincludegraphics[width=\imgwidth]{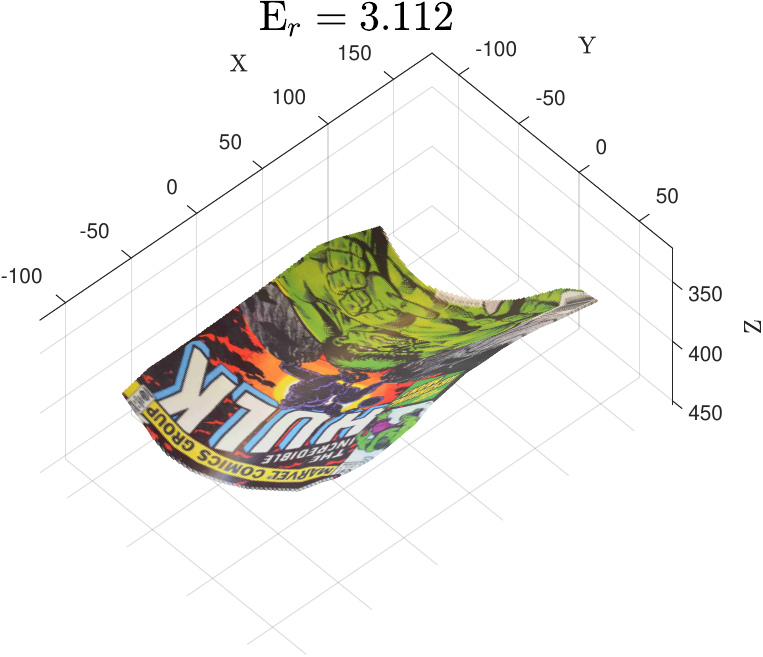}    &    \cincludegraphics[width=\imgwidth]{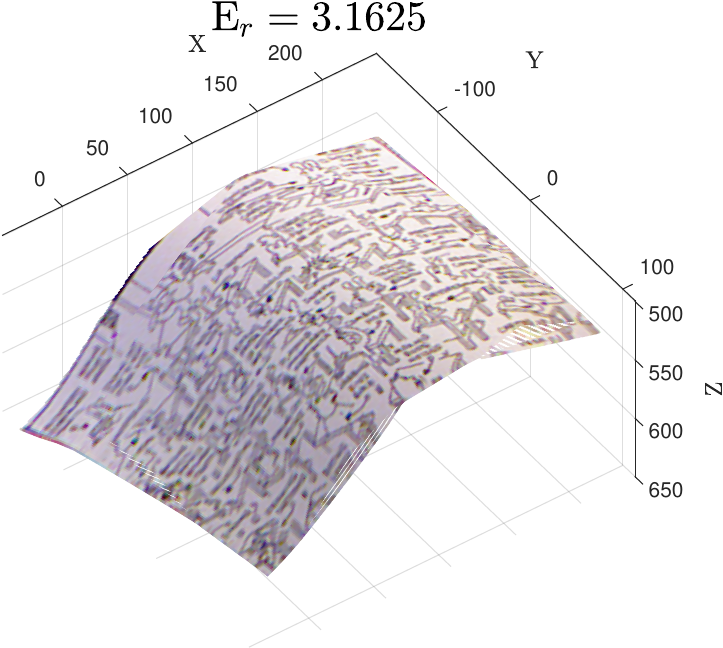}          &  \cincludegraphics[width=\imgwidth]{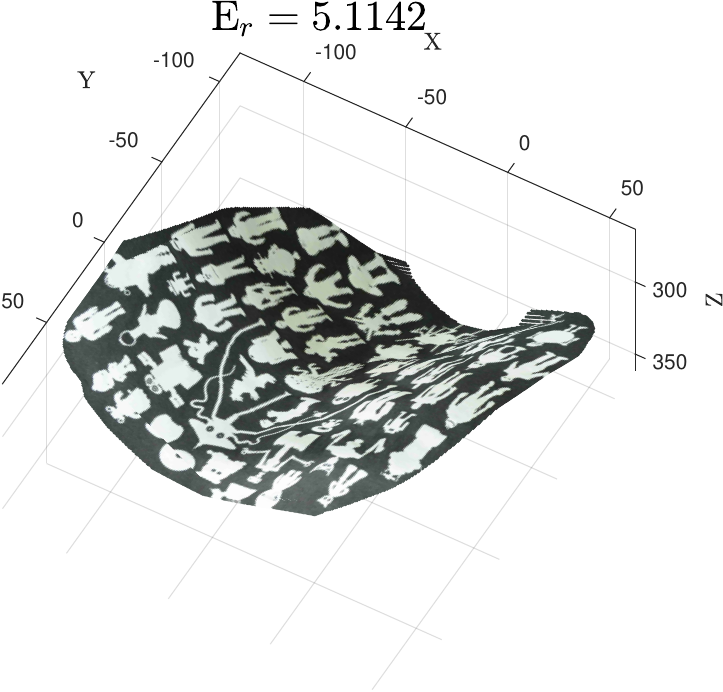}  & \cincludegraphics[width=\imgwidth]{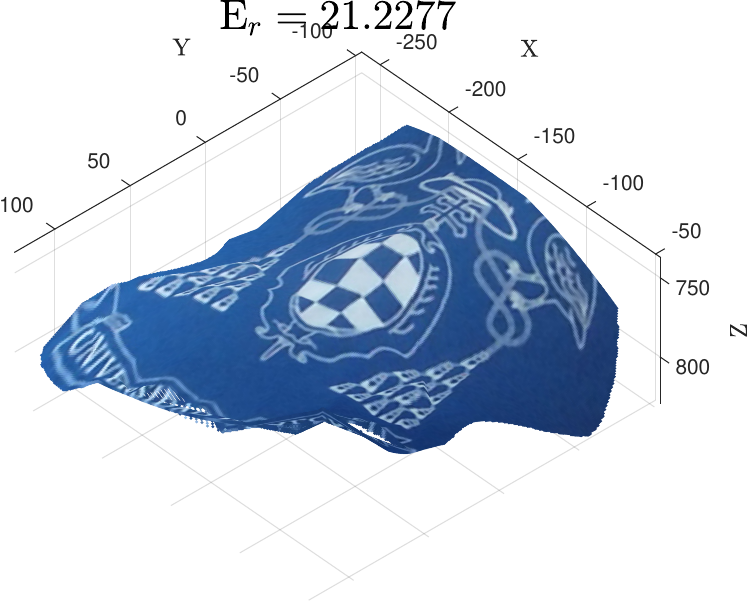}  \\ \hline
\multicolumn{1}{c|}{\cite{parashar2017isometric}-IP}& \cincludegraphics[width=\imgwidth]{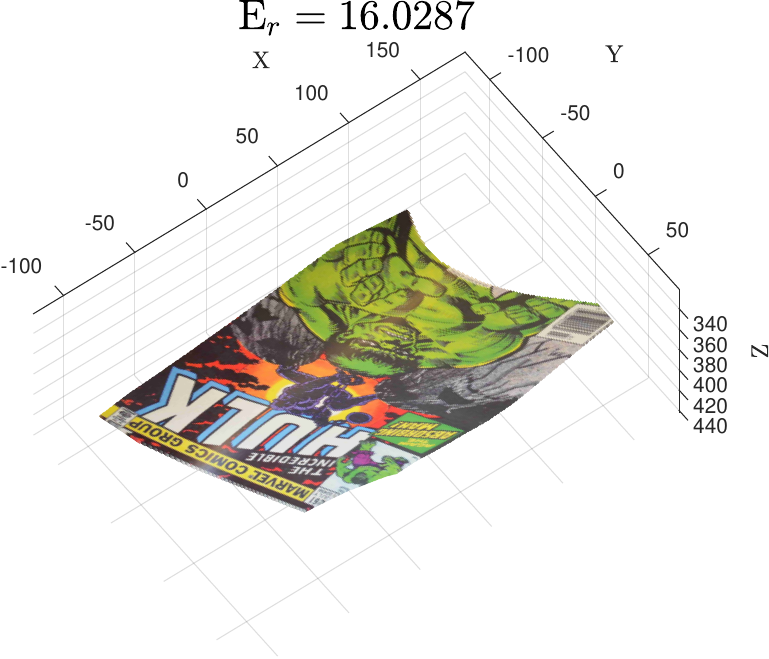}    &    \cincludegraphics[width=\imgwidth]{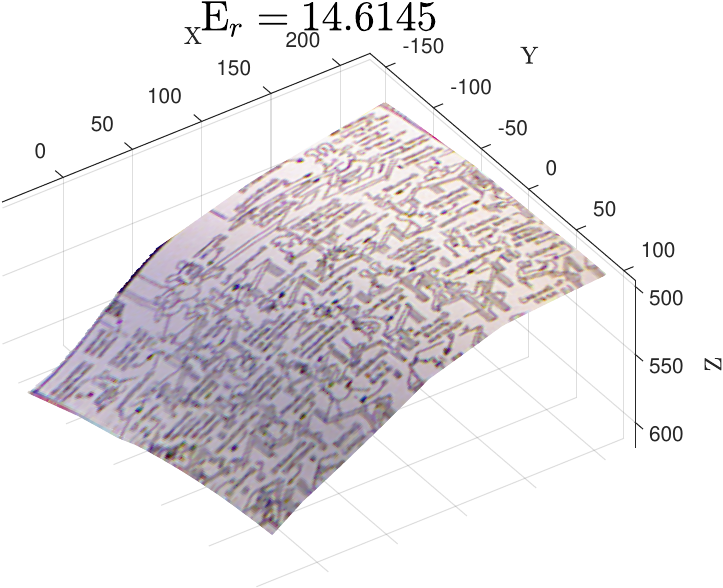}          &  \cincludegraphics[width=\imgwidth]{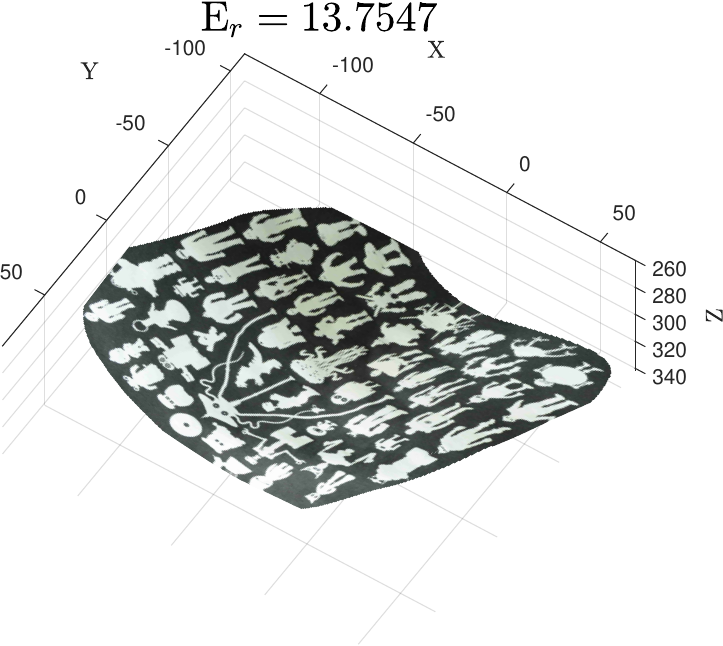}  & \cincludegraphics[width=\imgwidth]{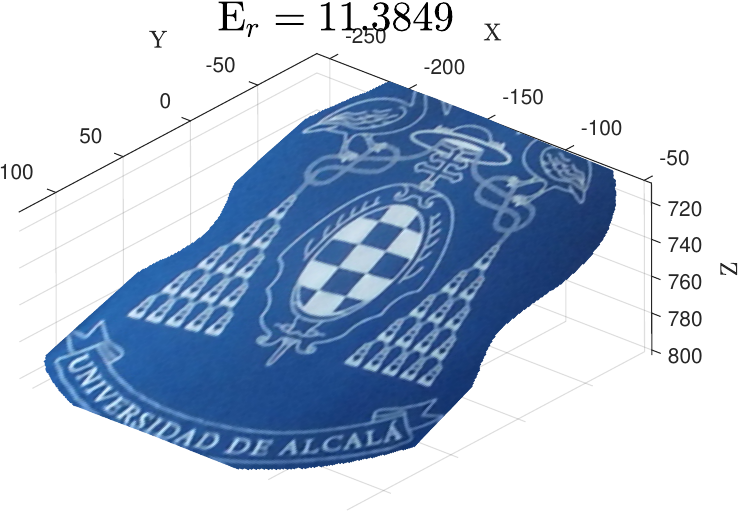}   \\ \hline
\multicolumn{1}{c|}{\cite{parashar2017isometric}-G} & \cincludegraphics[width=\imgwidth]{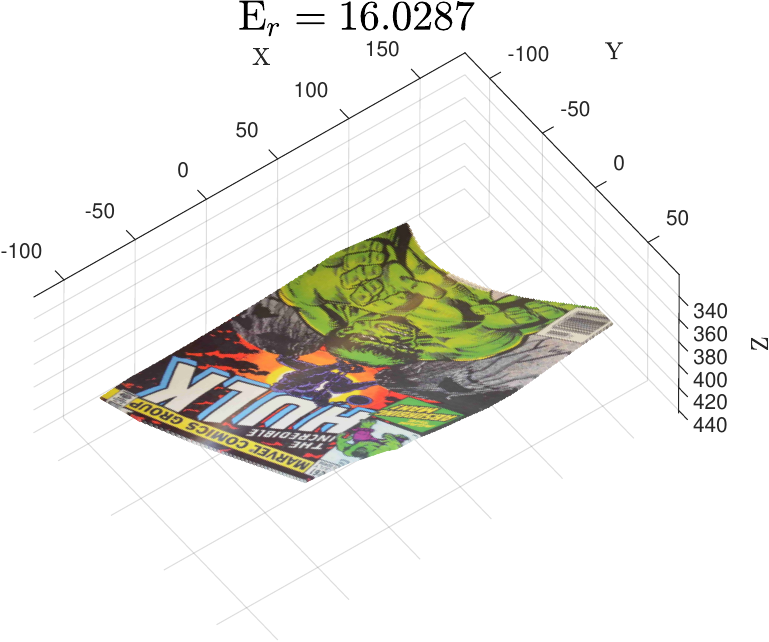}    &    \cincludegraphics[width=\imgwidth]{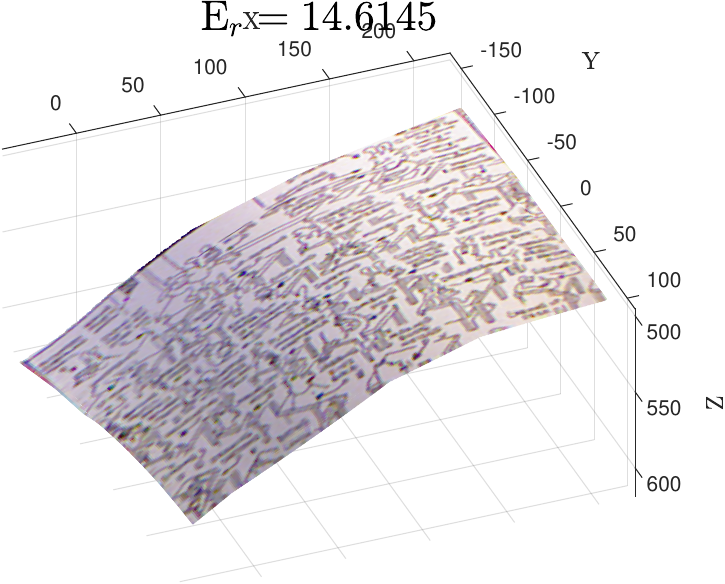}          &  \cincludegraphics[width=\imgwidth]{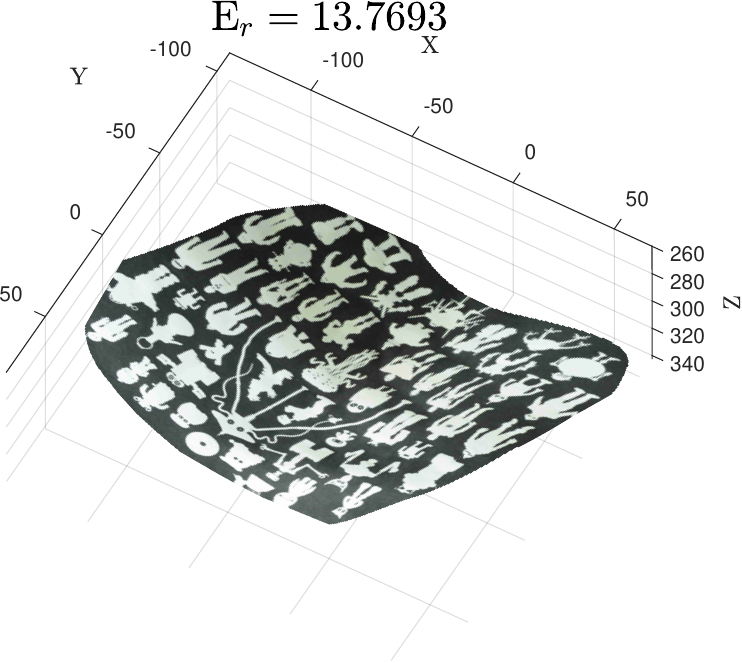}  & \cincludegraphics[width=\imgwidth]{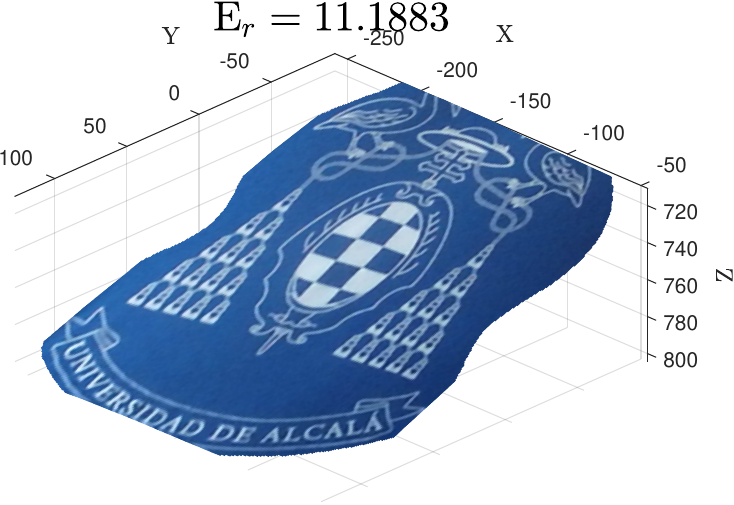} \\ \hline
\multicolumn{1}{c|}{\cite{hamsici2012learning}}& \cincludegraphics[width=\imgwidth]{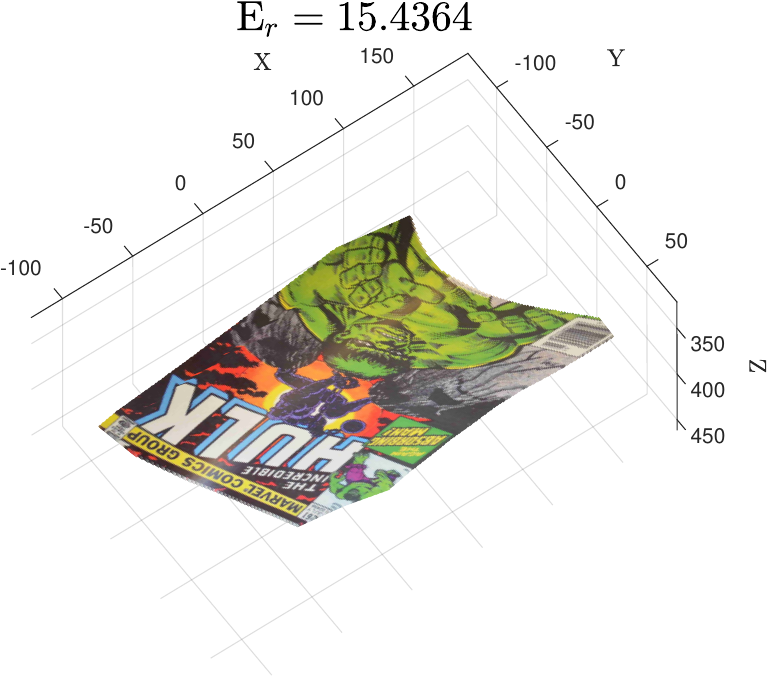}    &    \cincludegraphics[width=\imgwidth]{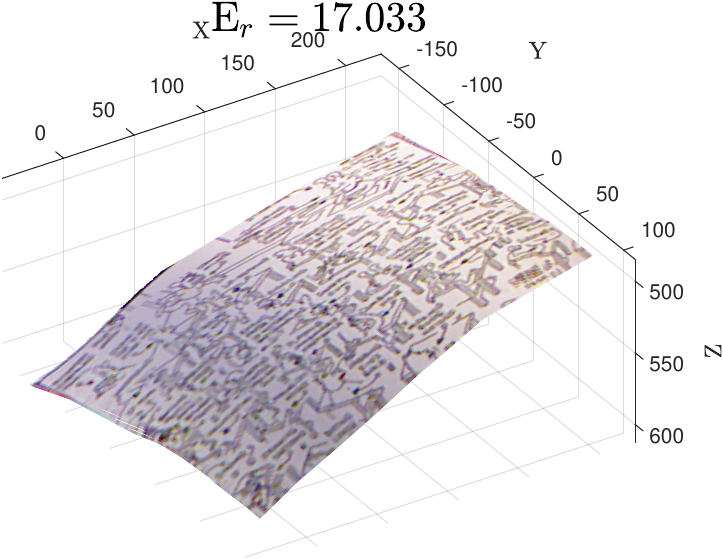}          &  \cincludegraphics[width=\imgwidth]{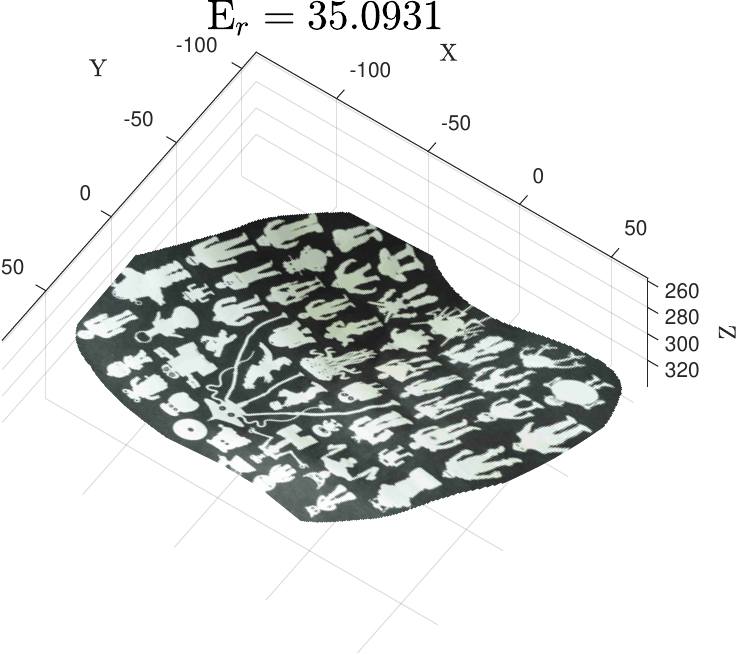}  & \cincludegraphics[width=\imgwidth]{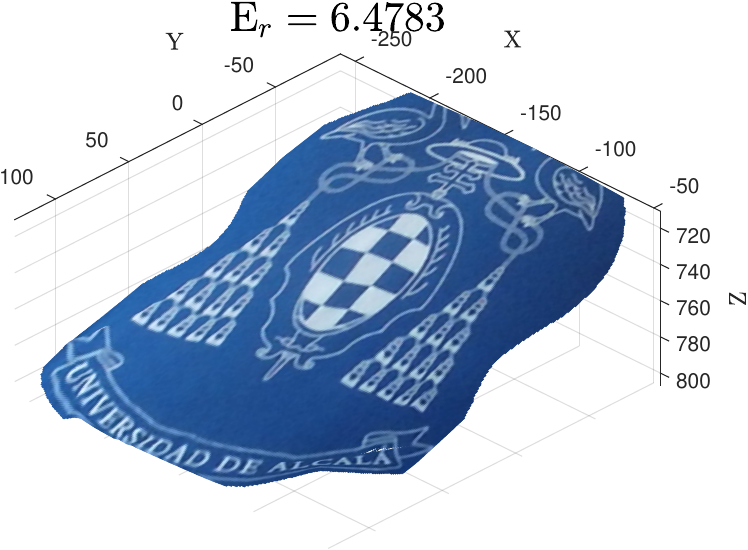} \\ \hline
\multicolumn{1}{c|}{\cite{gotardo2011kernel}}& \cincludegraphics[width=\imgwidth]{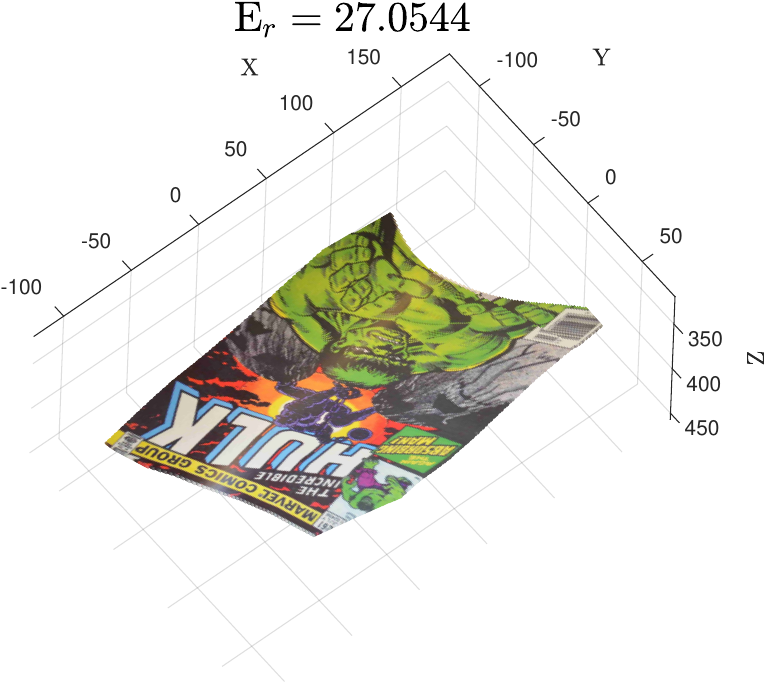}    &    \cincludegraphics[width=\imgwidth]{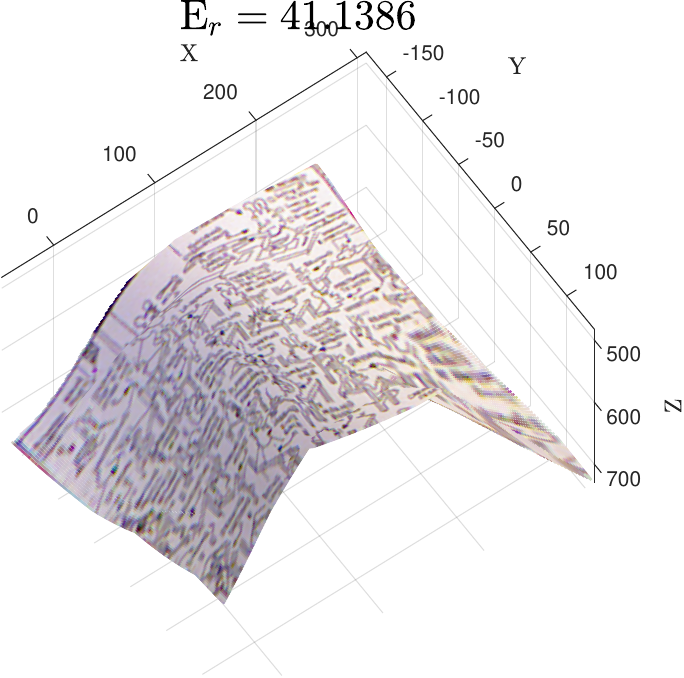}          &  \cincludegraphics[width=\imgwidth]{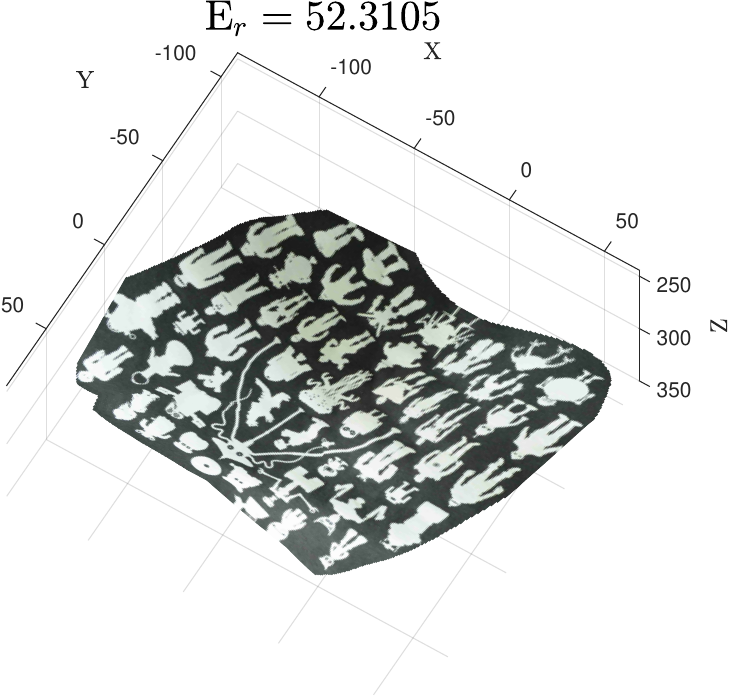}  & \cincludegraphics[width=\imgwidth]{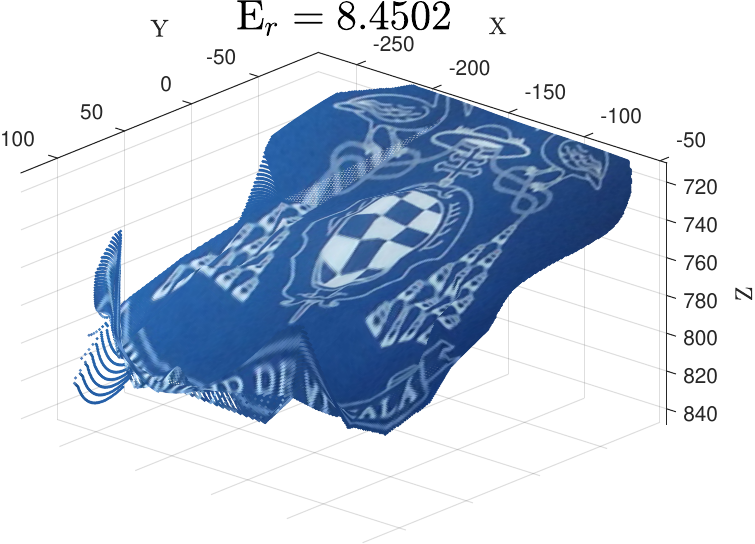} \\ \hline
\multicolumn{1}{c|}{\cite{dai2014simple}}& \cincludegraphics[width=\imgwidth]{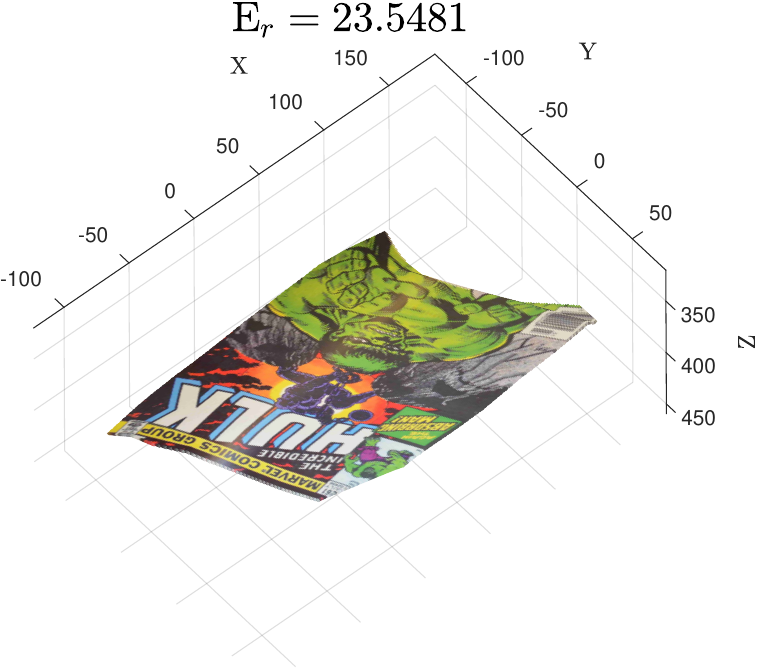}    &    \cincludegraphics[width=\imgwidth]{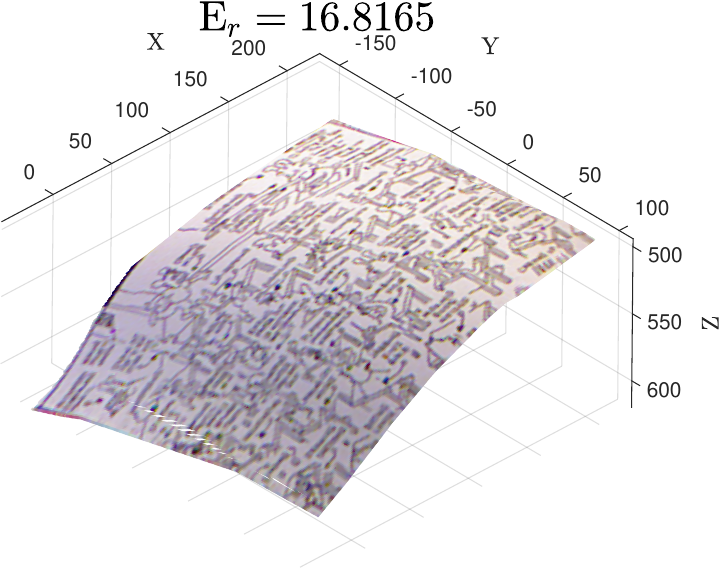}          &  \cincludegraphics[width=\imgwidth]{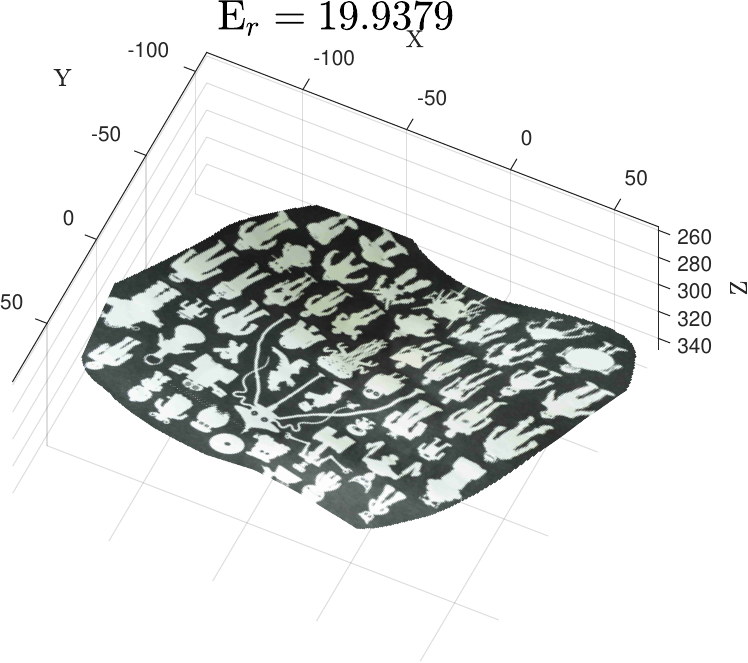}  & \cincludegraphics[width=\imgwidth]{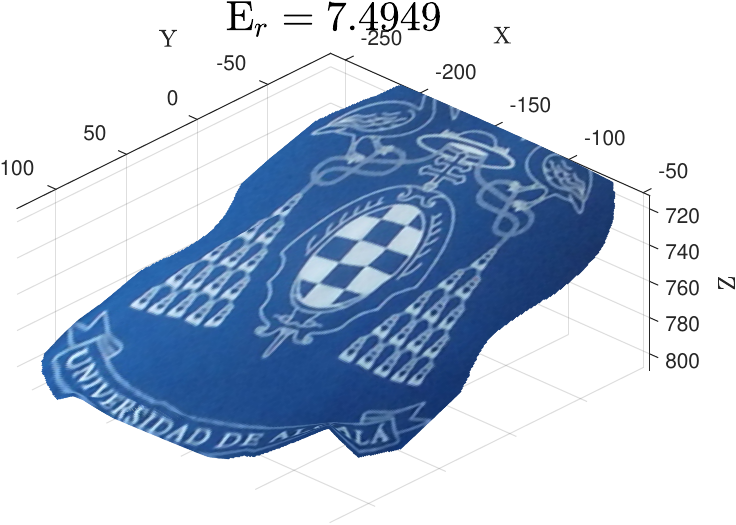} \\ \hline
\multicolumn{1}{c|}{\cite{chhatkuli2014non}}& \cincludegraphics[width=\imgwidth]{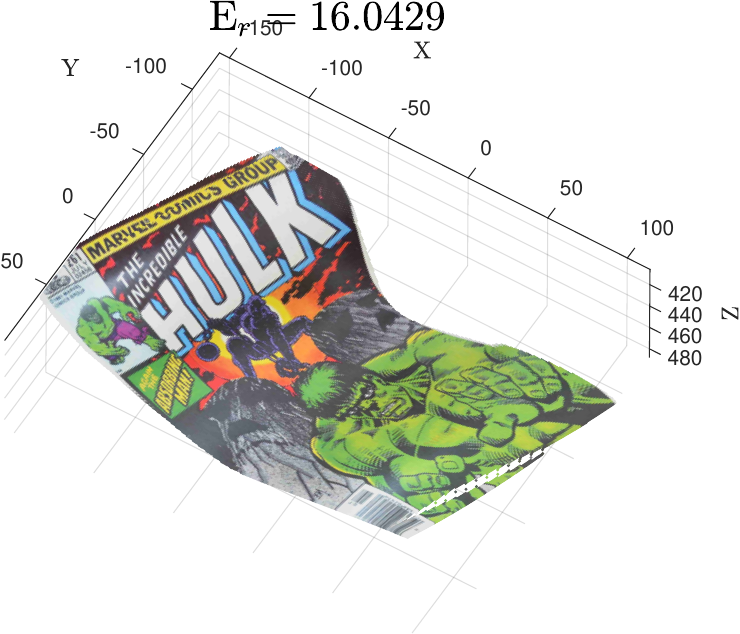}    &    \cincludegraphics[width=\imgwidth]{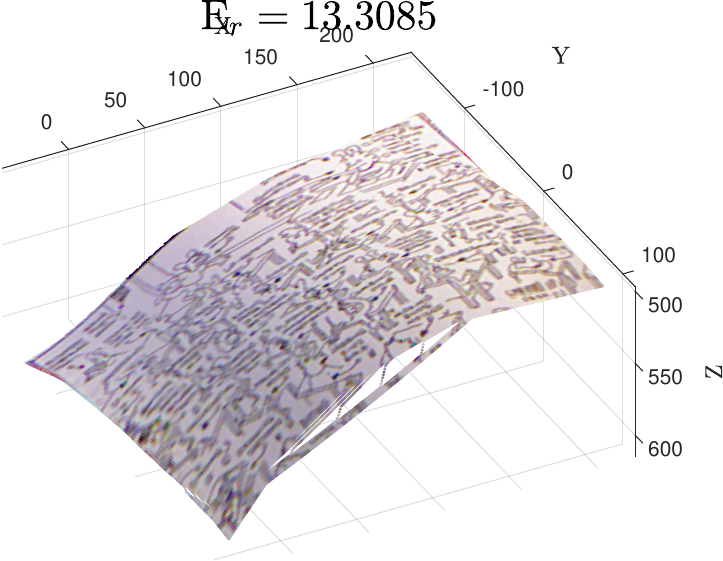}          &  \cincludegraphics[width=\imgwidth]{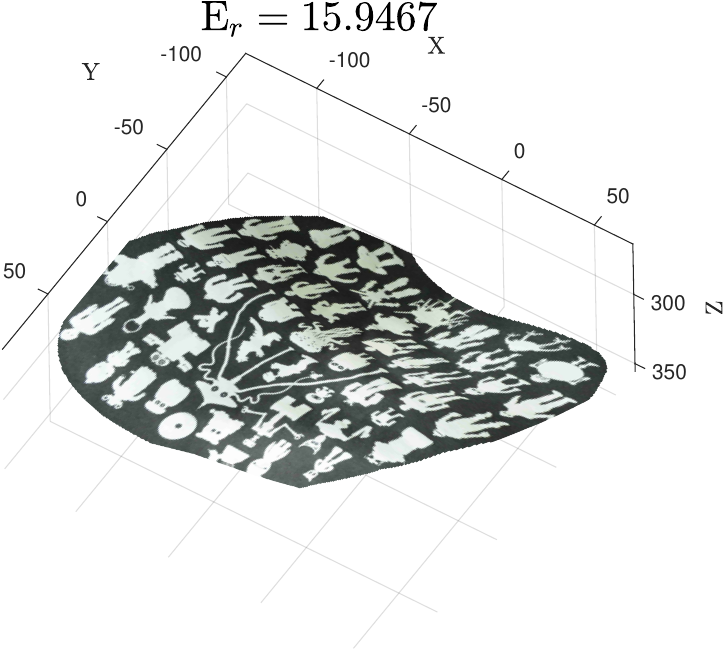}  & \cincludegraphics[width=\imgwidth]{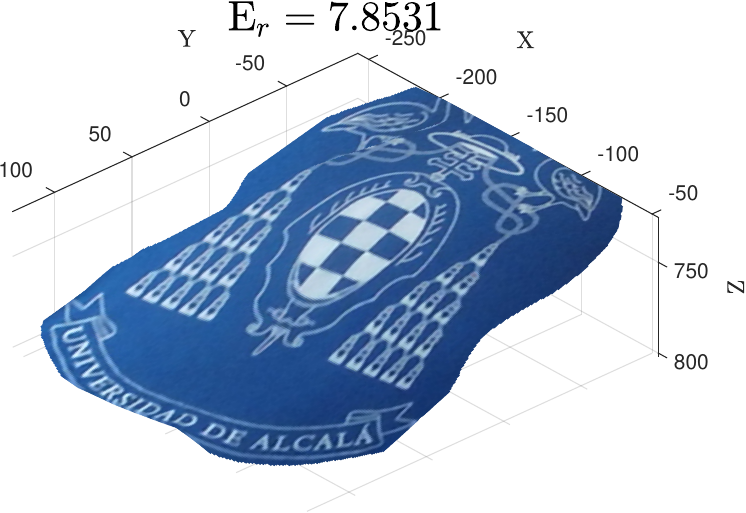} \\ \bottomrule
\end{tabular}
\end{table}

Next we show some qualitative results from our experiments on \y{sb} and \y{sl} with \y{h1}. The qualitative results from a sample frame of each dataset is shown in table~(\ref{tab3}). \y{sb} is a highly curved surface of a balloon being sheared by compressive force. All methods except \y{h1} fails to recover the curvature of the balloon, this is clearly observable in the qualitative results. On the other hand, \y{sl} is a rectangular patch from a nearly-planar surface of a highly stretched leggings. While most of the compared methods, except \cite{ji2017maximizing}, reconstructs the surface reasonably accurately, \cite{chhatkuli2014non}, \cite{hamsici2012learning} and \cite{chhatkuli2017inextensible} reconstructs a mildly curved surface, which is not similar to \y{gt}. In the example of \y{sl} shown in table~(\ref{tab3}), \y{q1} is already more accurate than the other methods, while \y{h1} improves the accuracy, over and above what is achieved by quasi-isometry.

\newcommand{\imgwidthP}{0.9cm}
\newcommand{\imgheightP}{2.0cm}
\newcommand{\imgheightC}{1.4cm}
\newcommand{\imgheightD}{1.2cm}
\begin{table}[]
\makegapedcells
\centering
\caption{Results on \y{sb} and \y{sl} (`$\times$' for methods that failed)}
\label{tab3}
\begin{tabular}{ccc}
\toprule
& \textbf{\y{sb}} & \textbf{\y{sl}}  \\ \hline
\multicolumn{1}{c|}{\small{Input}} &   \dincludegraphics[height=\imgheightP,trim=20cm 10cm 30cm 10cm,clip]{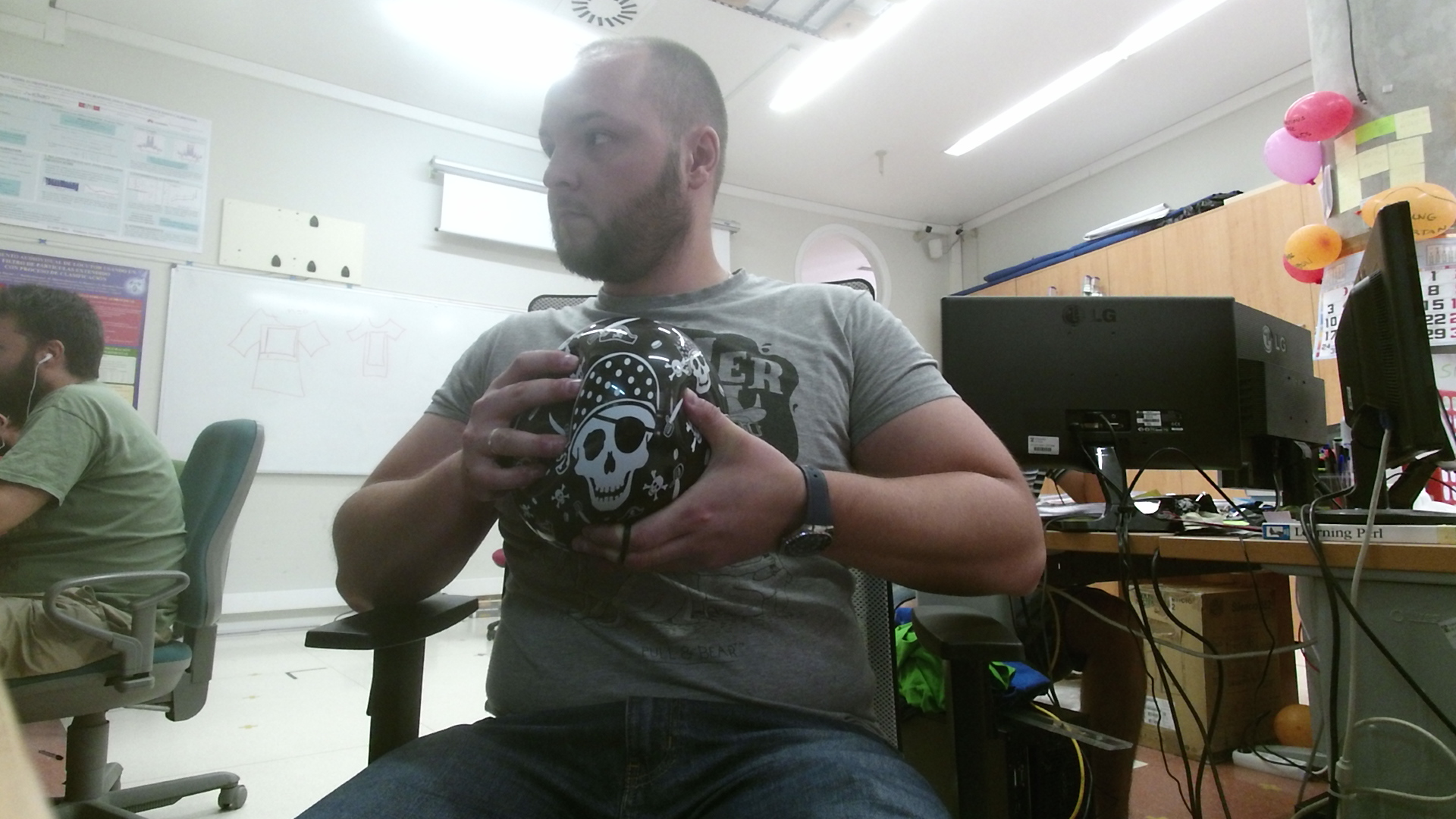}         &     \dincludegraphics[height=\imgheightP,trim=10cm 0 30cm 0,clip]{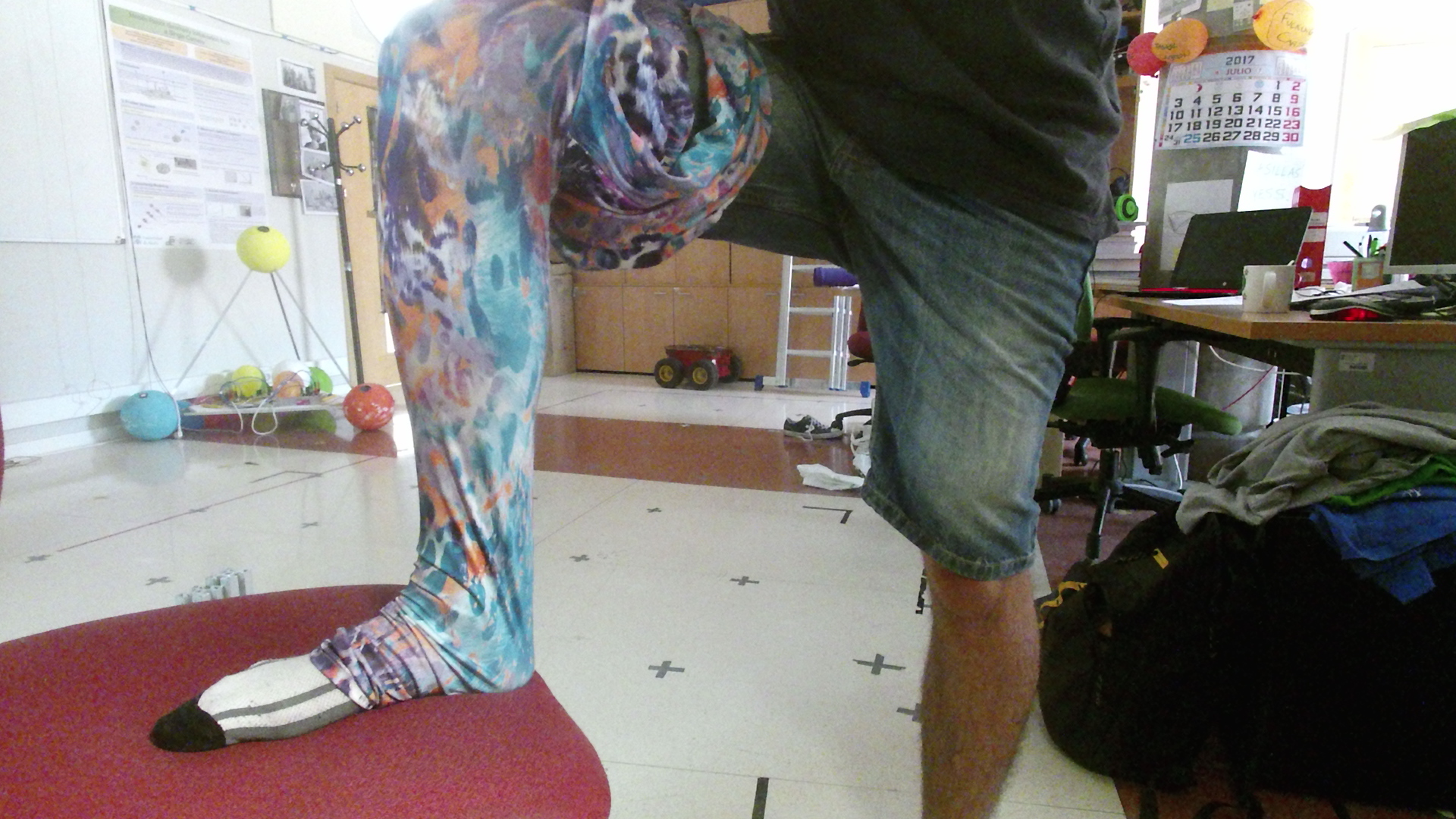}  \\ \hline  
\multicolumn{1}{c|}{\y{gt}} &   \dincludegraphics[height=\imgheightC]{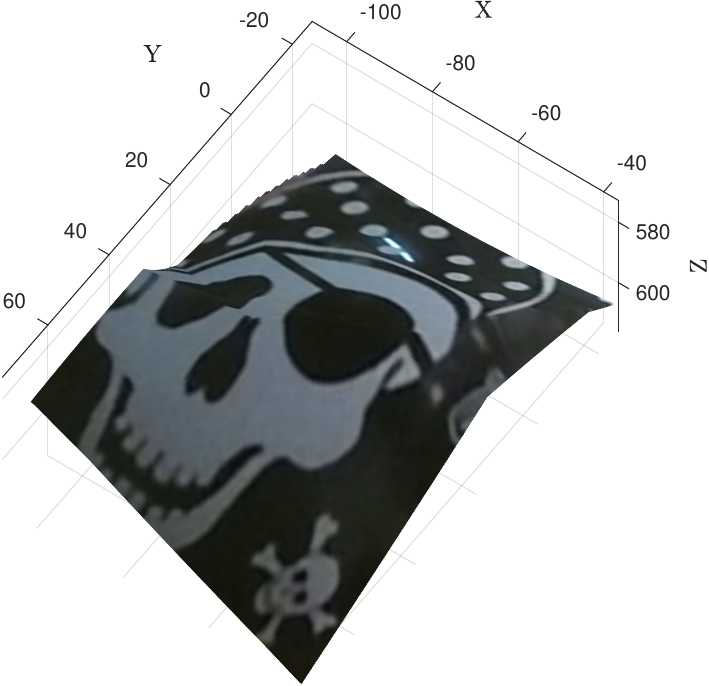}         &     \dincludegraphics[height=\imgheightC]{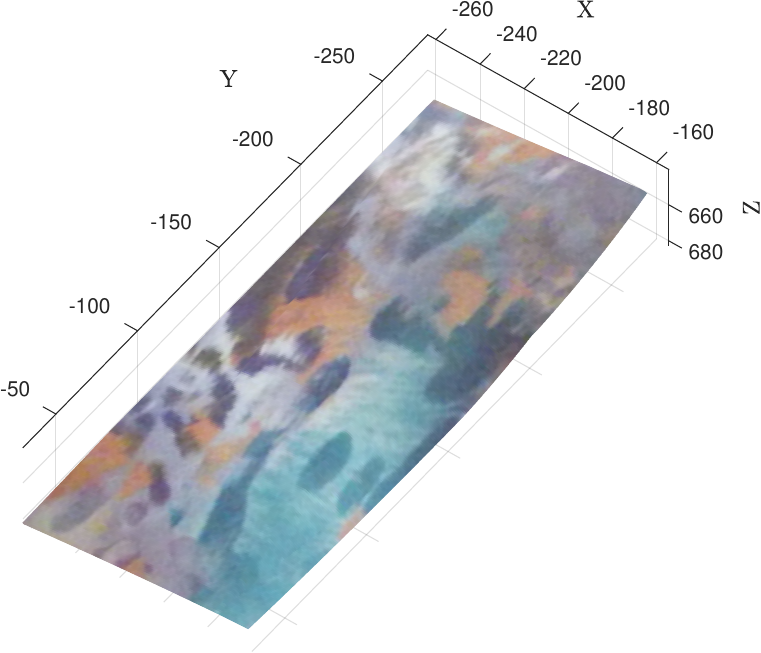}  \\ \hline  
\multicolumn{1}{c|}{\y{h1}} &   \dincludegraphics[height=\imgheightC]{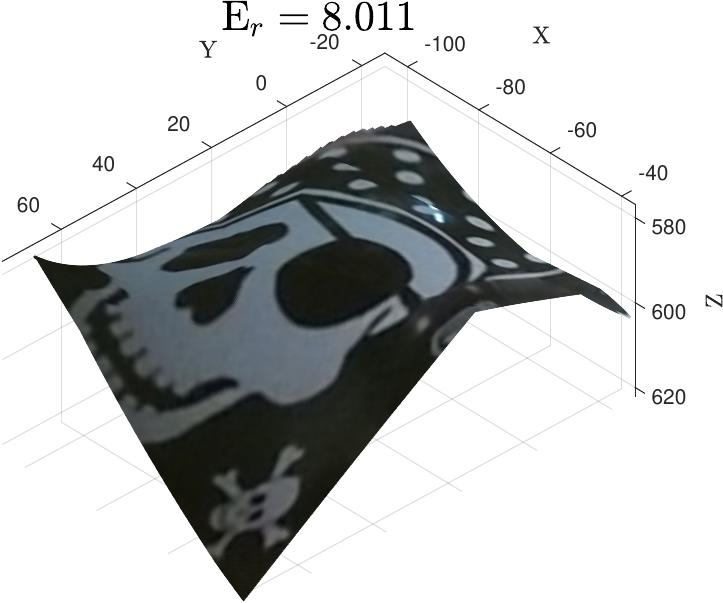}         &     \dincludegraphics[height=\imgheightC]{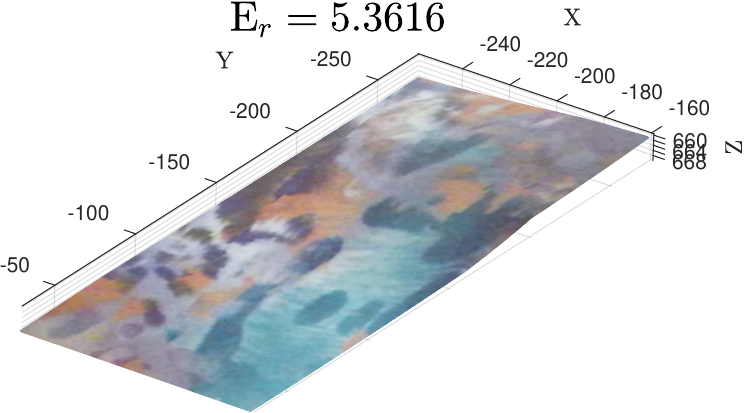}  \\ \hline  
\multicolumn{1}{c|}{\y{h2}} &   \dincludegraphics[height=\imgheightC]{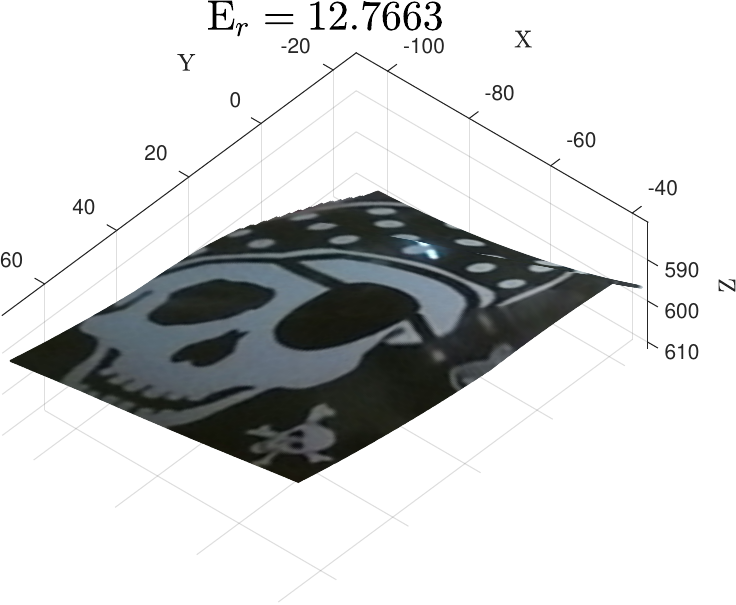}         &     \dincludegraphics[height=\imgheightD]{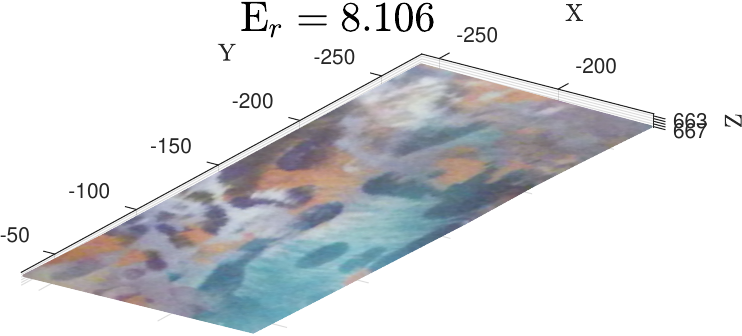}  \\ \hline  
\multicolumn{1}{c|}{\y{q1}} &   \dincludegraphics[height=\imgheightC]{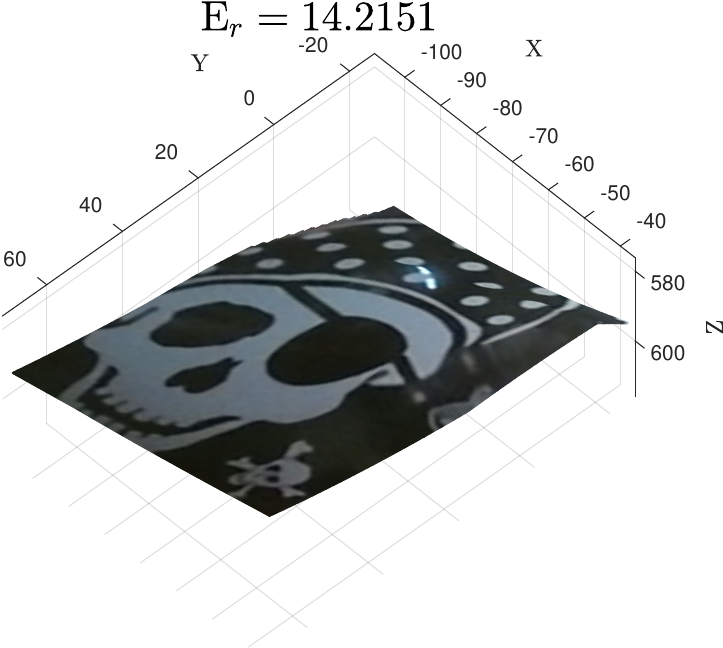}         &     \dincludegraphics[height=\imgheightC]{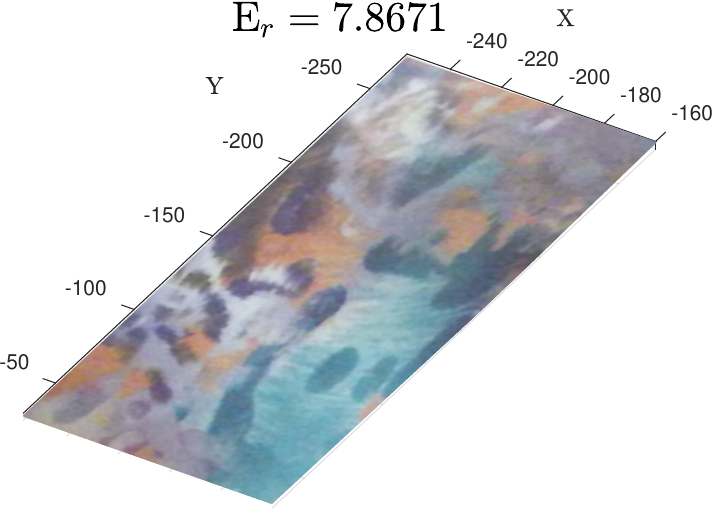} \\ \hline  
\multicolumn{1}{c|}{\y{q2}} &   \dincludegraphics[height=\imgheightC]{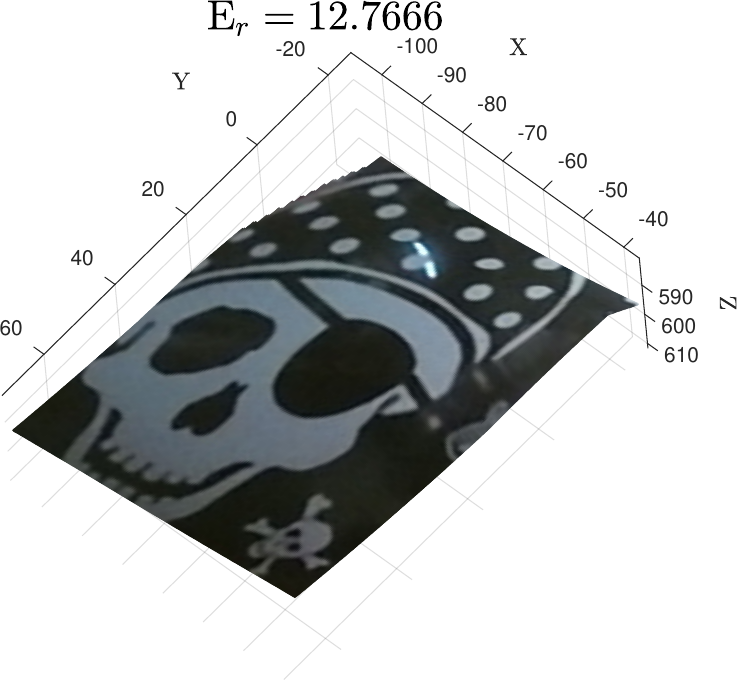}         &     \dincludegraphics[height=\imgheightC]{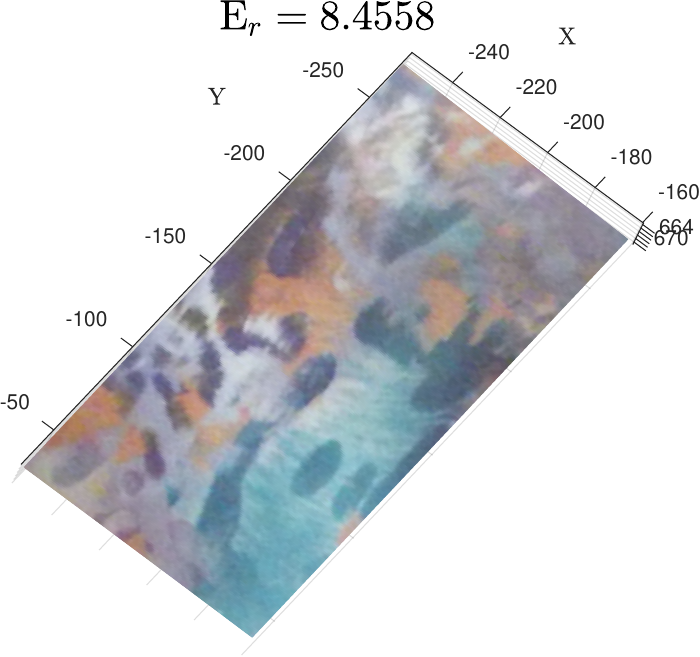} \\ \hline  
\multicolumn{1}{c|}{\cite{chhatkuli2014non}} &   \dincludegraphics[height=\imgheightC]{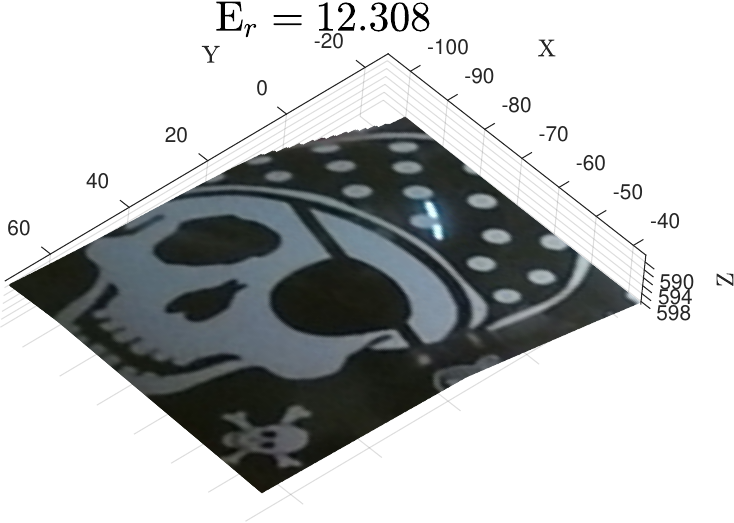}         &     \dincludegraphics[height=\imgheightC]{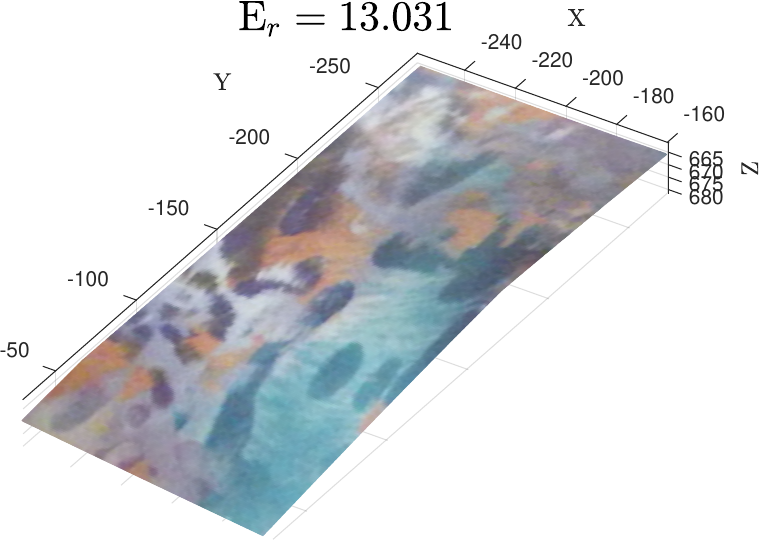}  \\ \hline  
\multicolumn{1}{c|}{\cite{hamsici2012learning}} &   \dincludegraphics[height=\imgheightC]{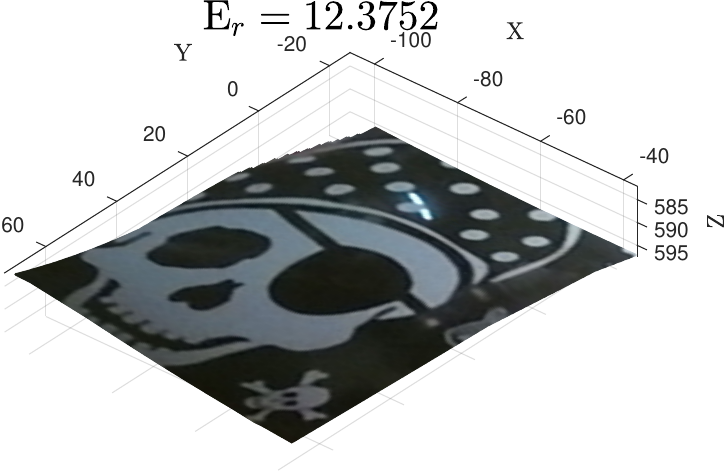}         &     \dincludegraphics[height=\imgheightC]{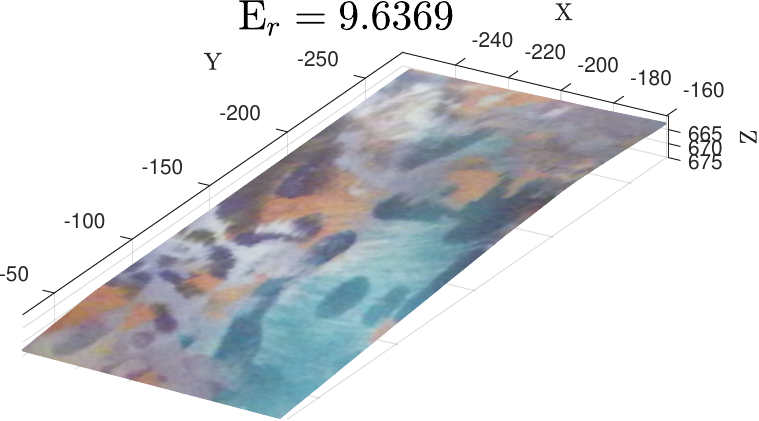}  \\ \hline  
\multicolumn{1}{c|}{\cite{chhatkuli2017inextensible}} &   \dincludegraphics[height=\imgheightC]{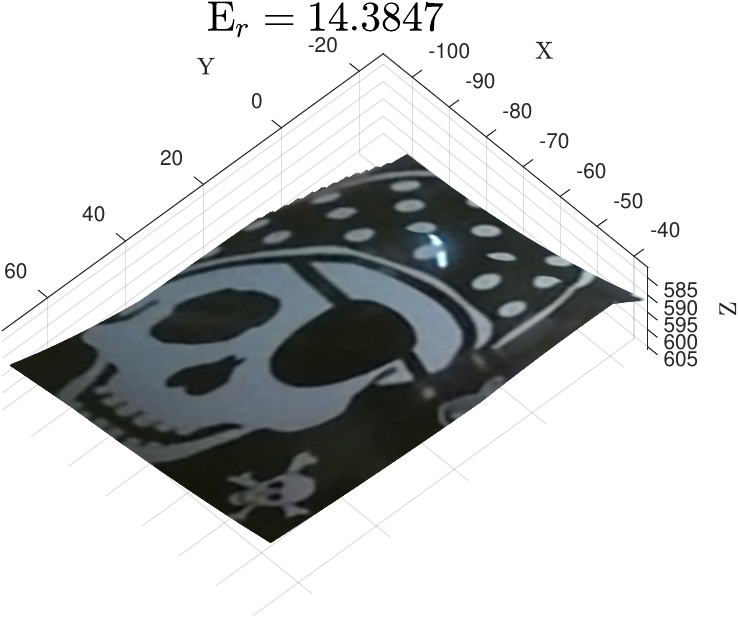}         &     \dincludegraphics[height=\imgheightC]{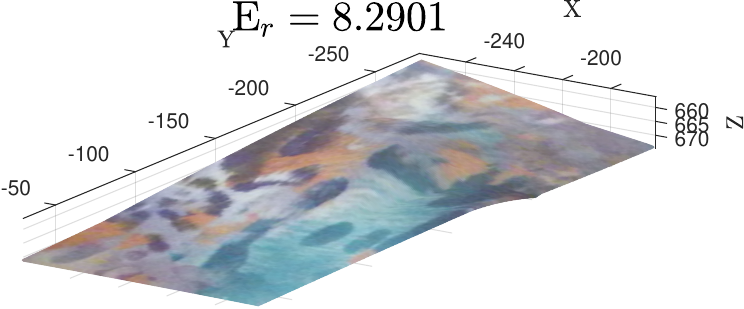}  \\ \hline  
\multicolumn{1}{c|}{\cite{ji2017maximizing}} &   \dincludegraphics[height=\imgheightC]{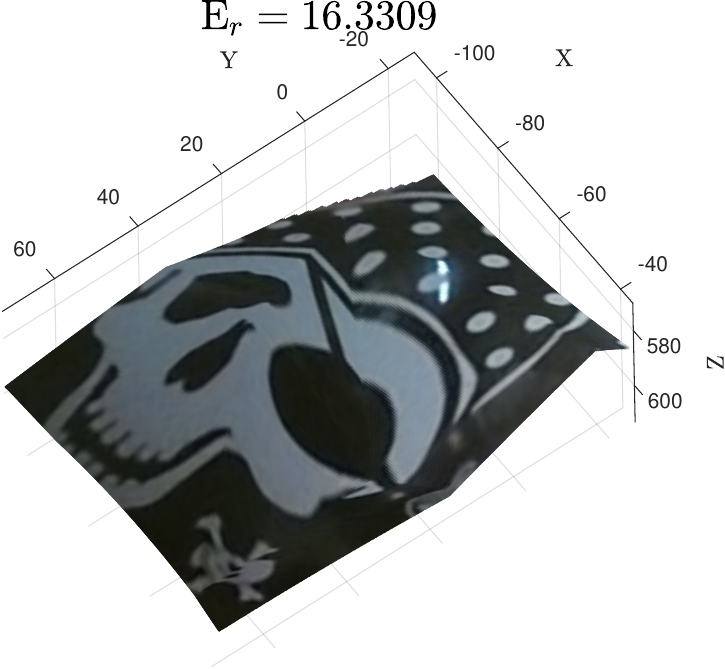}         &     \dincludegraphics[height=\imgheightC]{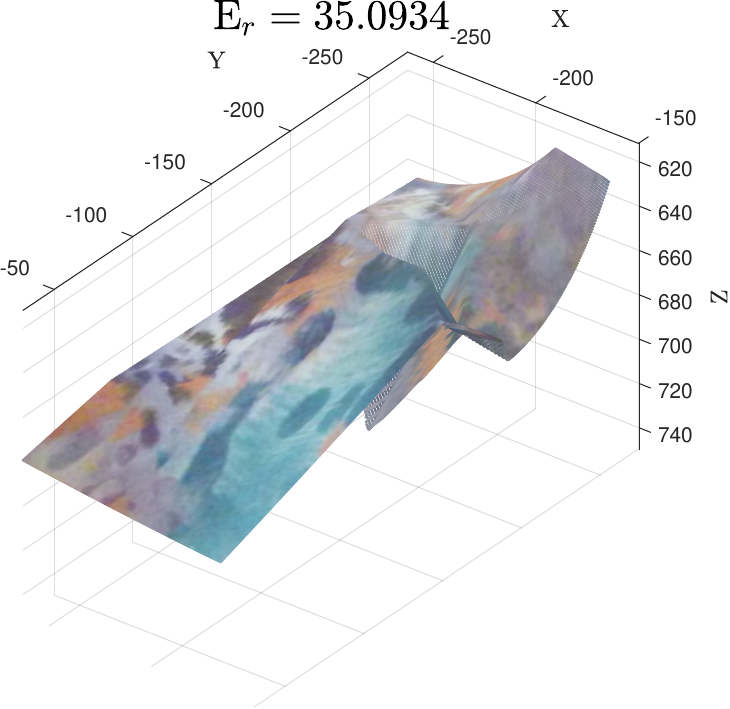}  \\ \hline  
\multicolumn{1}{c|}{\cite{gotardo2011kernel}} &   \dincludegraphics[height=\imgheightC]{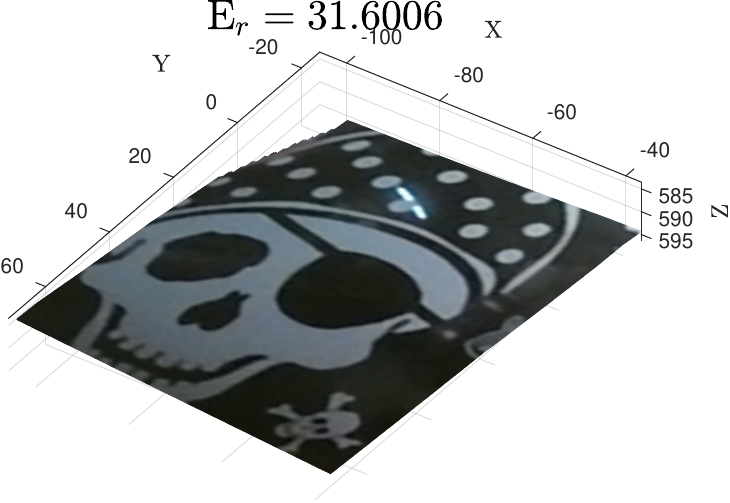}         &     $\times$
\end{tabular}
\end{table}

\begin{figure*}[t]
    \centering
    \begin{overpic}[width=14cm, trim=0 0 0 0,clip]{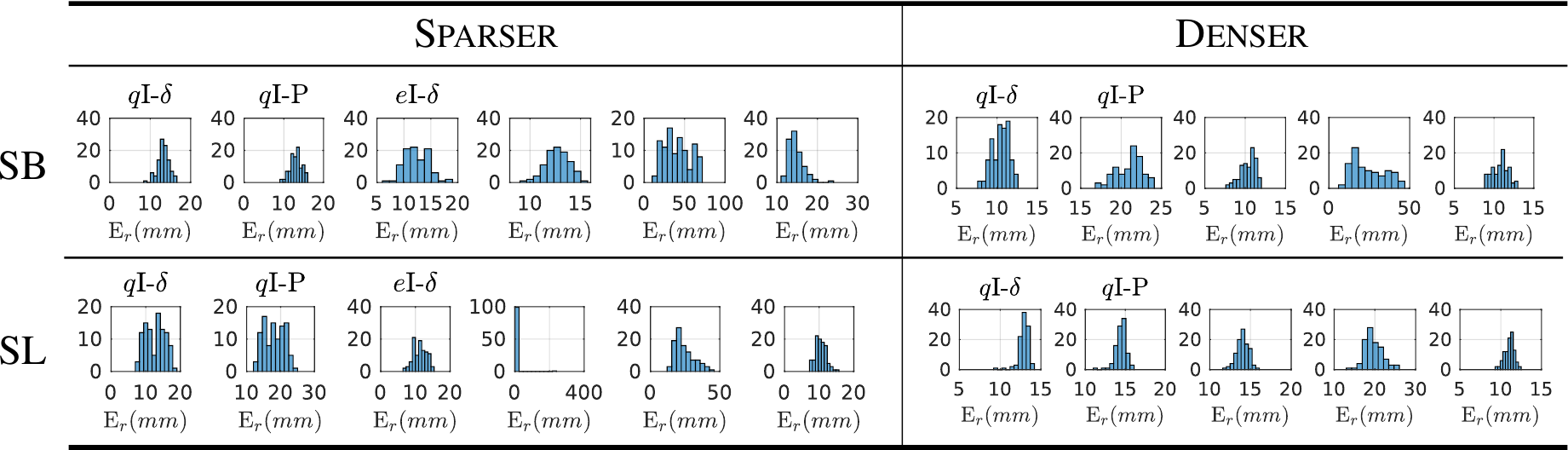}
    \put(134,88){\scalebox{.7}{\cite{chhatkuli2017inextensible}}}
    \put(167,88){\scalebox{.7}{\cite{ji2017maximizing}}}
    \put(202,88){\scalebox{.7}{\cite{chhatkuli2014non}}}
    
    \put(310,88){\scalebox{.7}{\cite{chhatkuli2017inextensible}}}
    \put(342,88){\scalebox{.7}{\cite{ji2017maximizing}}}
    \put(375,88){\scalebox{.7}{\cite{chhatkuli2014non}}}
    
    \put(134,41){\scalebox{.7}{\cite{chhatkuli2017inextensible}}}
    \put(167,41){\scalebox{.7}{\cite{ji2017maximizing}}}
    \put(202,41){\scalebox{.7}{\cite{chhatkuli2014non}}}
    
    \put(310,41){\scalebox{.7}{\cite{chhatkuli2017inextensible}}}
    \put(342,41){\scalebox{.7}{\cite{ji2017maximizing}}}
    \put(375,41){\scalebox{.7}{\cite{chhatkuli2014non}}}
    \end{overpic}
    \caption{The sparser plots are obtained by repeatedly and sparsely subsampling a set of large correspondences while the denser plots are obtained by densely subsampling the same set of large correspondences}
    \label{fig_dense}
\end{figure*}

\subsection{Effect of sparsity on accuracy of \y{h1}}
In the main article, we had briefly talked about the effect of sparsity on accuracy of \y{h1}. A non-chordal \y{psd} matrix completion without additional measures for handling sparsity can lead to sub-optimal solutions \cite{fukuda2001exploiting}. Unfortunately, for denser correspondences, maintaining a fully connected graph structure via $\mathcal{E}_3$ lead to extremely large $\{\mathfrak{T}_i\}$, which is computationally infeasible. We show this via a small-scale, synthetic example in figure~(\ref{fig_iso_chord}). As we increase the number of point correspondences in the input data while maintaining a constant $|\mathcal{E}_3|$ of 2, the accuracy decreases with increasing density, as the number of cliques in $\{\mathfrak{T}_i\}$ increases. In such a case, there can certainly exist situations where \y{s1} or \y{s2} dominates in accuracy over \y{h1}. However, this decrease in accuracy can be compensated by adapting $|\mathcal{E}_3|$, such that $\{\mathfrak{T}_i\}$ remains fully connected, as shown in figure~(\ref{fig_iso_chord}). But this is not generalisable to actual, larger dataset with large point correspondences due to computational limitations. Notably, such an effect of sparsity is absent in \y{s1}, \y{s2}, \y{q1} or \y{q2}, since arbitrarily increasing neighborhood does not create a significant computational bottleneck for these methods.

However, in figure~(\ref{fig_dense}), we show that for the two extensible objects \y{sl} and \y{sb}, the higher accuracy of \y{h1} (from Table~2) is generalisable to different correspondences across multiple resolutions. We randomly sample $\sim50$ point correspondences from \y{sb} and \y{sl} in two different resolutions, the \textit{sparser} one being 11 correspondences and the \textit{denser} one being $\sim$30 points. Over many repetitions, we show that \y{h1} with sparser correspondences form the most accurate method of reconstruction over all resolutions, although the accuracy improvement over \cite{chhatkuli2014non, chhatkuli2017inextensible} are relatively small. But importantly, there is a significant and consistent improvement in accuracy with \y{h1} over \y{q1}, which fits our goal of obtaining a \y{hnr} method that improves upon our proposed \y{qnr} in case of highly stretchable objects.
\begin{figure}[t]
    \centering
    \begin{overpic}[width=5.5cm, trim=0 0 0 0,clip]{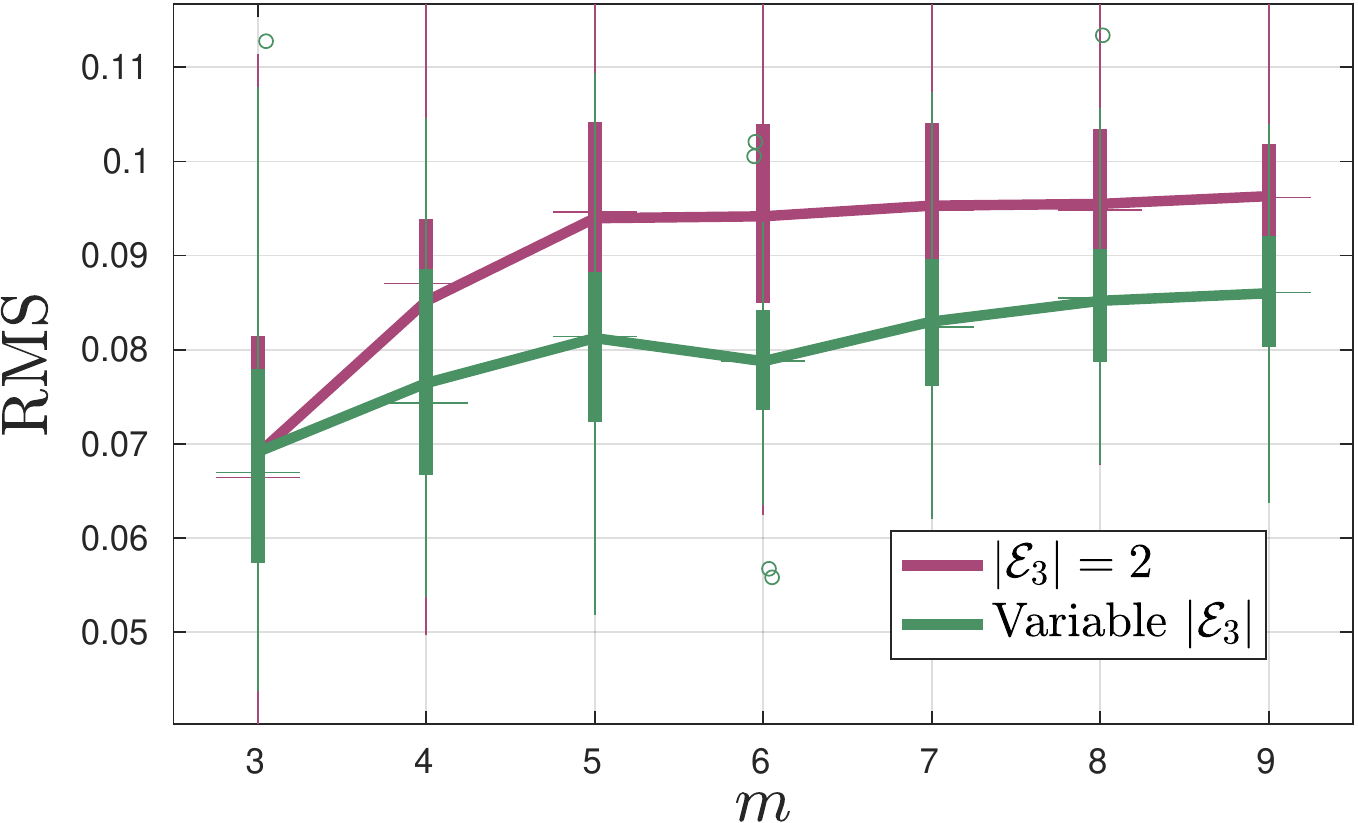}
    \end{overpic}
    \caption{Variation of accuracy with sparsity with two different strategies: one with a constant $\mathcal{E}_3$ of two 2-simplices and the other creates a nearly chordal graph-structure with adaptive $\mathcal{E}_3$}
    \label{fig_iso_chord}
\end{figure}

\subsection{Results from \y{h2}}
With the modification to \y{h2} from equation~(\ref{eqn:hnr_xyz_sol_3}), the \y{rms} on \y{sb} and \y{sl} are 12.48 and 8.35 $mm$ respectively, while the qualitative results have been shown along with other methods in table~(\ref{tab3}). In comparison with table~(2) of the main paper, \y{h2} appears less accurate than \y{h1}. This is surprising, since in most experiments, \y{pp} methods seem to outperform \y{dsl} methods. But for \y{h2}, this drop in accuracy stems from the additional relaxations of section~(\ref{sec_acc}). However, \y{h2} nonetheless marginally improves the results of \y{q2}.

\newcommand{\qw}{3.8cm}
\begin{figure*}[t]
    \centering
    \begin{overpic}[width=\qw, trim=283 40 25 25,clip]{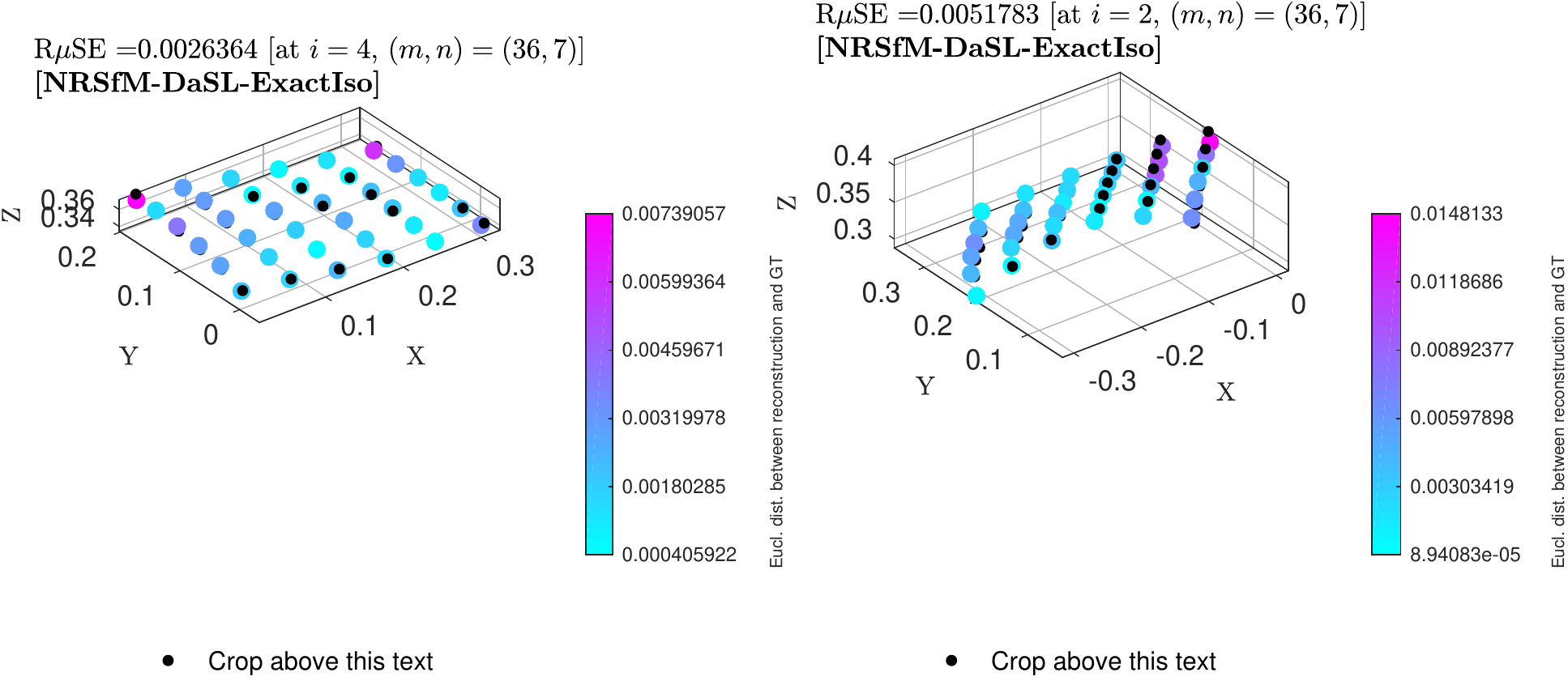}
    \end{overpic}
    \begin{overpic}[width=\qw, trim=283 40 25 25,clip]{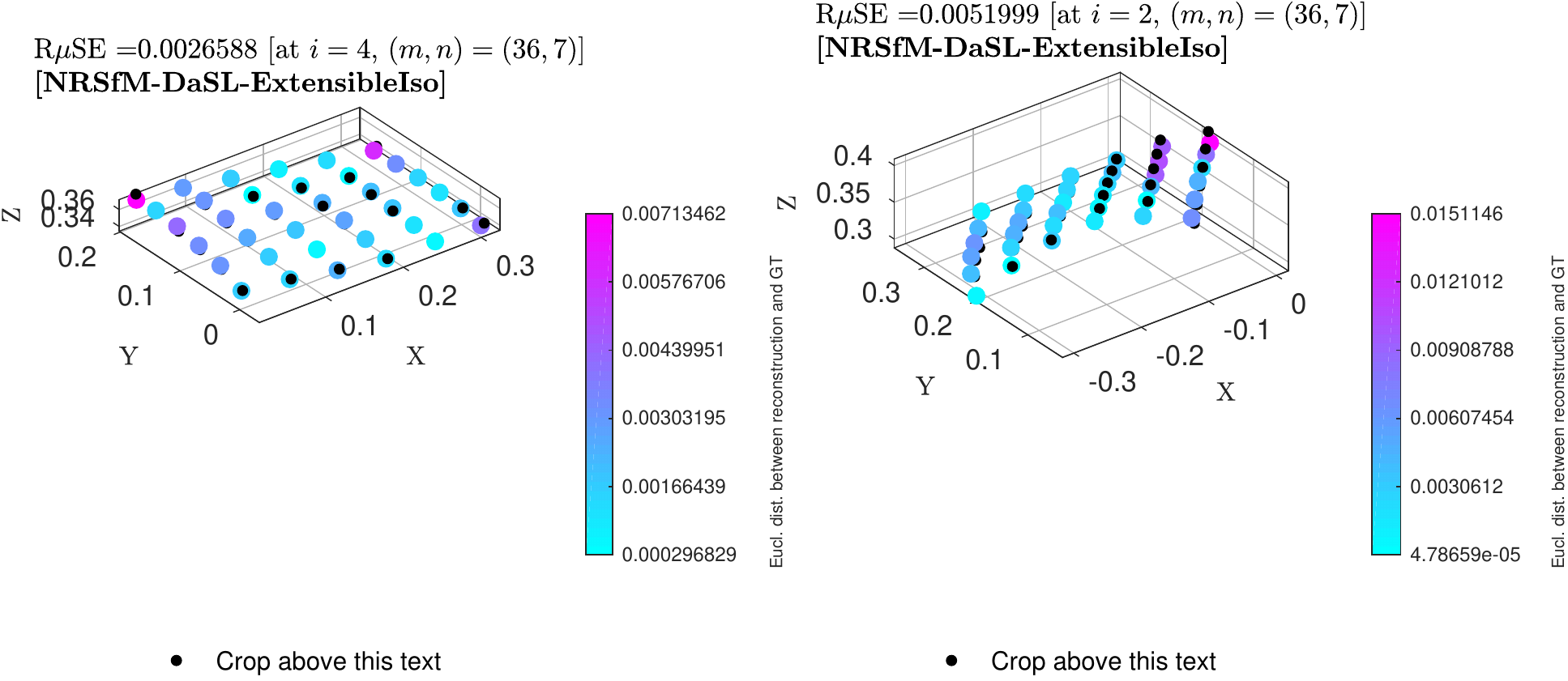}
    \end{overpic}
    \begin{overpic}[width=\qw, trim=283 40 25 25,clip]{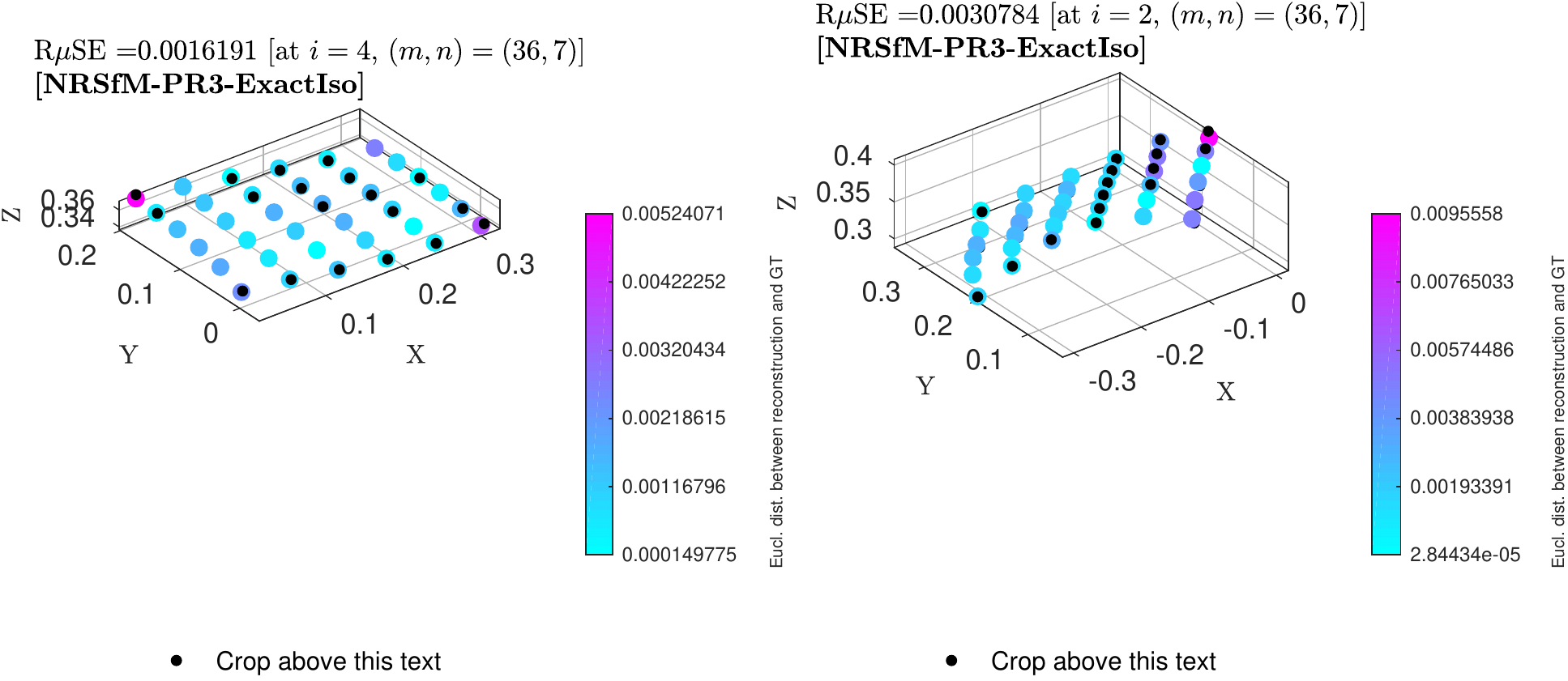}
    \end{overpic}
    \begin{overpic}[width=\qw, trim=283 40 25 25,clip]{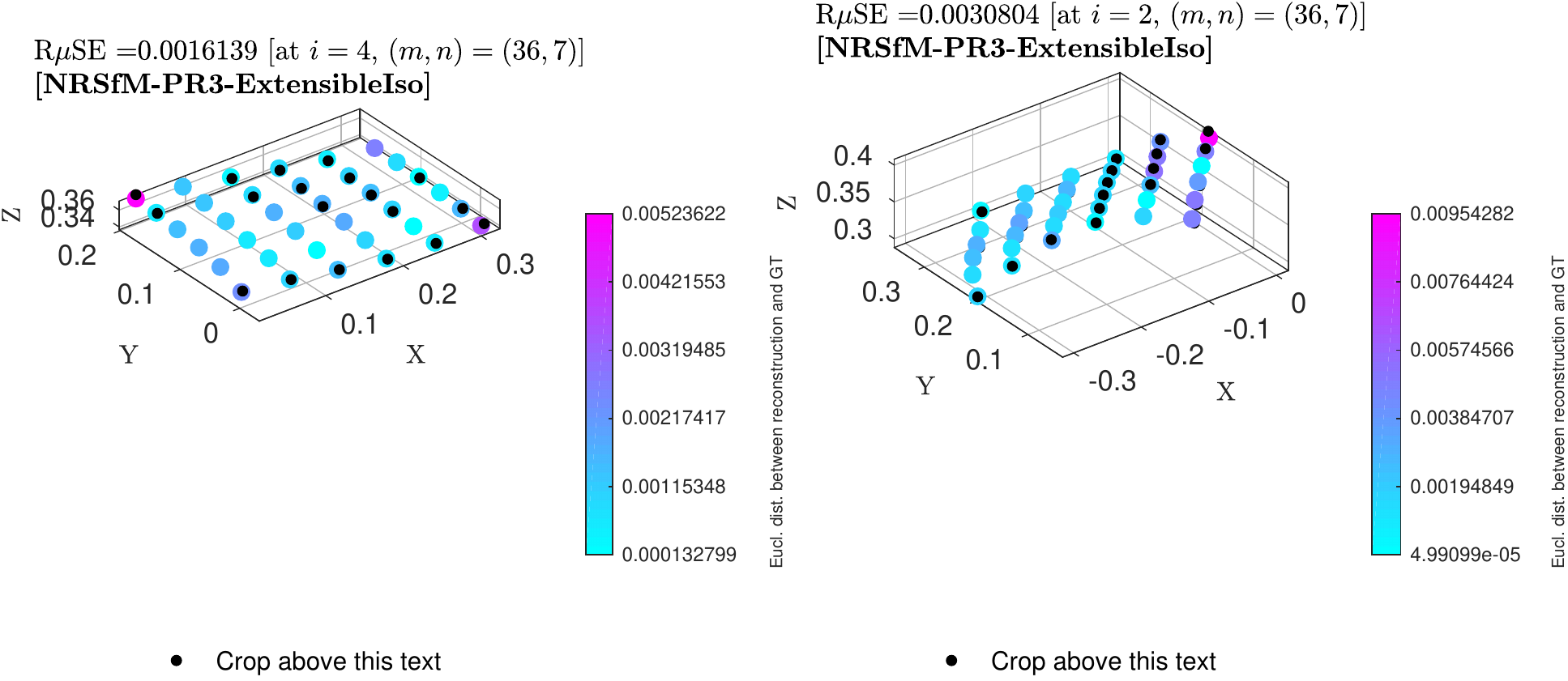}
    \end{overpic}
    
    \begin{overpic}[width=\qw, trim=286 40 25 25,clip]{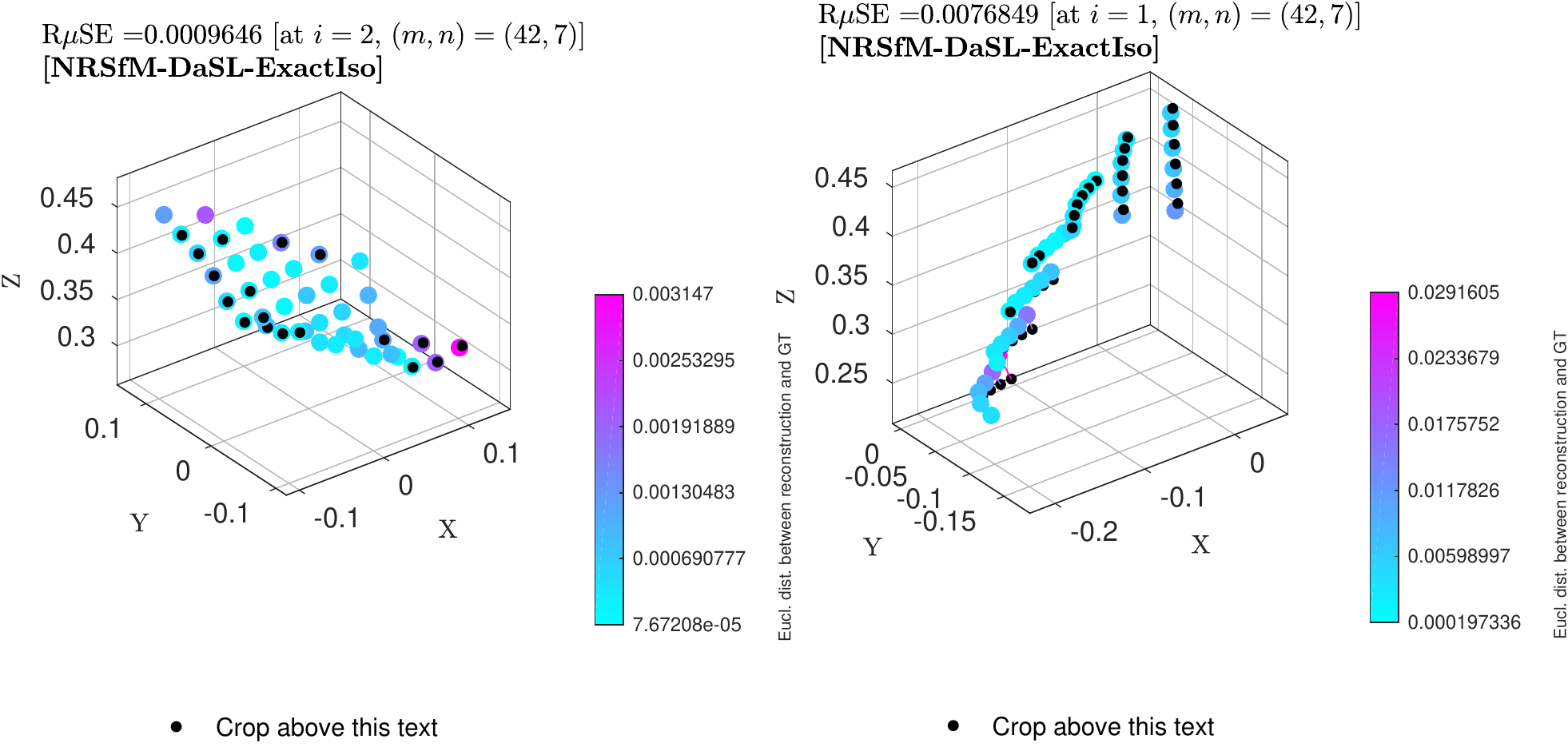}
    \end{overpic}
    \begin{overpic}[width=\qw, trim=286 40 25 25,clip]{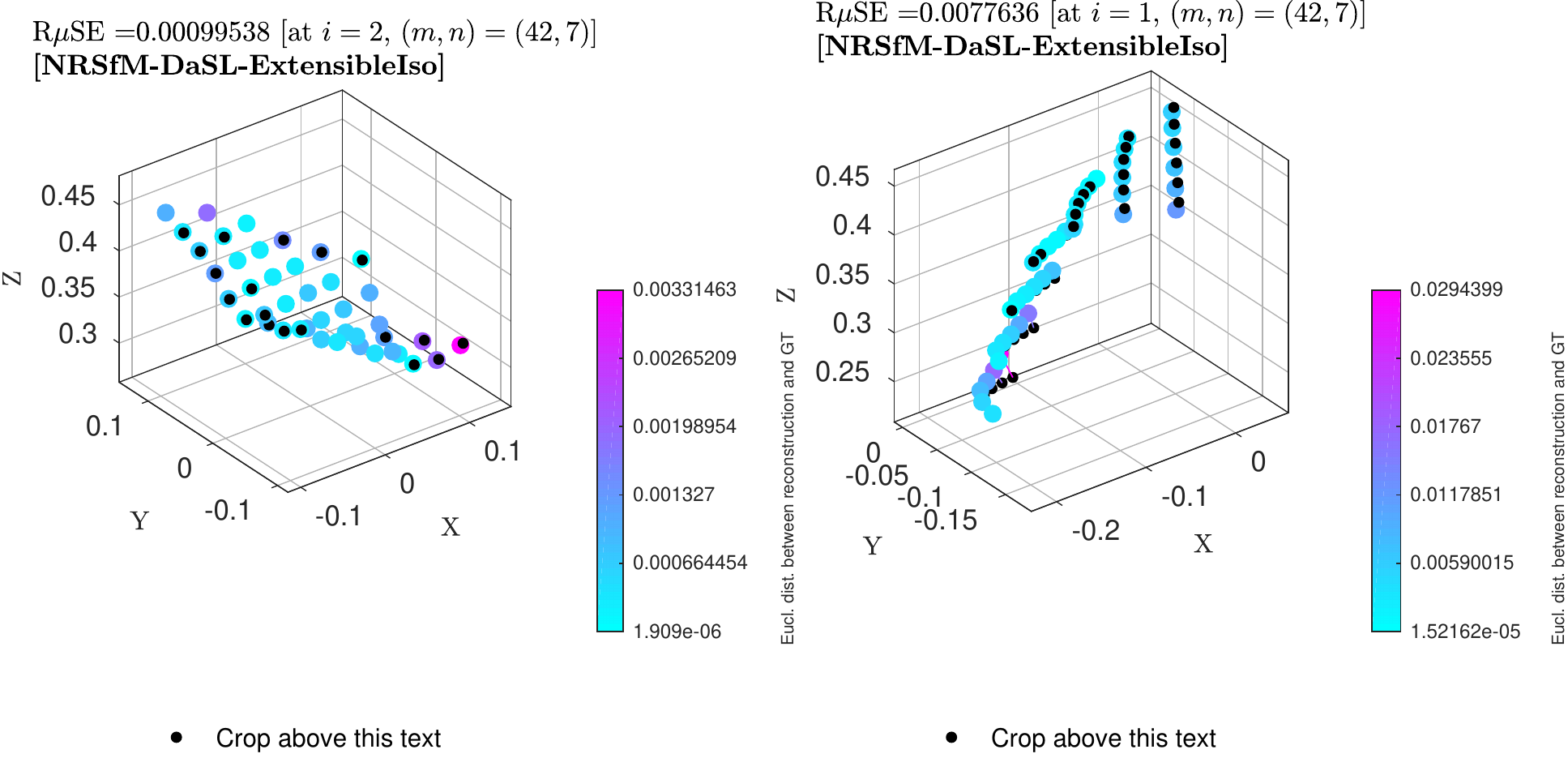}
    \end{overpic}
    \begin{overpic}[width=\qw, trim=286 40 25 25,clip]{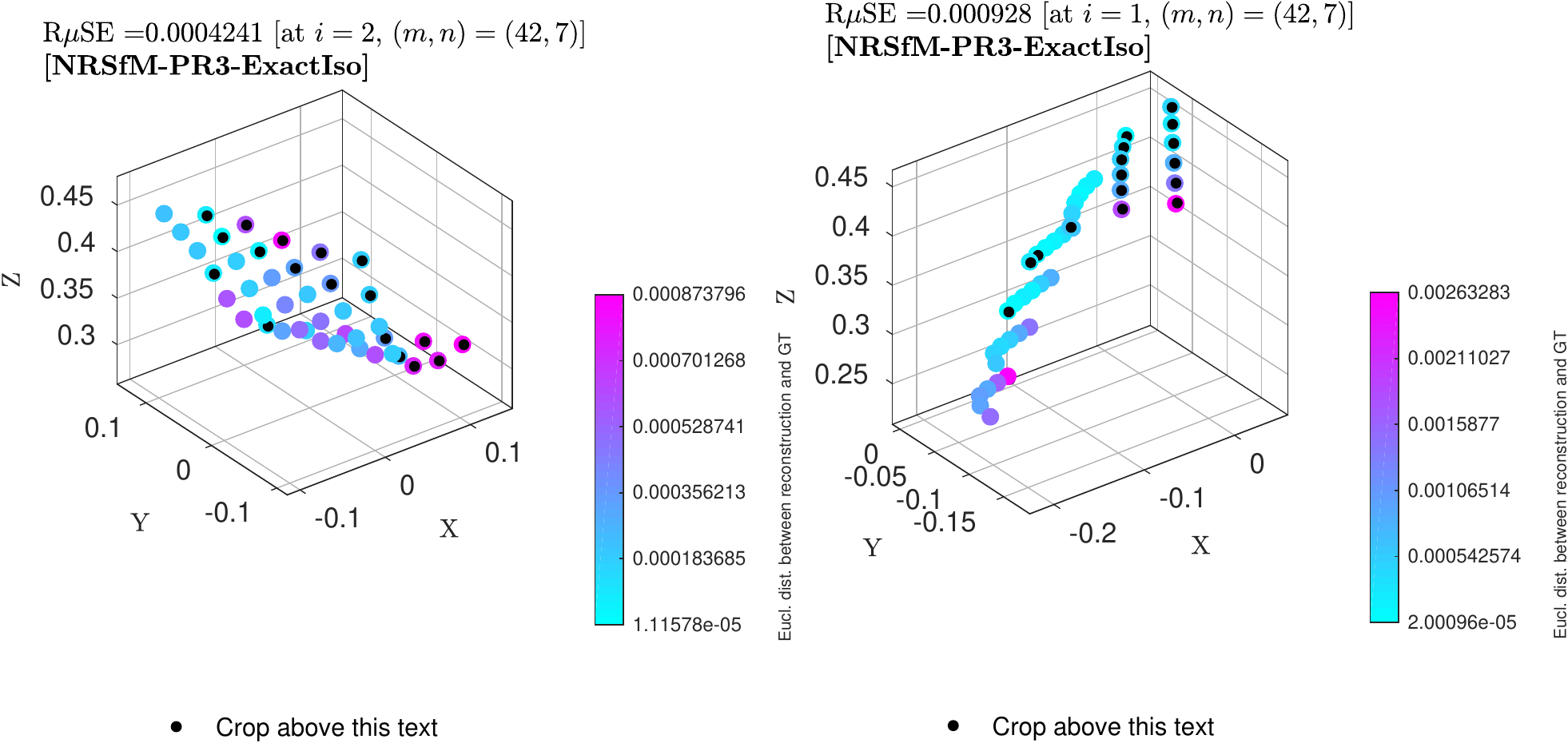}
    \end{overpic}
    \begin{overpic}[width=\qw, trim=286 40 25 25,clip]{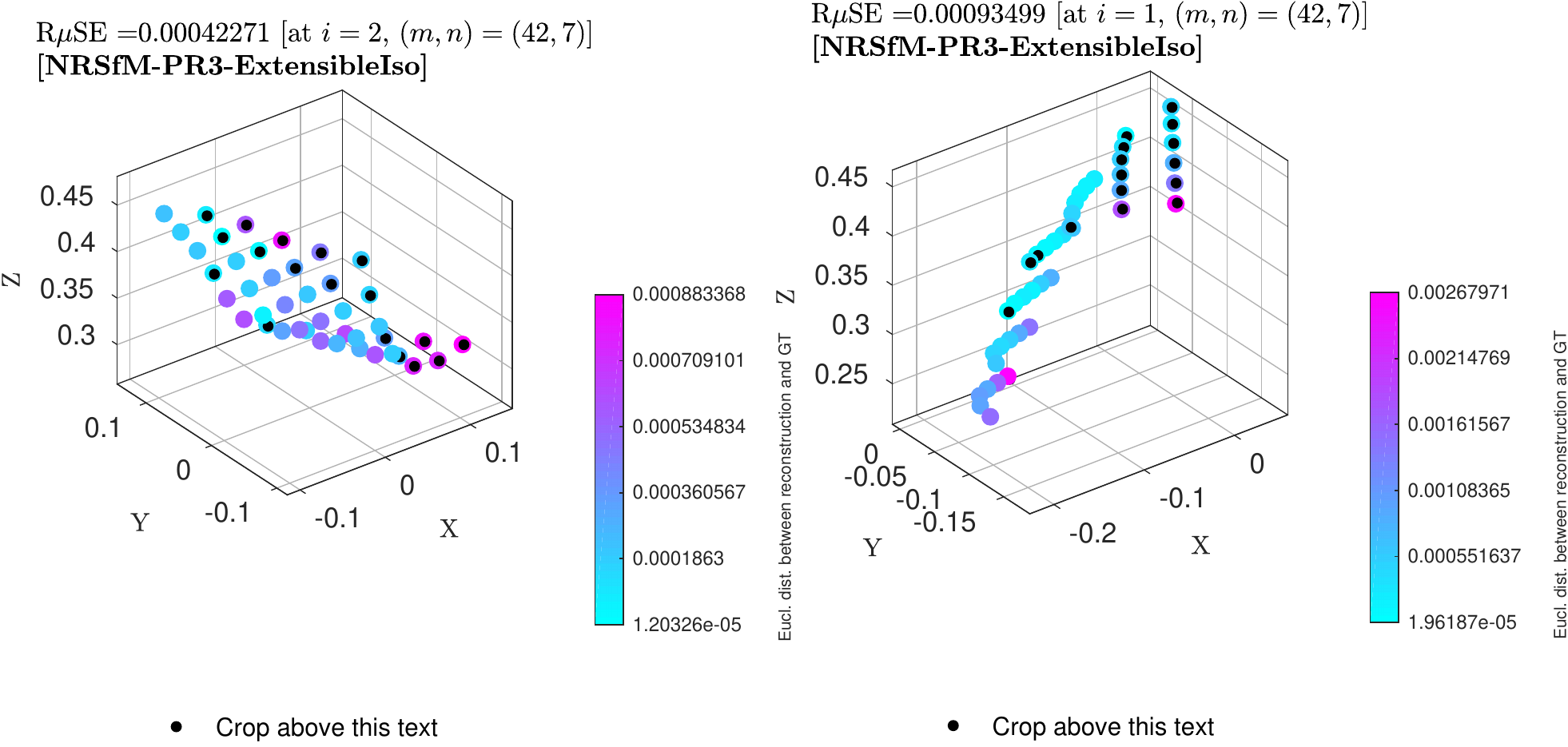}
    \end{overpic}    
    
    \begin{overpic}[width=\qw, trim=10 40 286 25,clip]{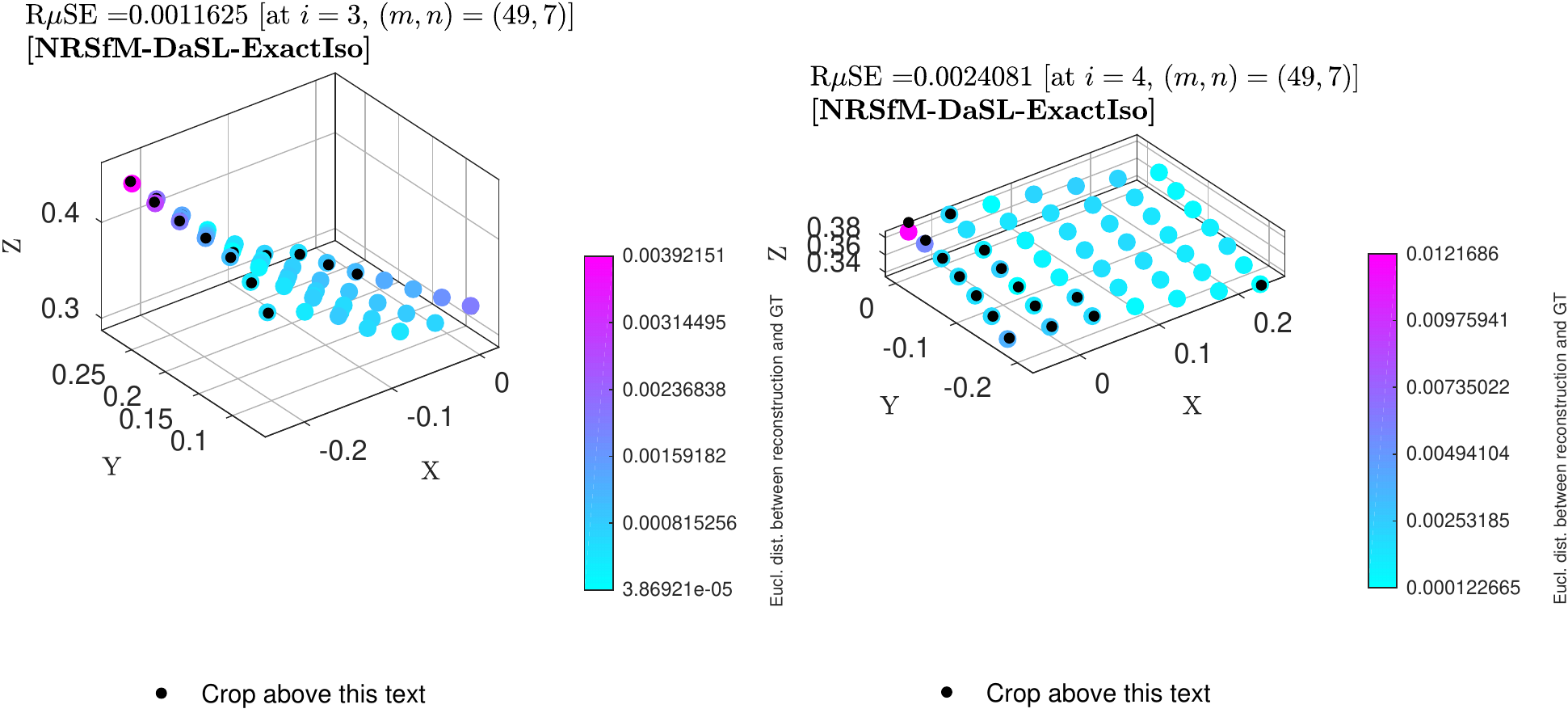}
    \end{overpic}
    \begin{overpic}[width=\qw, trim=10 40 286 25,clip]{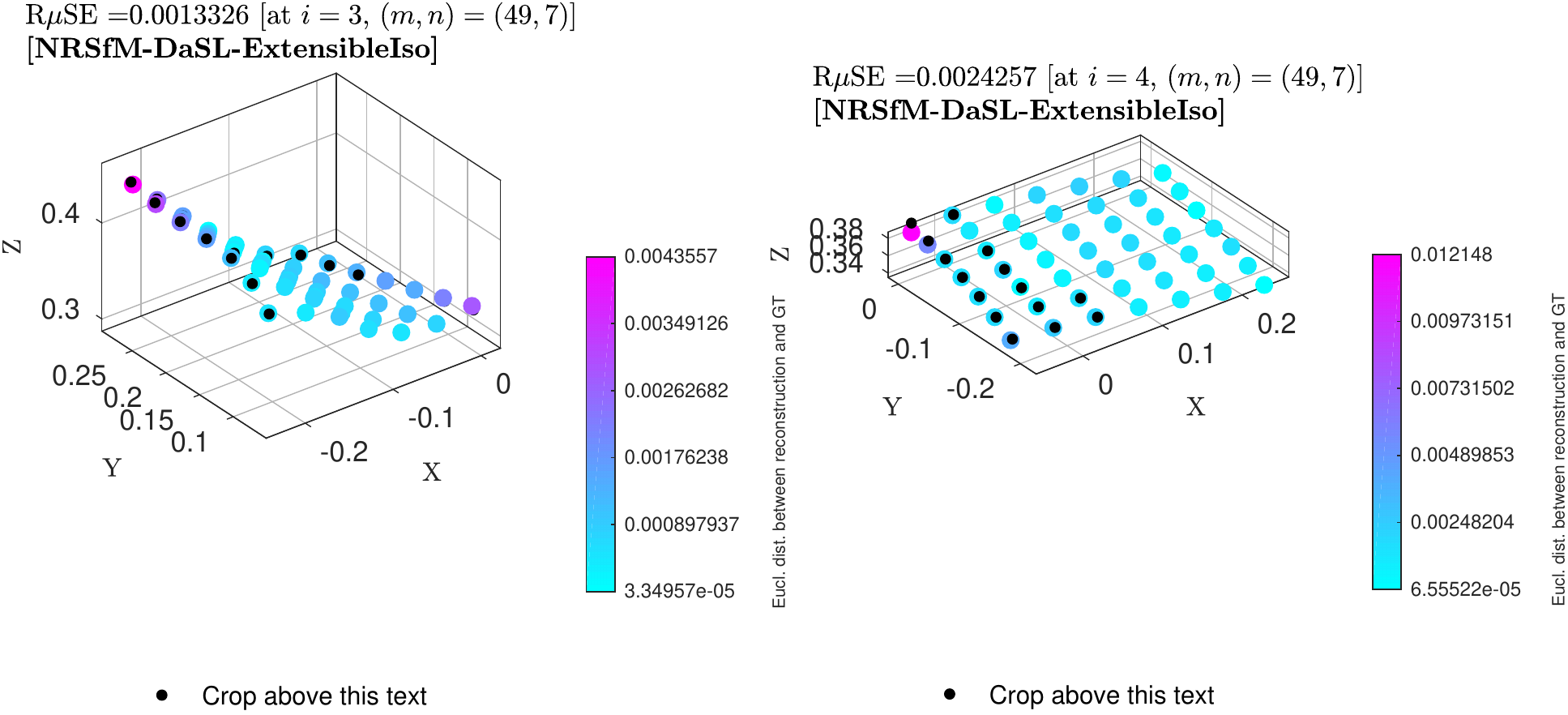}
    \end{overpic}
    \begin{overpic}[width=\qw, trim=10 40 286 25,clip]{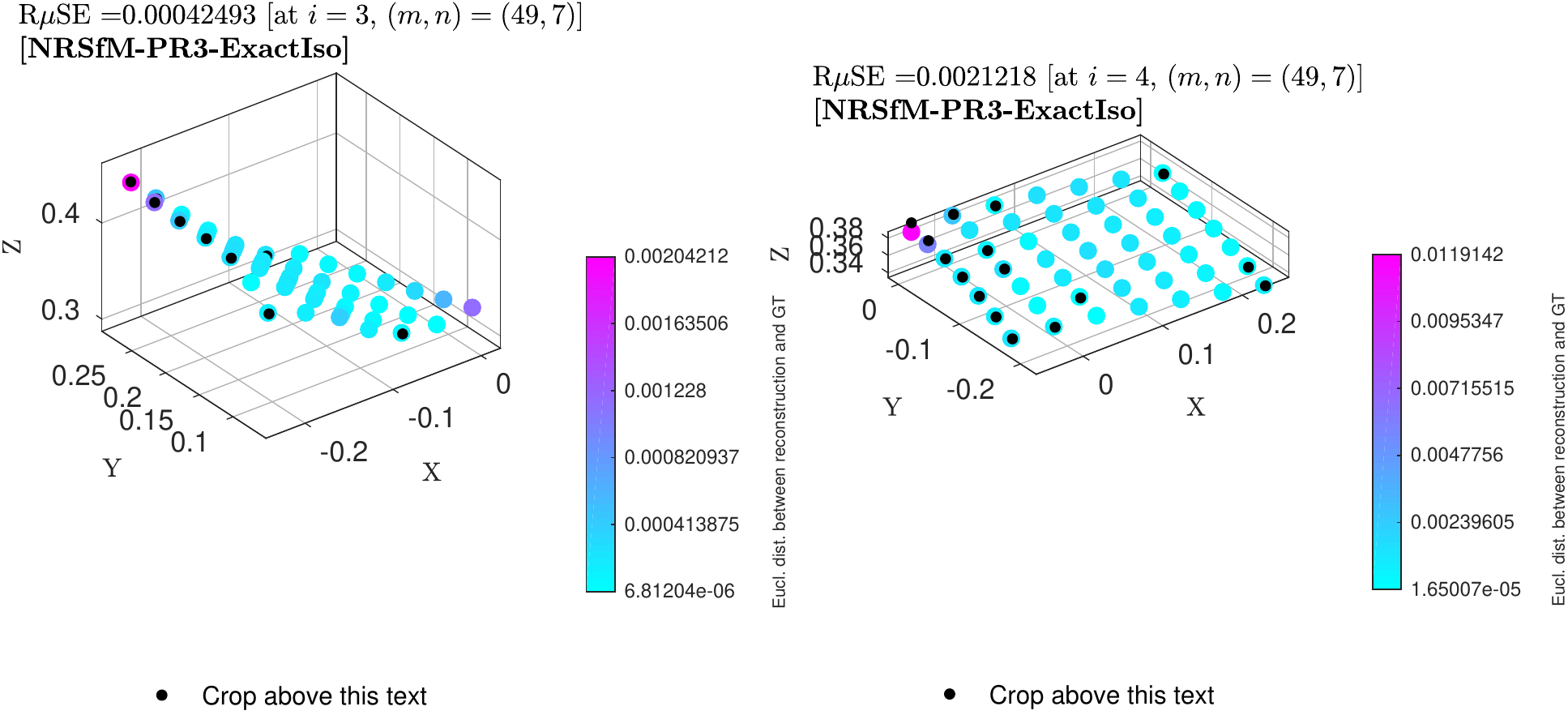}
    \end{overpic}
    \begin{overpic}[width=\qw, trim=10 40 286 25,clip]{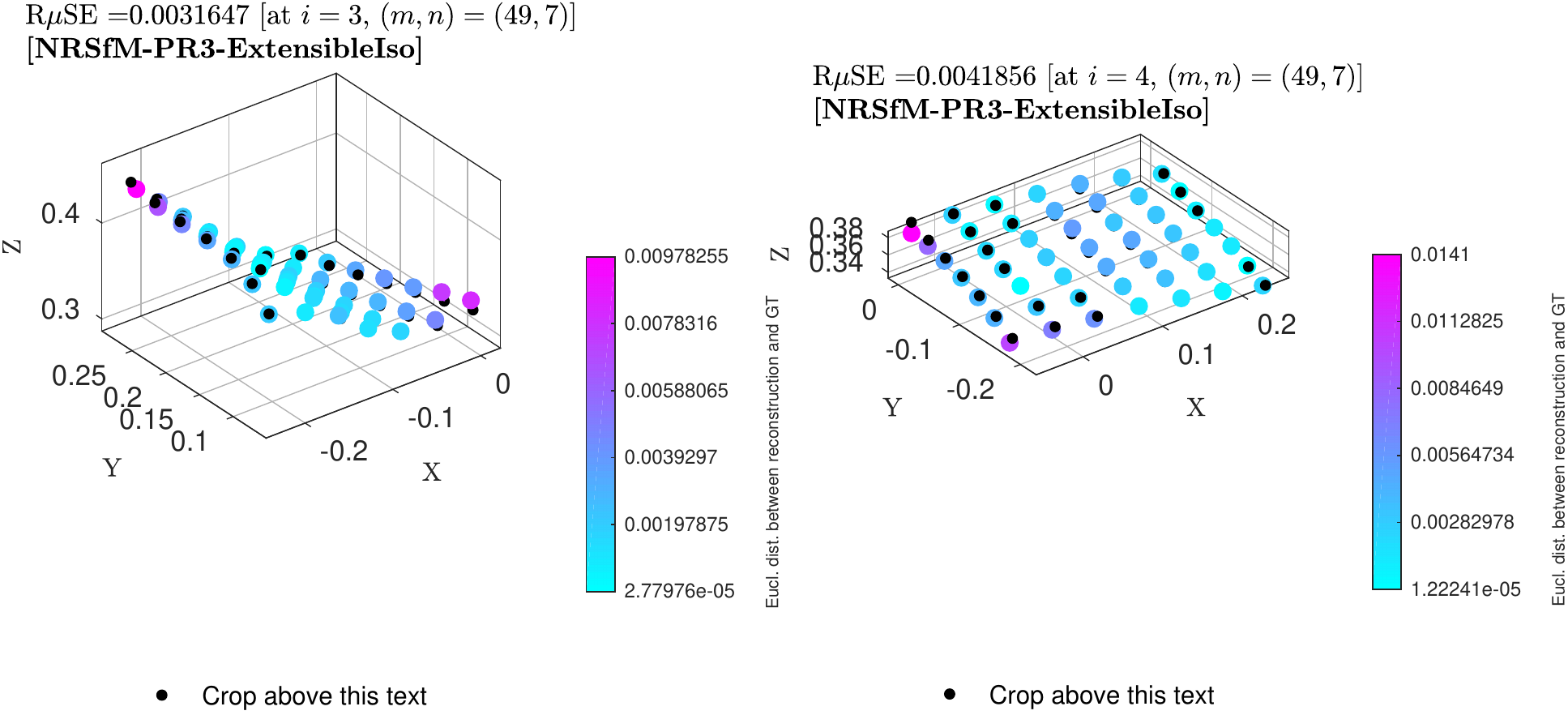}
    \end{overpic}
    
    \begin{overpic}[width=\qw, trim=10 40 286 25,clip]{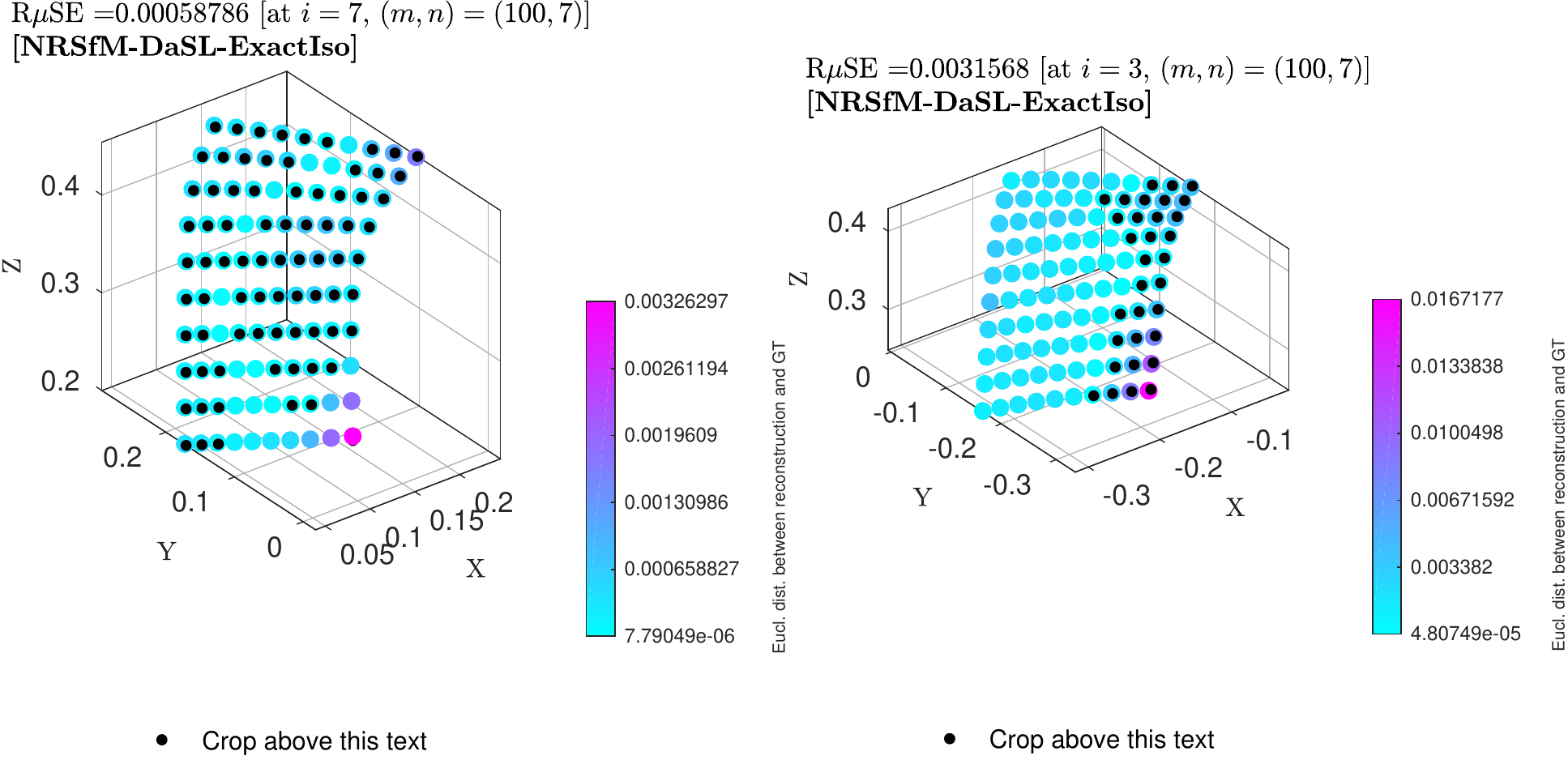}
    \end{overpic}
    \begin{overpic}[width=\qw, trim=10 40 286 25,clip]{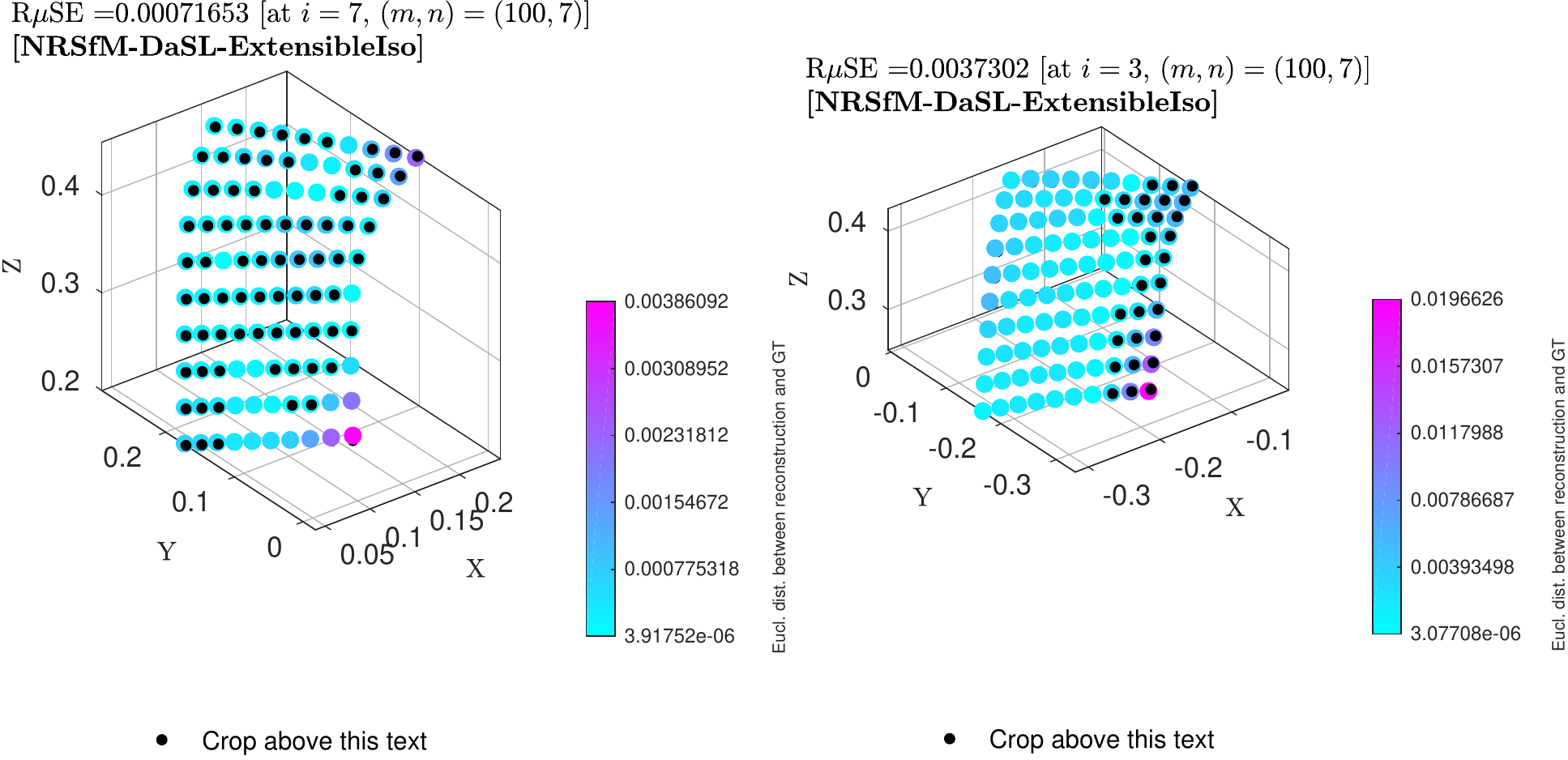}
    \end{overpic}
    \begin{overpic}[width=\qw, trim=10 40 286 25,clip]{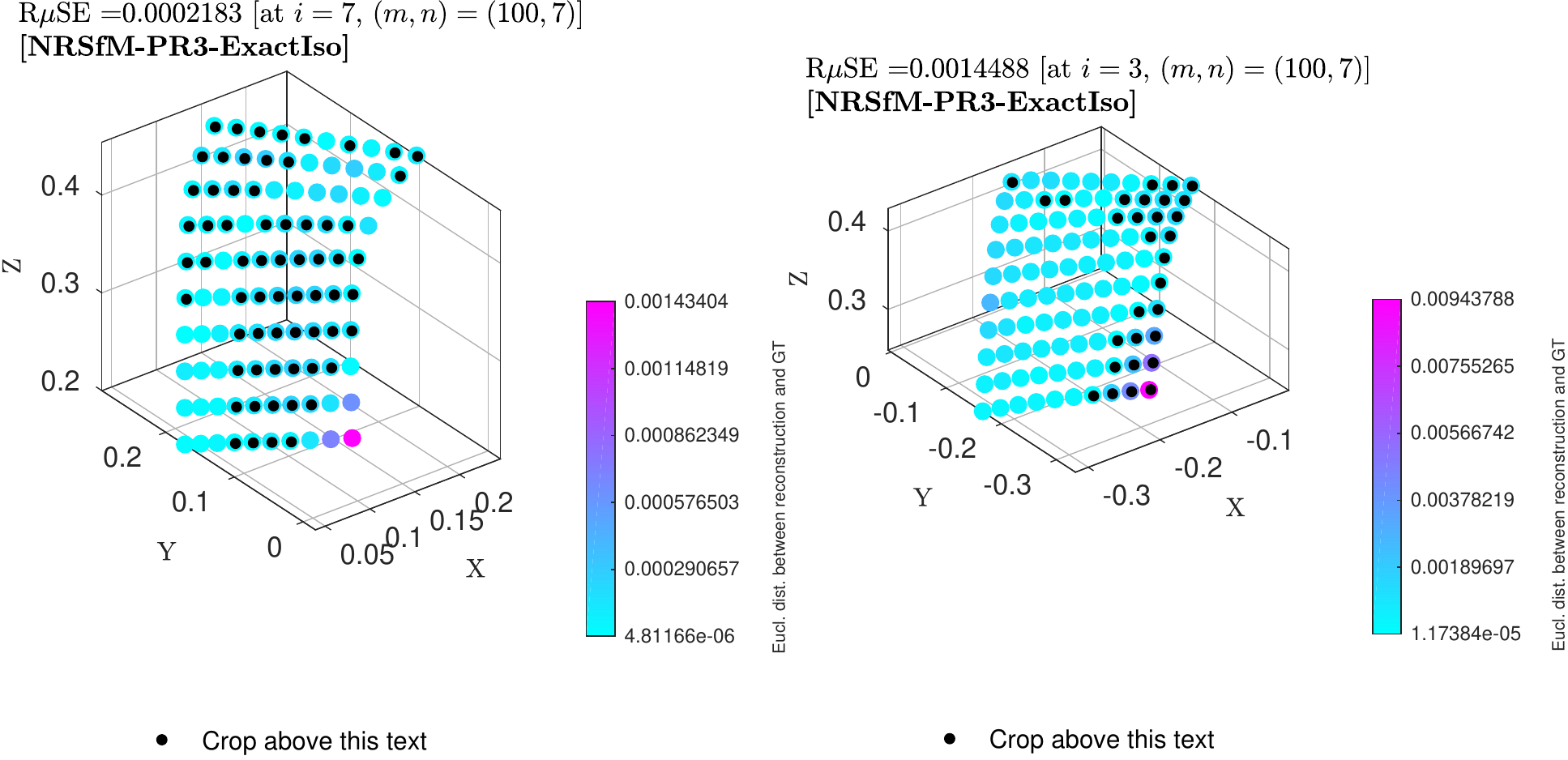}
    \end{overpic}
    \begin{overpic}[width=\qw, trim=10 40 286 25,clip]{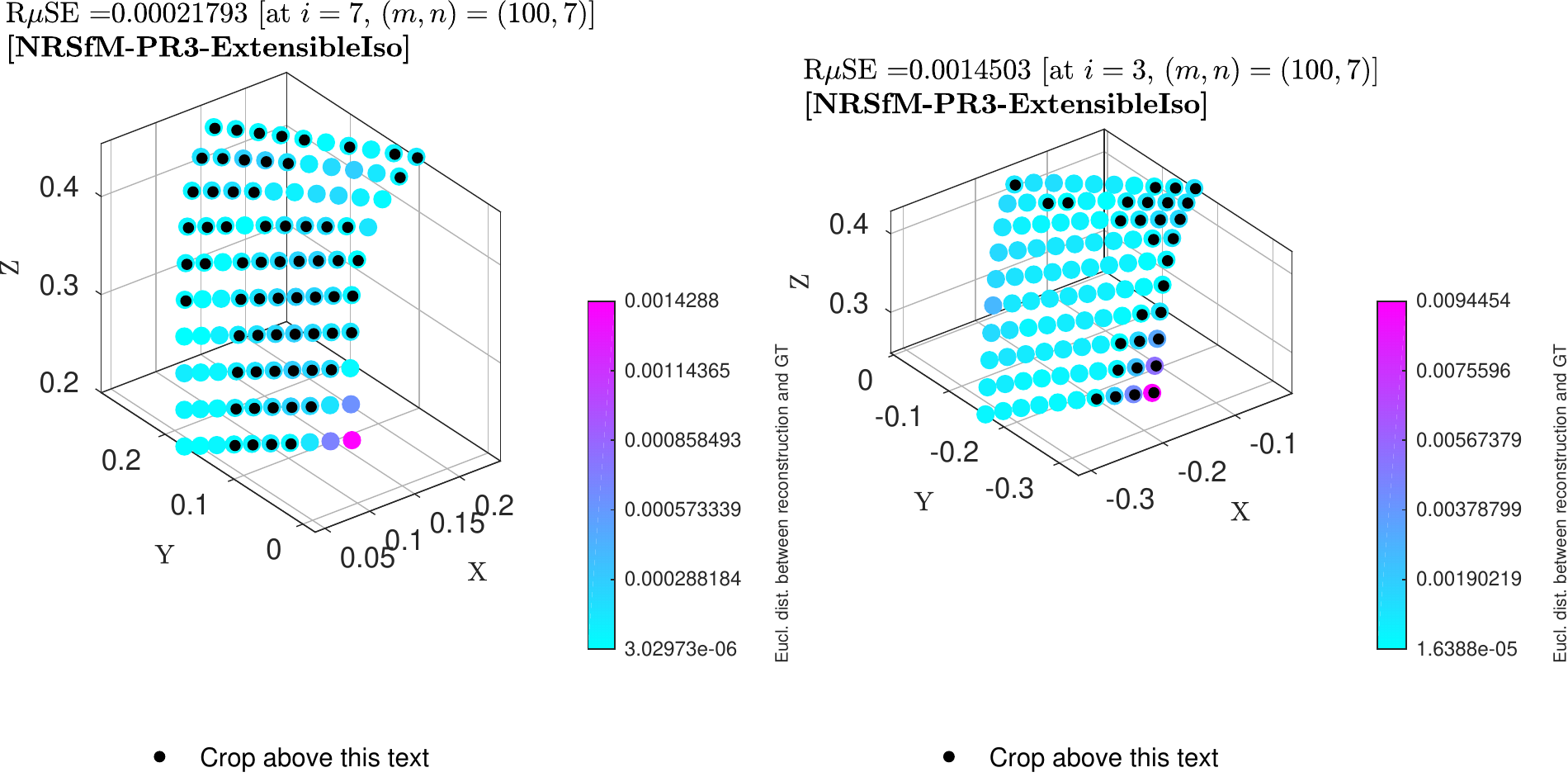}
    \end{overpic} 
    
    \begin{overpic}[width=\qw, trim=270 40 25 25,clip]{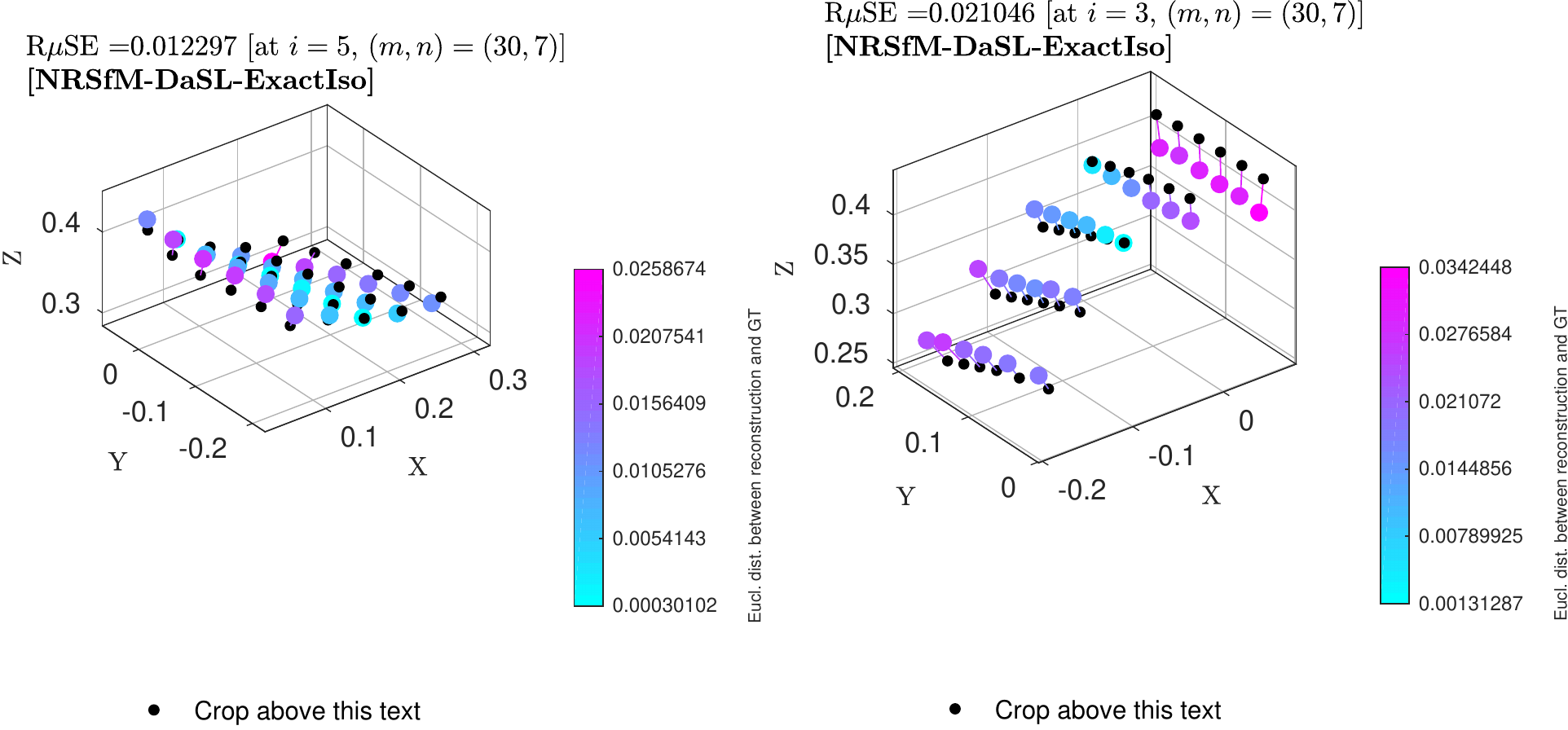}
    \end{overpic}
    \begin{overpic}[width=\qw, trim=273 40 25 25,clip]{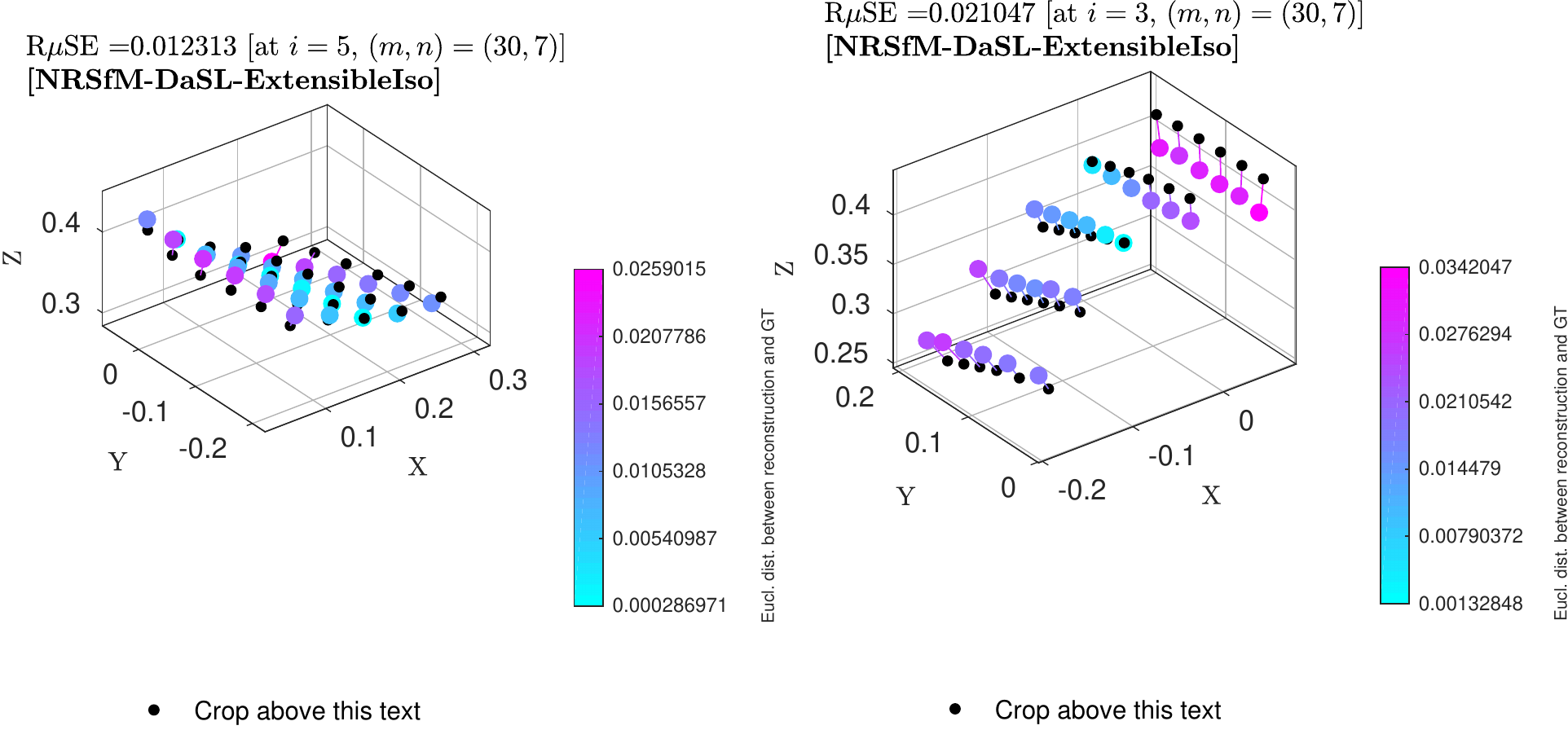}
    \end{overpic}
    \begin{overpic}[width=\qw, trim=280 40 25 25,clip]{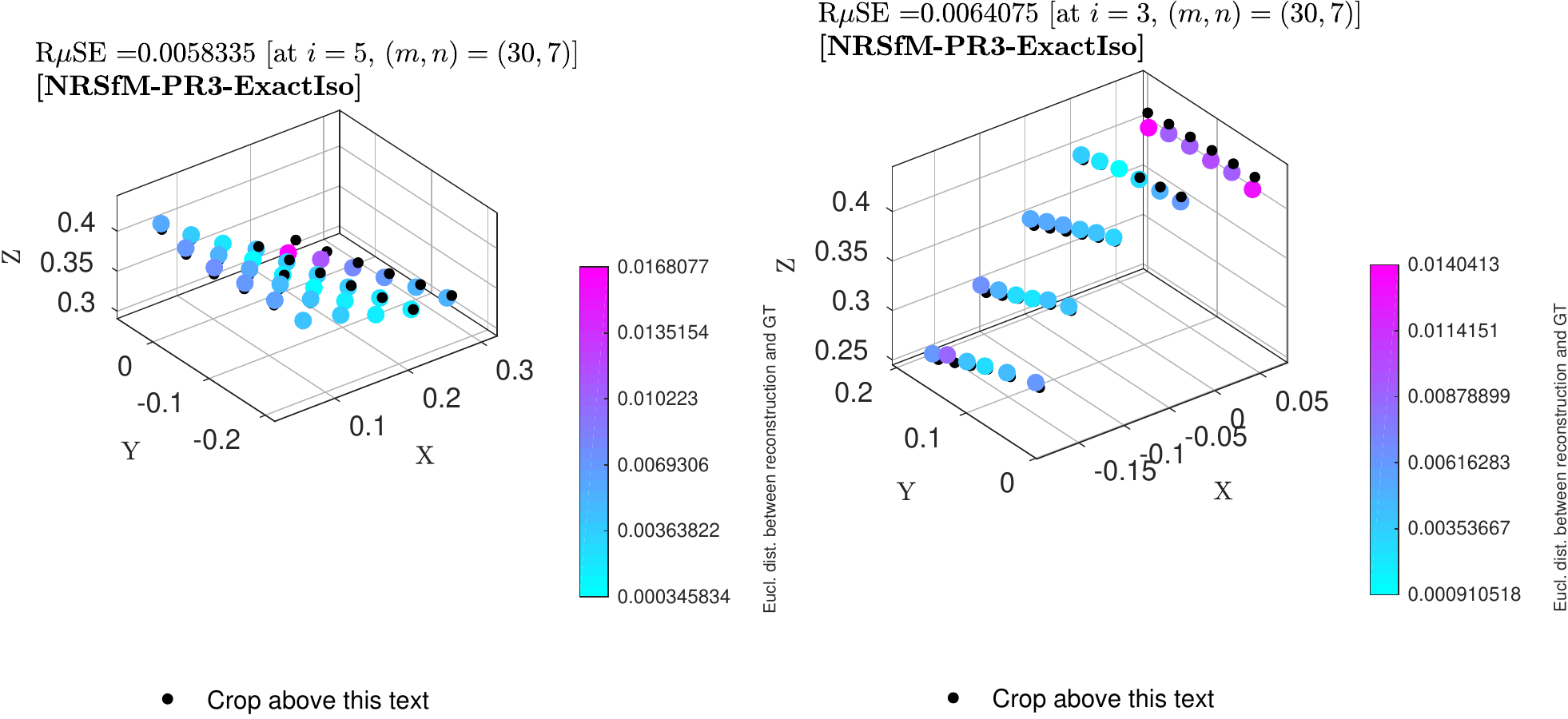}
    \end{overpic}
    \begin{overpic}[width=\qw, trim=280 40 25 25,clip]{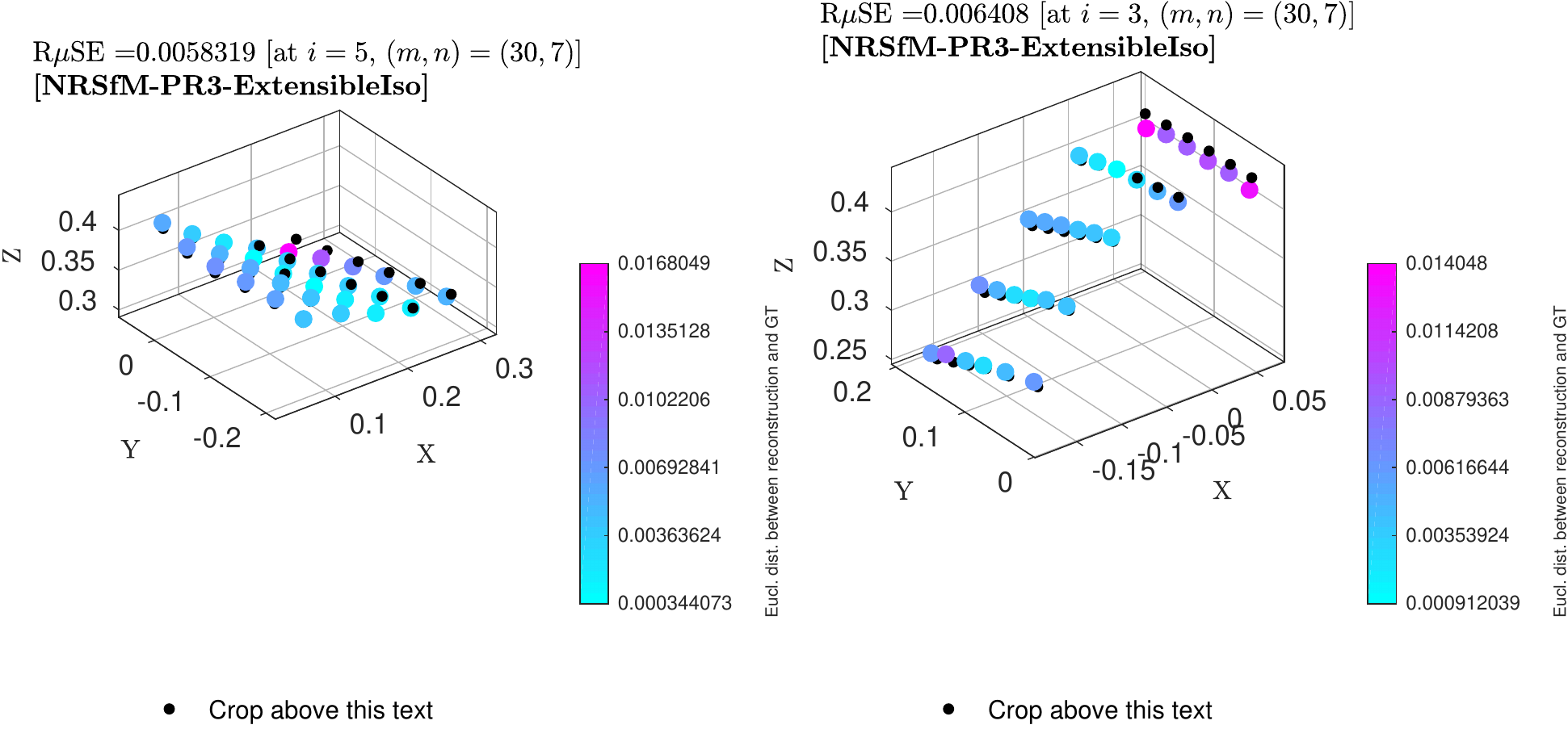}
    \end{overpic}

    \begin{overpic}[width=\qw, trim=10 40 286 25,clip]{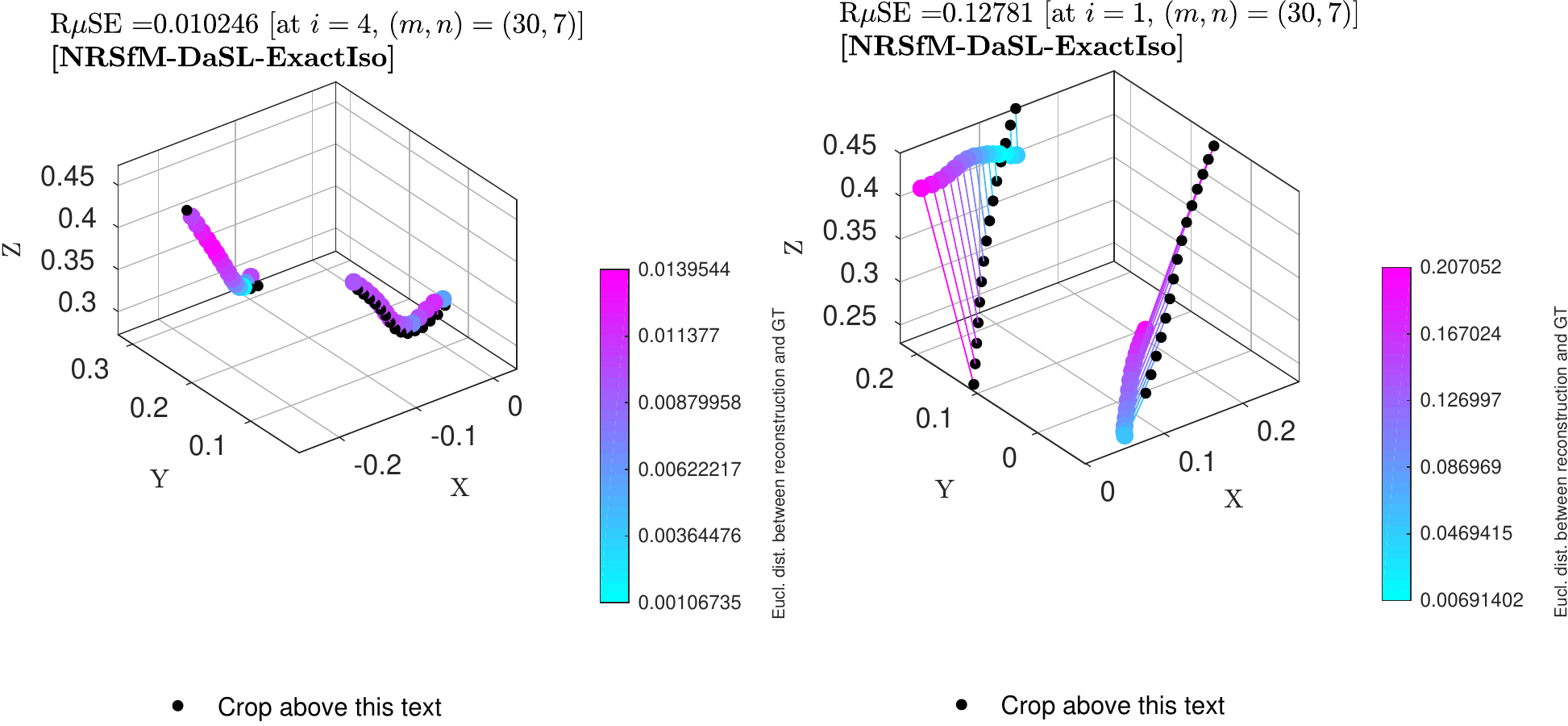}
    \end{overpic}
    \begin{overpic}[width=\qw, trim=10 40 286 25,clip]{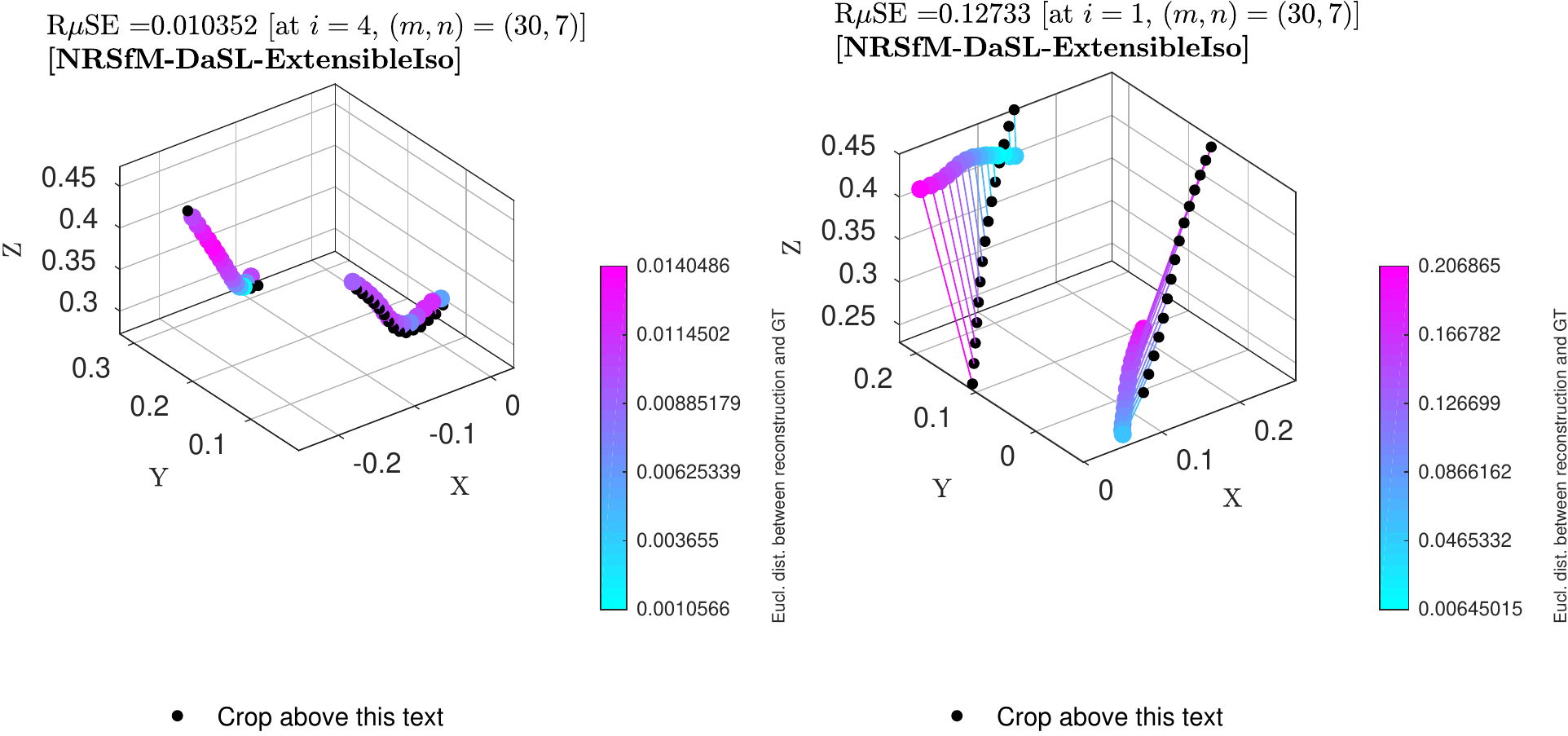}
    \end{overpic}
    \begin{overpic}[width=\qw, trim=10 40 286 30,clip]{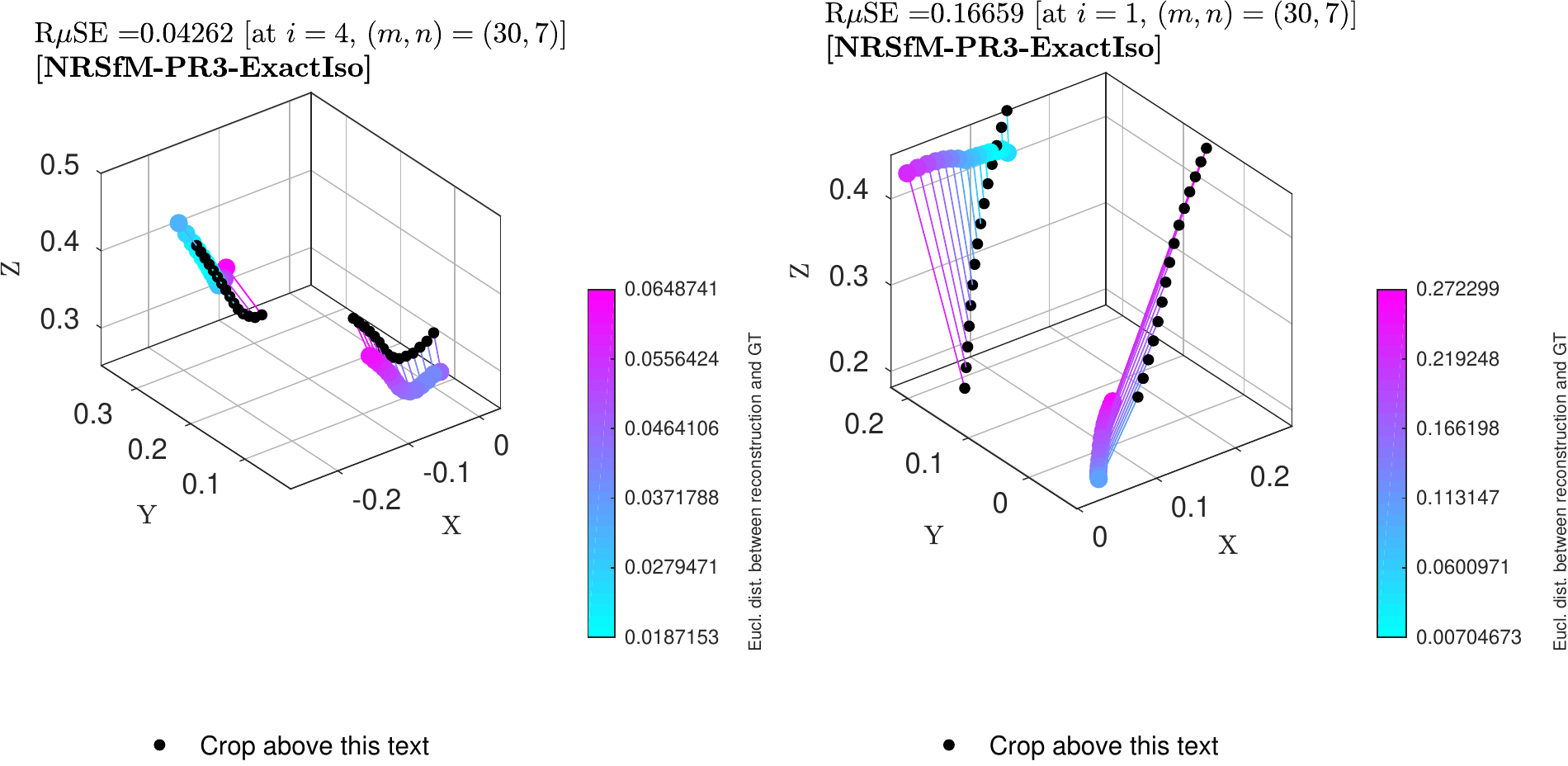}
    \end{overpic}
    \begin{overpic}[width=\qw, trim=10 40 286 30,clip]{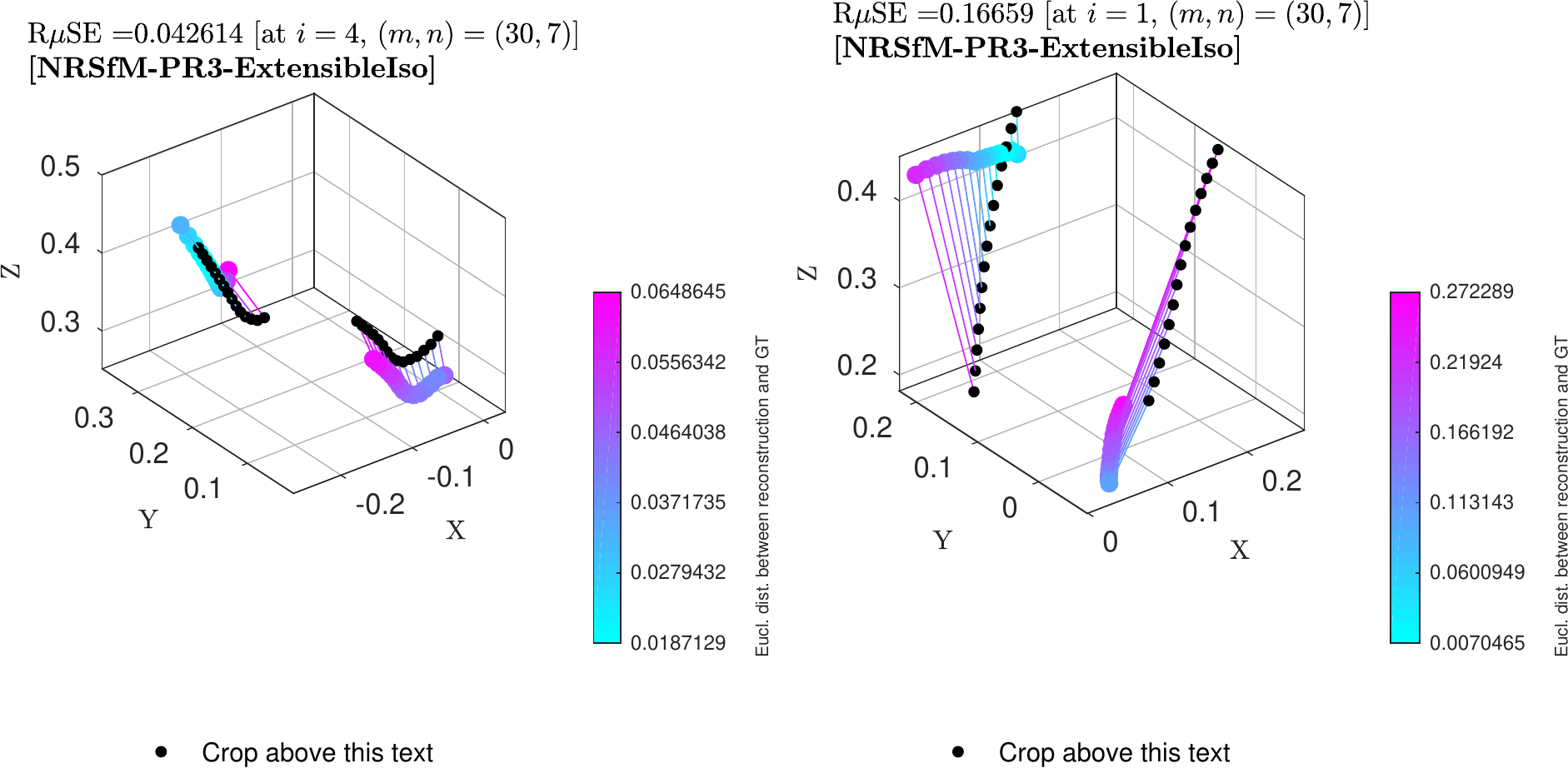}
    \end{overpic}     
    
    \begin{overpic}[width=\qw, trim=286 40 25 25,clip]{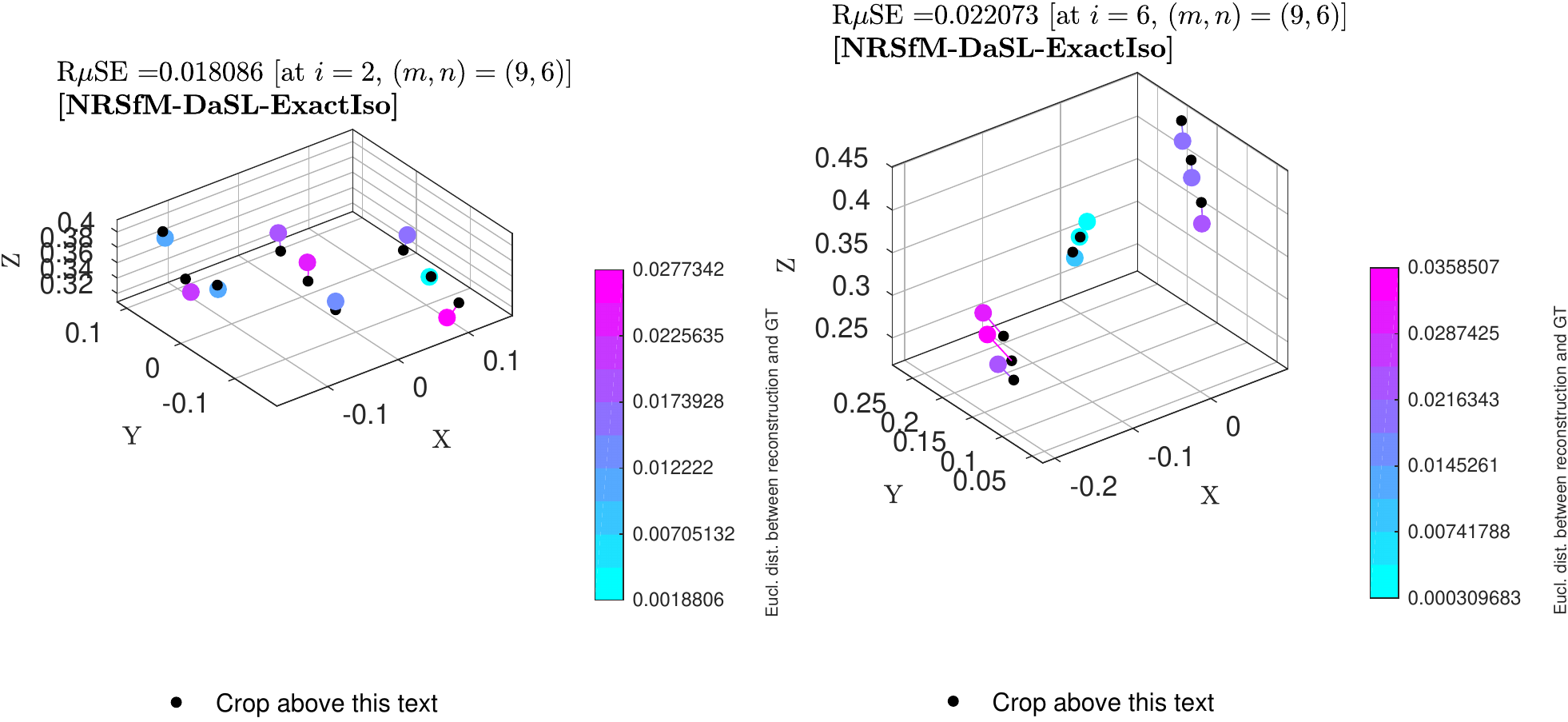}
    \end{overpic}
    \begin{overpic}[width=\qw, trim=286 40 25 25,clip]{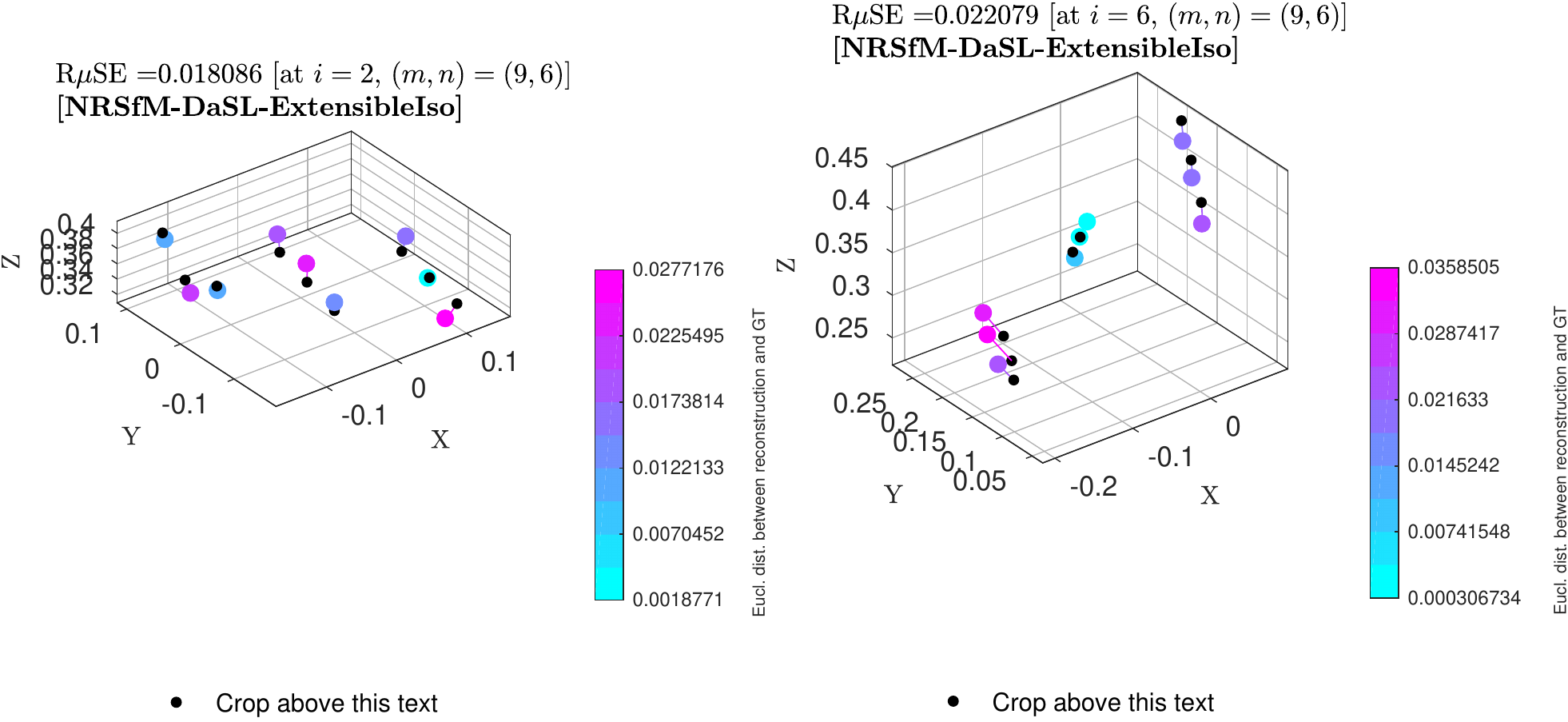}
    \end{overpic}
    \begin{overpic}[width=\qw, trim=286 40 25 25,clip]{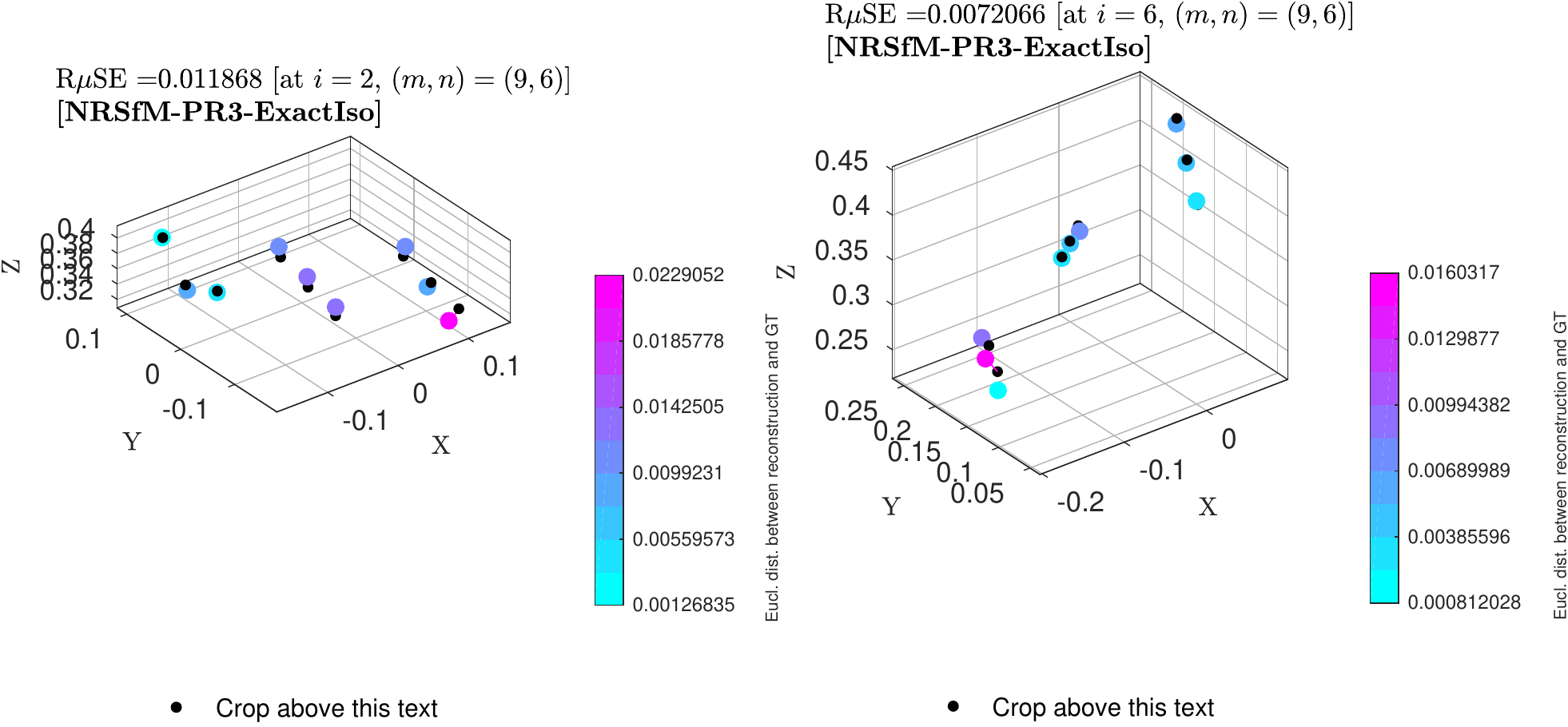}
    \end{overpic}
    \begin{overpic}[width=\qw, trim=286 40 25 25,clip]{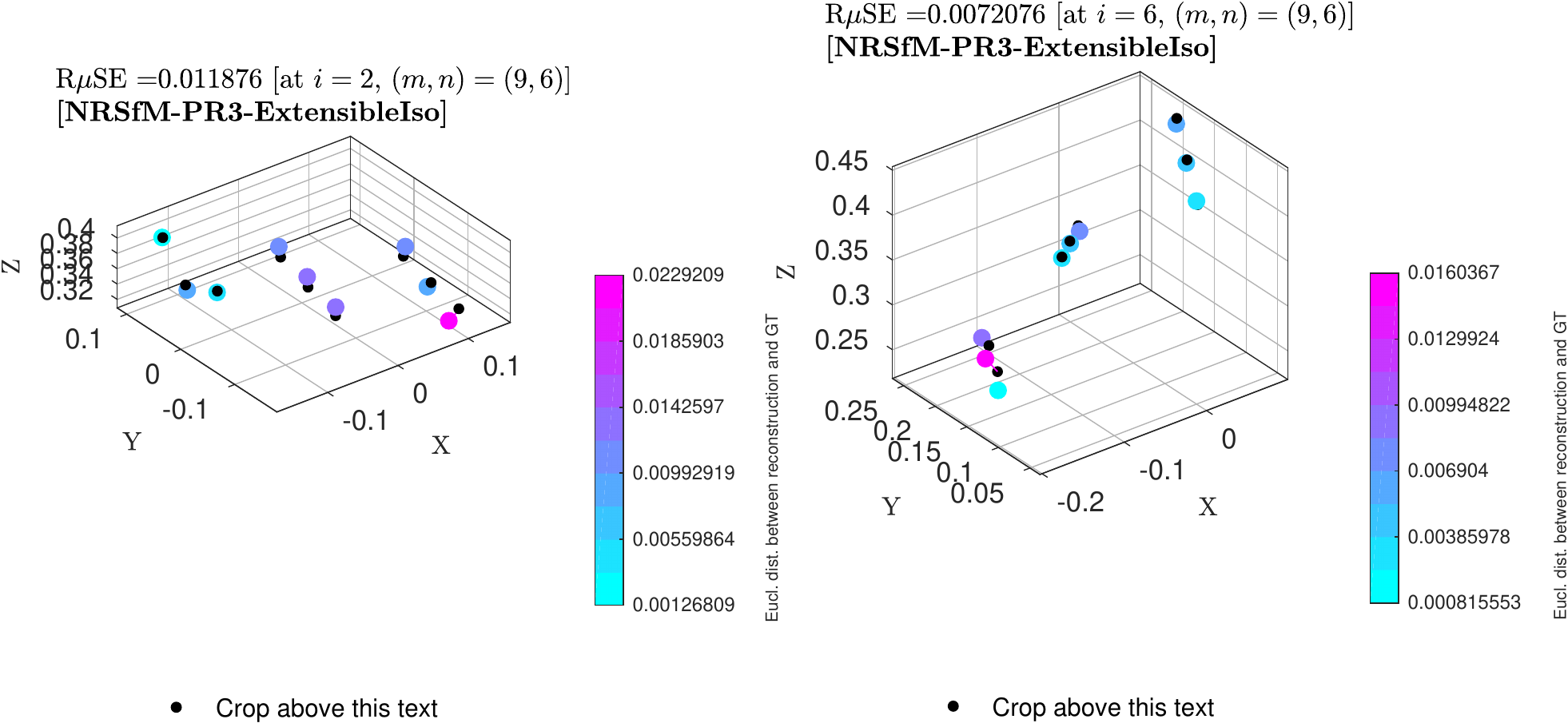}
    \end{overpic}       
    \caption{We show some demonstrative synthetic results from \y{isg} with four of our isometric methods. Each row represents a new dataset. The \textit{first} column is \y{s1}, the \textit{second} column is \y{q1}, the \textit{third} column is \y{s2} and the \textit{fourth} column is \y{q2}. Black points are \y{gt} and the reconstructed points are color-coded by their Euclidean distance to corresponding \y{gt} point (connected by a thin line). All units in \y{au}}
    \label{fig_qual_synth_large}
\end{figure*}

{\small
\bibliographystyle{ieee_fullname}
\bibliography{references}
}

%% file: acronyms.tex
\setabbreviationstyle{long-short-sc}
\setabbreviationstyle[acronym]{long-short}%
\newacronym{ttp}{TTP}{Tubular Topology Prior}
\newacronym{stp}{STP}{Spherical Topology Prior}
\newacronym{wtp}{WTP}{Weak Topology Prior}
\newacronym{sft}{S\textit{f}T}{Shape-from-Template}
\newacronym{nrsfm}{NRS\textit{f}M}{Non-Rigid Structure-from-Motion}
\newacronym{mdh}{MDH}{Maximum Depth Heuristic}
\newacronym{cnn}{CNN}{Convolutional Neural Network}
\newacronym{dcnn}{DCNN}{Deep Convolutional Neural Network}
\newacronym{crf}{CRF}{Conditional Random Fields}
\newacronym{nng_}{NNG}{Nearest-Neighborhood Graph}
\newacronym{tps}{TPS}{Thin-Plate Spline}
\newacronym{sfs}{S\textit{f}S}{Shape-from-Shading}
\newacronym{sd}{S.D.}{Standard-Deviation}
\newacronym{apss}{APSS}{Algebraic Point-Set Surfaces}
\newacronym{slam}{SLAM}{Simultaneous Localisation and Mapping}
\newacronym{icp}{ICP}{Iterative Closest Point}
\newacronym{wtas}{GWT}{Generic Weak Template}
\newacronym{wtts}{SWT}{Specific Weak Template}
\newacronym{m1}{NRS\textit{f}M++}{NRSfM-suffix-`++'}
\newacronym{m2}{NRS\textit{f}M+}{NRSfM-suffix-`+'}
\newacronym{svd}{SVD}{Singular Value Decomposition}
\newacronym{lls}{LLS}{Linear Least-Squares}
\newacronym{ps}{PS}{Parametric Surface}
\newacronym{nps}{NS}{Non-parametric Surface}
\newacronym{rbf}{RBF}{Radial Basis Function}
\newacronym{gpm}{GPM}{Generalised Problem of Moments}
\newacronym{nrba}{BA}{Bundle Adjustment} 
\newacronym{lm}{LM}{Levenberg-Marquardt}
\newacronym{bnb}{BnB}{Branch-and-Bound}
\newacronym{lft}{LFT}{Linear Fractional Transformations}
\newacronym{sp}{SP}{Stereographic Projection}
\newacronym{cf}{CF}{conformal flattening}
\newacronym{sfm}{S\textit{f}M}{Structure-from-Motion}
\newacronym{socp}{SOCP}{Second-Order Cone Programming}
\newacronym{mlh}{MLH}{Maximum-Leg Heuristics}
\newacronym{sdp}{SDP}{Semi-Definite Programming}
\newacronym{arap}{ARAP}{As-Rigid-As-Possible}
\newacronym{sfc}{S\textit{f}C}{Structure-from-Category}
\newacronym{rms}{$\mathrm{E}_r$}{Root Mean Square Error}
\newacronym{p2p}{$\mu$P$_2$P}{Mean Point-to-Plane Error}
\newacronym{med}{$\mathrm{E}_d$}{Mean Euclidean Distance}
\newacronym{gt}{G\textit{t}}{Groundtruth}
\newacronym{dcf}{DCF}{Distance Constraint Function}
\newacronym{isonr}{Iso-NRS\textit{f}M}{Isometric NRS\textit{f}M}
\newacronym{mdhnr}{MDH-NRS\textit{f}M}{Maximum Depth Heuristic NRS\textit{f}M}
\newacronym{ao}{AO}{Alternating Optimisation}
\newacronym{grp}{GRP}{Graph Realisation Problem}
\newacronym{agrp}{AF-G-RP}{Anchor-Free Graph Realization Problem}
\newacronym{bysr}{BY-$\mathbf{S}^n_+$-R}{Biswas-Ye Semidefinite Relaxation}
\newacronym{edm}{EDM}{Euclidean Distance Matrix}
\newacronym{gm}{GM}{Gram Matrix}
\newacronym{dgm}{DGM}{Depth Gram Matrix}
\newacronym{pgm}{PGM}{Position Gram Matrix}
\newacronym{mt}{MT}{Metric Tensegrity}
\newacronym{lrsb}{LRSB}{Low Rank Shape Basis}
\newacronym{sos}{S\textit{o}S}{Sum-of-Squares}

\newacronym{au}{$au$}{Arbitrary Units}

\newacronym{sm}{SM}{Supplementary Materials}

\newacronym{snr}{\textit{s}I-NRS\textit{f}M}{Strictly Isometric NRS\textit{f}M}
\newacronym{qnr}{\textit{q}I-NRS\textit{f}M}{Quasi Isometric NRS\textit{f}M}
\newacronym{enr}{E-NRS\textit{f}M}{Equiareal NRS\textit{f}M}
\newacronym{hnr}{\textit{q}E-NRS\textit{f}M}{Quasi Equiareal NRS\textit{f}M}
\newacronym{sl}{SL}{Sight-Line}
\newacronym{psd}{PSD}{Positive Semi-definite}
\newacronym{dsl}{D\textit{a}SL}{Depth-Along-Sight-Lines}
\newacronym{pp}{\textit{p}3D}{Position in $\mathbb{R}^3$}
\newacronym{po}{PO}{Polynomial Optimisation}
\newacronym{pn}{PN}{Proximal Neighbors}
\newacronym{lmi}{LMI}{Linear Matrix Inequality}
\newacronym{llr}{L\textit{h}LR}{Lasserre's hierarchy of LMI Relaxations}
\newacronym{isg}{ISG}{Isometric Surface Generator}
\newacronym{qsg}{\textit{q}ESG}{Quasi-Equiareal Surface Generator}
\newacronym{wc}{WC}{White Cartoon T-shirt}
\newacronym{kp}{KP}{Kinect Paper}
\newacronym{st}{ST}{Stretched T-shirt}
\newacronym{mh}{MH}{Model House}
\newacronym{sb}{SB}{Squeezed Balloon}
\newacronym{slg}{SL}{Stretched Leggings}

\newacronym{s1}{\textit{s}I-$\delta$}{Strictly Isometric - D\textit{a}SL}
\newacronym{s2}{\textit{s}I-P}{Strictly Isometric - \textit{p}3D}
\newacronym{q1}{\textit{q}I-$\delta$}{Quasi Isometric - D\textit{a}SL}
\newacronym{q2}{\textit{q}I-P}{Quasi Isometric - \textit{p}3D}
\newacronym{h1}{\textit{e}I-$\delta$}{Hybrid Equiareal - D\textit{a}SL}
\newacronym{h2}{\textit{e}I-P}{Hybrid Equiareal - \textit{p}3D}

\newacronym{wt}{WT}{Weak Template}
\newacronym{gwt}{GWT}{Generic Weak Template}
\newacronym{swt}{SWT}{Specific Weak Template}
\newacronym{qpbo}{QPBO}{Quadratic Pseudo-Boolean Optimisation}
\newacronym{qclp}{QCLP}{Quadratically Constrained Linear Program}
\newacronym{qcqp}{QCQP}{Quadratically Constrained Quadratic Program}

%% file: sections/introduction.tex
Many objects and structures commonly deform. \y{nrsfm} is the problem of reconstructing the geometry of such deforming objects from point correspondences across monocular images. \y{nrsfm} requires a deformation model, for which the length-preserving isometry is the most common~\cite{chen2020dense, probst2018incremental, chhatkuli2014non}. Consequently, the reconstruction of isometrically deforming surfaces has been well explored and many solutions already exist.  However, stretchable surfaces break isometry and defeat these methods. In contrast, the area-preserving equiareal model is obeyed by many real-world stretchable objects and hence, offers a valuable alternative to isometry. We propose models and methods for \y{nrsfm} that allow significant deviation from isometry, in the form of: 1) deviation from the isometric solution up to some tolerance, and 2) deviation from the isometric solution while following the equiareal model as closely as possible. 

\glsunset{gt}
\begin{figure}\centering
\begin{overpic}[width = \columnwidth,trim=30 0 0 0,clip]{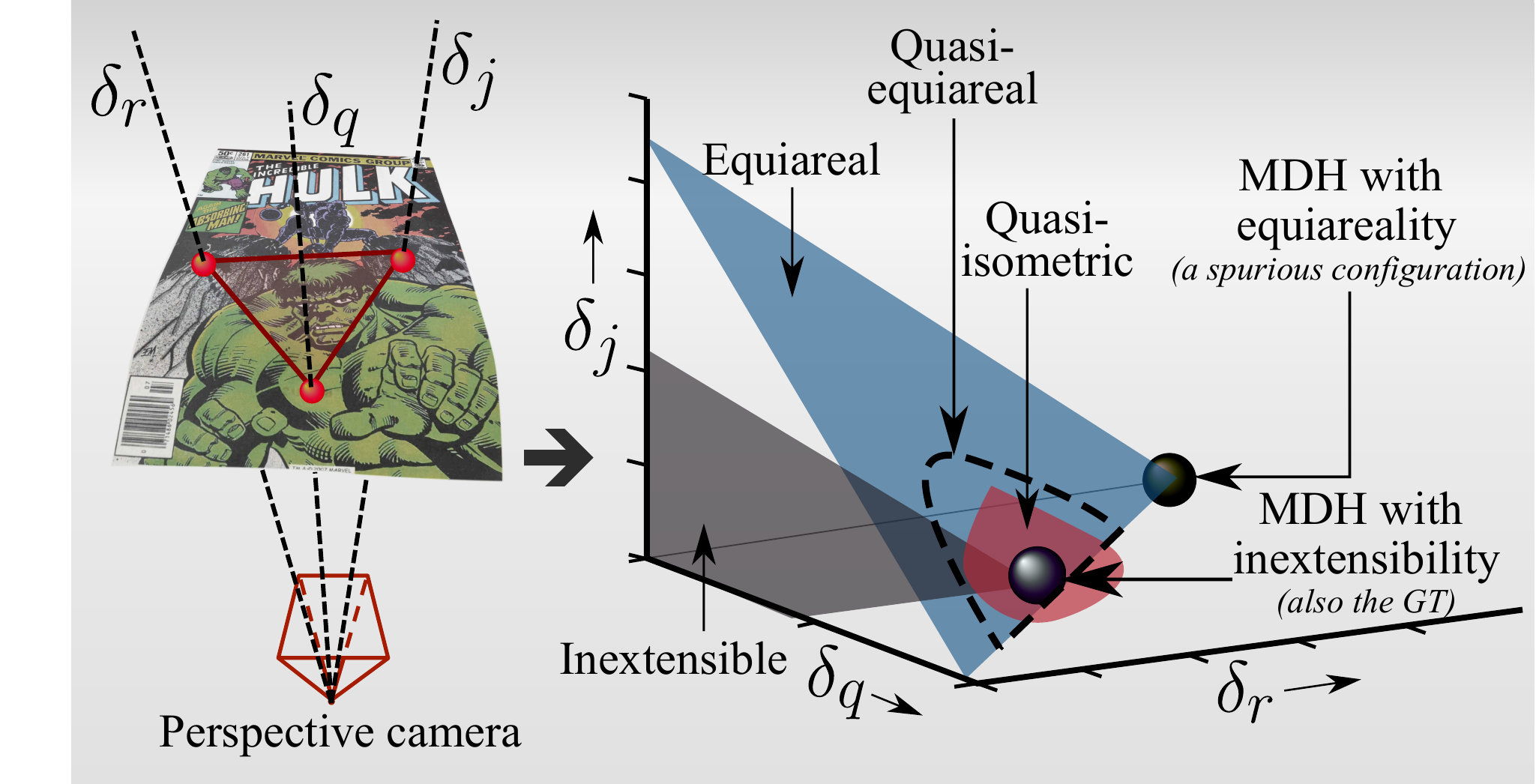}
\end{overpic}
\caption{Observation of the solution spaces for three simulated points in general configuration observed by a perspective camera and parameterised by their depths $\delta_j$, $\delta_q$, and $\delta_r$.
The \y{gt} point is in the centre. 
The solution spaces for inextensibility and equiareality are unbounded subspaces.
Using the MDH with inextensibility yields a solution point reasonably close to the \y{gt}. 
However, using the MDH with equiareality yields a spurious solution point.
The proposed quasi-isometric and quasi-equiareal formulations are smaller bounded subspaces enclosing the \y{gt}.}
\label{fig:theme_outlay}
\end{figure}
\glsresetall

Existing methods for \y{nrsfm} are broadly split into two categories, depending on the type of deformation model they use: {\em i)} \y{lrsb}~\cite{gotardo2011kernel,kumar2020non,kumar2022organic, dai2014simple}, and {\em ii)} physics-based or differential models. The latter type was solved from point correspondences~\cite{chhatkuli2017inextensible}, the so-called {\em zeroth-order methods}, and their derivatives~\cite{parashar2020local,parashar2017isometric}.
Isometric \y{nrsfm} received solutions based on convex relaxations with the zeroth-order methods~\cite{chhatkuli2017inextensible, ji2017maximizing}.
These solutions uses the \y{mdh} with inextensibility constraints. 
The angle-preserving conformal deformation model was also investigated~\cite{parashar2019local} but shows very limited applicability to real-world objects. 
In contrast, equiareality has never been explored.  
A possible reason is the inherent intractability of the problem \cite{heidelberg_05}. 
Interestingly however, the area of simplices bounded by sightlines can be expressed as polynomials which can be reparameterised and relaxed to convex formulations. 
Our proposed approaches search the near vicinity of the solution space of \y{mdh} - isometric \y{nrsfm}, either freely up to a tolerance, or using equiareality as a secondary objective. Figure~(\ref{fig:theme_outlay}) illustrates the intuition behind this reasoning.

Technically, our \textit{first} contribution is to reformulate isometric \y{nrsfm} with a more restrictive relaxation than any other previous approaches, yielding a strictly isometric method. Our \textit{second} contribution is to relax this strict isometry up to a defined tolerance which, unlike the well-known inextensibility, is not restricted to a convex cone. Our \textit{third} contribution is to parameterise the solutions using 3D point positions in $\mathbb{R}^3$, instead of the prevalent noise-sensitive \y{dsl}. Our \textit{fourth} contribution is the introduction of a deformation model combining isometry and equiareality.
In addition, we establish the close relationship of \y{nrsfm} to the \y{grp} and show how the two fields can cross-fertilise.
All our methods are solved with convex optimisation. They reconstruct benchmark datasets with state-of-the-art accuracy and stretchable surfaces with an accuracy unmatched by any other \y{nrsfm} method.

%% file: sections/background.tex
We describe the \y{nrsfm} literature and background.

{\flushleft \textbf{\y{nrsfm}.}} \y{nrsfm} has been essentially dealt with using isometry~\cite{perriollat2011monocular,chhatkuli2014non, chen2020dense, parashar2017isometric}. Following early work on \y{sft} \cite{perriollat2011monocular}, which showed that isometry leads to surface flipping ambiguities, heuristics were proposed to force a unique solution, \y{mdh} being the most prominent~\cite{salzmann2010linear}.
\y{nrsfm} has also been addressed using the conformal model~\cite{parashar2019local}. In constrast, equiareality is highly ambiguous as a deformation constraint~\cite{heidelberg_05, do2016differential}. 
An earlier attempt to solve equiareal \y{nrsfm} concluded on its impossibility~\cite{parashar2019local}.
Indeed, an equiareally deforming surface has more independent parameters than the number of constraints it generates. Later, a hybrid equiareal \y{nrsfm} was attempted~\cite{parashar2020local}, but with lower accuracy than isometric and conformity. Interestingly, \cite{casillas2019equiareal} attempted to solve the closely related and better constrained \y{sft} problem with equiareality; it concluded that even equiareal \y{sft} is unsolvable without additional priors in the form of boundary conditions.
To date, no convex formulation to equiareal \y{nrsfm} has been known.

{\flushleft \textbf{Convex solutions to \y{nrsfm}.}} To solve the \y{nrsfm} problem, \cite{chhatkuli2017inextensible} proposed convex inextensibility relaxations. The squared Euclidean distance between points in $\mathbb{R}^3$ can be trivially obtained as a quadratic function. The inextensible model upper bounds the Euclidean distance up to an unknown geodesic distance. Such inequality constraints in the Lorentz cone of quadratic terms form a special subclass of conic optimisation, \y{socp}. \cite{ji2017maximizing} realised that a reformulation of inextensibility constraints allow the Euclidean distances between point-pairs to be restricted to semidefinite cones, hence posed the problem as an \y{sdp} and a trace minimisation of these \y{sdp} matrices (as a convex surrogate for rank) of depth-of-points.
This also has the desirable side-effect of maximising rigidity. 

Existing approaches do not deal with the unrelaxed, original isometry constraint, nor with the equiareal constraint.
They all follow inextensibility and rely on an \y{nng_} to impart a graph structure to the input point correspondences. 
Isometric \y{nrsfm} is therefore the problem of finding the most isometric configuration of the \y{nng_} in $\mathbb{R}^3$, subject to additional constraints, including reprojection and surface smoothness.
This formulation compels comparison with \y{grp} \cite{belk2007realizability}, which has many applications in fields including wireless sensor network localisation, molecular conformation, design and analysis of tensegrity structures. 

{\flushleft \textbf{Graph realisation problem.}} 
We consider a set of $n$ vertices $\{\mathbf{v}_i \in \mathbb{R}^d\}$, $i \in [1,n]$, with a prescribed \y{edm} denoted by $\mathbf{D} \in \mathbb{R}^{n \times n}$ such that $\mathbf{D}_{i,i'} = \|\mathbf{v}_i - \mathbf{v}_{i'}\|$. 
For an adjacency matrix $\mathcal{E}$ imparting a graph structure on $\{\mathbf{v}_i\}$, the problem of recovering the unknown vertex positions $\{\mathbf{v}_i\}$ from a sparse but known \y{edm} forms the \y{agrp}. 
The wireless sensor network localisation literature~\cite{so2007theory} offers a well-known \y{sdp} solution to \y{agrp}. A \y{gm} obtained from the vertex positions as $\mathbf{Y} = \mathfrak{v}\mathfrak{v}^{\top}$, where $\mathfrak{v}^{\top} = (\mathbf{v}_1^{\top},\hdots,\mathbf{v}_n^{\top})$ parameterises the \y{sdp}. Given a sparse observation of some elements of $\mathbf{Y}$, attempting to reduce the rank of $\mathbf{Y}$ while maintaining it as \y{psd} performs \textit{matrix completion}~\cite{candes2010power}. Matrix completion by \y{sdp} is a well-studied problem. There exists a linear expression in $\mathbf Y$ for the Euclidean distance between vertices. Combining the matrix completion via \y{sdp} with this linear expression yields:
\begin{subequations}\label{eqn:biswas_ye_solution}
 \begin{gather}
     \min_{\mathbf{Y}} \Tr(\mathbf{Y}) \label{eqn:bys_1}\\
     \text{s.t.} \qquad \mathbf{Y}_{i,i} + \mathbf{Y}_{i',i'} - 2\mathbf{Y}_{i,i'} = \mathbf{D}_{i,i'}, \label{eqn:bys_2}\\
     \mathbf{Y} \in \mathbf{S}^n_+, \qquad \forall i, i' \in [1,n]. \label{eqn:bys_3}
 \end{gather}
 \end{subequations}
A solution to equation~(\ref{eqn:biswas_ye_solution}) forms the \textit{Biswas-Ye Semidefinite Relaxation} solutions to \y{agrp} problem \cite{biswas2008distributed}. There also exists variants of equation~(\ref{eqn:biswas_ye_solution}) where only the upper and/or lower bounds of the Euclidean distances in $\mathbf{D}$ can be observed, e.g., in \textit{molecular conformation}. 

{\flushleft \textbf{Beyond Euclidean distances.}} The linearisation of Euclidean distances in equation~(\ref{eqn:bys_2}) happens due to `lifting' the solution to a higher-dimensional space (higher-rank \y{sdp}). However, from the theory of \textit{polynomial optimisation}, there exists well-known methods to solve any global unconstrained minimisation of a \y{sos} polynomial by a finite sequence of \y{lmi} relaxations \cite{lasserre2006convergent}. Therefore, the approximation of Euclidean distances, as in equation~(\ref{eqn:bys_2}), is just one possible variant of polynomial valued function linearised using a higher dimensional \y{psd} matrix. Such an approach had indeed already been used for solving first-order, isometric \y{nrsfm} \cite{probst2019convex}. Interestingly, the experiments of \cite{probst2019convex} point to the discretionary nature of the higher degree \y{lmi} relaxations for the \y{nrsfm} problem. Another possibility is to formulate constraints based on angle between vertices, as in the \textit{angle based sensor network localisation} problem \cite{jing2021angle}. 
Importantly, equiareal constraints have never been handled using \y{sdp} in the context of computer vision or otherwise.

%% file: sections/method.tex
We describe our \y{nrsfm} methods in details.
\subsection{Notation and Preliminaries}
\y{nrsfm} takes as inputs $m$ point correspondences across $n$ images taken from a perspective camera with known intrinsics. The point correspondences are denoted by their normalised homogeneous coordinates $\{\mathbf{x}_i = (\mathbf{p}_{i,1},\hdots,\mathbf{p}_{i,m})^{\top} \in \mathbb{R}^{m \times 3}\}$, $i \in [1,n]$, and corresponding unknown 3D points by $\{\mathbf{X}_i = (\mathbf{P}_{i,1},\hdots,\mathbf{P}_{i,m})^{\top} \in \mathbb{R}^{m \times 3}\}$, $i \in [1,n]$. $\mathbf{P}_{i,j} = (X_{i,j}, Y_{i,j}, Z_{i,j})^{\top}$ is linked to $\mathbf{p}_{i,j}$ by the perspective projection function $\Pi(\mathbf{P}_{i,j}) = \mathbf{p}_{i,j} = \mathbf{P}_{i,j}/Z_{i,j}$. The \y{sl} is the line directed along $\mathfrak{d}_{i,j} \in S^2$ passing through the camera centre such that $\mathfrak{d}_{i,j} = \mathbf{p}_{i,j}/\|\mathbf{p}_{i,j}\| = (x_{i,j},y_{i,j},z_{i,j})^{\top}$. Each $\mathbf{X}_i$ lies on the surface embedding $\mathcal{S}_i \subset \mathbb{R}^3$ which is linked to an unknown parameterisation $\mathcal{T} \subset \mathbb{R}^2$. For isometric deformations, \textit{geodesic distances} along the surface manifold $\mathcal{S}_i$ must be preserved, while for equiareal deformations, area on the surface manifold $\mathcal{S}_i$ must be preserved.



We denote as $\mathcal{E}$, the set of edges of a fully connected graph with points in $\mathbf{x}_1$ as vertices. From $\mathcal{E}$, we derive the following: 1) a set of 1-simplices $\mathcal{E}_2 \subset \mathcal{E}$, and 2) a set of 2-simplices $\mathcal{E}_3 = \{(j, q, r)\}$, where $(j,q),(q,r),(r,j) \in \mathcal{E}_2$, such that  $(\mathbf{P}_{i,j},\mathbf{P}_{i,q},\mathbf{P}_{i,r})$ forms a 2-simplex. 
We define $|\mathcal{E}_{2}| = p_1$ and $|\mathcal{E}_{3}| = p_2$. We define two more sets $\mathcal{G}_2$ and $\mathcal{G}_3$, such that: 1) $\mathcal{G}_2$ is the set of geodesic distances for all 1-simplices in $\mathcal{E}_2$, and 2) $\mathcal{G}_3$ is the set of areas for all 2-simplices in $\mathcal{E}_3$. For \y{nrsfm}, $\mathcal{G}_2$ and $\mathcal{G}_3$ are unknowns. We define two functions: 1) $\mathfrak{g}_I(i,j,q)$ computes the distance between $\mathbf{P}_{i,j}$ and $\mathbf{P}_{i,q}$ and 2) $\mathfrak{g}_{E}(i,j,q,r)$ computes the area of the triangle with $\mathbf{P}_{i,j}$, $\mathbf{P}_{i,q}$ and $\mathbf{P}_{i,r}$ as vertices in $\mathbb{R}^3$. We denote by $\mathbf{S}_+^n$ the group of $n \times n$ dimensional \y{psd} matrices. We use the notation $\sum_{i_1,\hdots,i_z = 1}^{x_1, \hdots, x_z}$ to denote the $z$-times summation $\sum_{i_1 = 1}^{x_1} \sum_{i_2 = 1}^{x_2}\hdots\sum_{i_z = 1}^{x_z}$.

\subsection{Problem Statement}
We assume $\mathcal{E}_{2}$ and $\mathcal{E}_{3}$ are given and we solve the \y{nrsfm} problem using isometric and equiareal constraints. We do so by solving three different variants, two involving just isometry and one involving equiareality. Isometric \y{nrsfm} does not possess a unique solution without additional priors, \y{mdh} being the most widely used. This is implemented as a depth-maximisation term added to the cost, denoted by $f_{\mathrm{mdh}}$. For equiareal \y{nrsfm}, we observe the following:
\begin{lemma}\label{lem:non_unique_nrsfm}
There does not exist a unique solution, even up to scale, for zeroth-order \y{nrsfm} using just equiareal and reprojection constraints.
\end{lemma}
\begin{proof}
Section~1 of supplementary materials\footnote{Supplementary materials appended at the end of this document}.
\end{proof}
Unfortunately, \y{mdh} is not a valid prior for equiareal \y{nrsfm}. Therefore, equiareal \y{nrsfm} requires other additional constraints and priors. The problem we actually solve is a hybrid-equiareal \y{nrsfm} with a quasi-isometric cost weighted down to allow significant deviation from isometry, maintaining perfect equiareality between corresponding 2-simplices.

To describe our problem statement, we introduce two cost functions: {\em i)} a cost for \y{mdh} $f_{\mathrm{mdh}}(\mathbf{X}_i)$, the maximisation of which increases the depth of $\mathbf{X}_i$ along the positive $Z$ axis in $\mathbb{R}^3$, and {\em ii)} a reprojection cost $f_{\mathrm{prj}}(\mathbf{X}_{i},\mathbf{x}_{i})$ which penalises deviation of $\Pi(\mathbf{P}_{i,j})$ from $\mathbf{p}_{i,j}$ for $i \in [1,n],j \in [1,m]$. Our problem statements are:

1) \y{snr}:
\begin{subequations} \label{eqn:iso_nrsfm_prblm_formul}
\begin{gather}
    \{\mathbf{X}_i\}, \mathcal{G}_{2} = \argmin_{\{\mathbf{X}_i\}, \mathcal{G}_{2}} \sum_{i=1}^n \Big(f_{\mathrm{prj}}(\mathbf{X}_{i},\mathbf{x}_{i}) - f_{\mathrm{mdh}}(\mathbf{X}_i)\Big) \label{eqn:inpf_1}\\
    \text{s.t.} \quad \mathfrak{g}_{I}(i,j,q) = \mathcal{G}_{2}(j,q), \label{eqn:inpf_1} \\
    \forall \qquad i \in [1,n],\quad j,q \in [1,m], \quad (j,q) \in\mathcal{E}_{2}.
\end{gather}
\end{subequations}

2) \y{qnr}:
\begin{subequations} \label{eqn:quasi_nrsfm_prblm_formul}
\begin{gather} \{\mathbf{X}_i\}, \mathcal{G}_{2} = \argmin_{\{\mathbf{X}_i\}, \mathcal{G}_{2}} \sum_{i,j,q = 1}^{n,m,\mathcal{E}_{2}(j)} |\mathfrak{g}_{I}(i,j,q) - \mathcal{G}_{2}(j,q)| \notag \\ + \sum_{i=1}^n \Big(f_{\mathrm{prj}}(\mathbf{X}_{i},\mathbf{x}_{i}) - f_{\mathrm{mdh}}(\mathbf{X}_i)\Big) \label{eqn:qnpf_2}\\ 
     \forall \quad i \in [1,n],\quad j,q \in [1,m], \quad (j,q) \in\mathcal{E}_{2}. \label{eqn:qnpf_3}
\end{gather}
\end{subequations}

3) \y{hnr}:
\begin{subequations} \label{eqn:hybrid_nrsfm_prblm_formul}
\begin{gather}
    \{\mathbf{X}_i\}, \mathcal{G}_{2}, \mathcal{G}_{3} = \argmin_{\{\mathbf{X}_i\}, \mathcal{G}_{2}, \mathcal{G}_{3}}  \Big(f_{\mathrm{prj}}(\mathbf{X}_{i},\mathbf{x}_{i}) - f_{\mathrm{mdh}}(\mathbf{X}_i)\Big) \notag\\
    + \lambda_I\sum_{i,j,q=1}^{n,m,\mathcal{E}_{2}(j)} |\mathfrak{g}_{I}(i,j,q) - \mathcal{G}_{2}(j,q)| \notag \\+ \lambda_E \sum_{i,j,q,r=1}^{n,m,\mathcal{E}_{2}(j),\mathcal{E}_{3}(j,q)} |\mathfrak{g}_{E}(i,j,q,r) - \mathcal{G}_{3}(j,q,r)| \label{eqn:enpf_1}\\
    \forall \qquad i \in [1,n],\quad j,q,r \in [1,m], \quad (j,q,r) \in\mathcal{E}_{3}.
\end{gather}
\end{subequations}
Beyond these formulations, our technical contributions are reparamerisations and convex relaxation to \y{snr} and \y{qnr} in section~\ref{sec:iso_sol} and to \y{hnr} in section~\ref{sec:equi_sol}.

\subsection{Isometric Cost: Relaxation, Convex Reparameterisation and Isometric Methods}\label{sec:iso_sol}
The next two sections give our isometric methods.
The first set of methods uses the \y{dsl} parameterisation, where an unknown point $\mathbf{P}_{i,j}$ is parameterised by its depth along its \y{sl} $\mathfrak{d}_{i,j}$, as in existing work.
The second set of methods uses the \y{pp} parameterisation, where $\mathbf{P}_{i,j}$ may not lie on $\mathfrak{d}_{i,j}$ exactly. Importantly, using \y{pp} allows us to handle correspondence noise must better than \y{dsl}.

\subsubsection{Parameterisation with \y{dsl}}
The \y{dsl} parameterisation defines $\mathbf{P}_{i,j} = \delta_{i,j} \mathfrak{d}_{i,j}$ for depth $\delta_{i,j} \in \mathbb{R}_+$. Stacked together for the $i$-th image, the depth vector is $\mathfrak{D}_i^{\top} = \begin{pmatrix}\delta_{i,1}, \hdots, \delta_{i,m}\end{pmatrix}$. We define the \y{dgm} as $\mathfrak{R}_i = \mathfrak{D}_i \mathfrak{D}_i^{\top}$, which is related to the Euclidean distances by:
\begin{equation}\label{eqn:geod_first_expr}
    \mathfrak{g}_I^{\delta}(i,j,q) = \mathfrak{R}_{i,j,j} + \mathfrak{R}_{i,q,q} - 2\mathfrak{R}_{i,j,q}\langle\mathfrak{d}_{i,j},\mathfrak{d}_{i,q}\rangle,
\end{equation}
where $\mathfrak{g}_I^{\delta}(i,j,q)$ is the implementation of $\mathfrak{g}_I$ for \y{dsl}. Therefore, we express equation~(\ref{eqn:iso_nrsfm_prblm_formul}) parameterised by \y{dgm} and $\mathfrak{g}_I^{\delta}(i,j,q)$ as equation~(\ref{eqn:geod_first_expr}). To handle the global scale ambiguity, we require the geodesics to sum to 1. Moreover, \y{dgm} has to be \y{psd} and of rank 1. 
However, strict rank restrictions being non-convex, we use the common convex relaxation provided by trace minimisation of $\mathfrak{R}_i$. Eventually, we implement $f_{\mathrm{mdh}}(\mathbf{X}_i)$ by the minimisation of the inverse diagonal elements of $\mathfrak{R}_i$, whilst $f_{\mathrm{prj}}$ is already enforced by \y{dsl}. Therefore, our final solution to \y{snr} parameterised by \y{dsl} is:
\begin{subequations}
    \begin{gather}
        \min_{\{\mathfrak{R}_i\},\mathcal{G}_{2}}\sum_{i=1}^n \Tr(\mathfrak{R}_i)  + \sum_{i,j=1}^{n,m} \frac{1}{\mathfrak{R}_{i,j,j}}  \label{eqn:sds2_1} \\ \text{s.t.} \notag \\
        \mathfrak{g}_I^{\delta}(i,j,q) = \mathcal{G}_{2}(j,q), 
        \sum_{j,q=1}^{m,\mathcal{E}_{2}(j)} \mathcal{G}_{2}(j,q) = 1,  \notag\\ \mathfrak{R}_i \in \mathbf{S}_+^m, \quad 
        \forall \quad i \in [i,n], j \in [1,m],(j,q) \in \mathcal{E}_{2},\label{eqn:sds2_2}
    \end{gather}
\end{subequations}
which is a convex formulation solved by standard \y{sdp}. 

We modify equation~(\ref{eqn:sds2_1})-(\ref{eqn:sds2_2}) to solve \y{qnr} by recasting it in Lagrangian form, modifying the strict equality in equation~(\ref{eqn:sds2_2}) to an $\ell^1$-norm minimisation, giving the \y{dsl} solution to \y{qnr}:
\begin{subequations}\label{eqn:qnr_dep_sol}
    \begin{gather}
        \min_{\{\mathfrak{R}_i\},\mathcal{G}_{2}} \quad \sum_{i=1}^n \Tr(\mathfrak{R}_i) + \sum_{i,j=1}^{n,m} \frac{1}{\mathfrak{R}_{i,j,j}}\quad + \notag \\
        \lambda_I \sum_{i,j,q=1}^{n,m,\mathcal{E}_{2}(j)} \Big( |\mathfrak{g}_I^{\delta}(i,j,q) - \mathcal{G}_{2}(j,q)|\Big),\label{eqn:qds_1}\\
        \text{s.t.} \qquad \sum_{j,q=1}^{m,\mathcal{E}_{2}(j)} \mathcal{G}_{2}(j,q) = 1, \quad \mathfrak{R}_i \in \mathbf{S}_+^m,\label{eqn:qds_2}\\
        \forall \quad i \in [i,n], \quad j \in [1,m], \quad (j,q) \in \mathcal{E}_{2}.
    \end{gather}
\end{subequations}

\subsubsection{Parameterisation with \y{pp}}
The \y{pp} parameterisation directly uses $\mathbf{P}_{i,j} \in \mathbb{R}^3$. 
Stacked together for the $i$-th image, the  position vector is $\mathfrak{P}_i^{\top} = \begin{pmatrix}\mathbf{P}_{i,1}^{\top},\hdots,\mathbf{P}_{i,m}^{\top}\end{pmatrix} \in \mathbb{R}^{3m}$.
We define the \y{pgm} as $\mathfrak{S}_i = \mathfrak{P}_i \mathfrak{P}_i^{\top}$, which is related to the Euclidean distances by:
\begin{equation}
    \begin{gathered}
    \mathfrak{g}_I^{P}(i,j,q) = \|\mathbf{P}_{i,j} - \mathbf{P}_{i,q}\|_2^2 = \\ \mathfrak{S}_{i,j_x,j_x} + \mathfrak{S}_{i,j_y,j_y} +  \mathfrak{S}_{i,j_z,j_z} +  \mathfrak{S}_{i,q_x,q_x} +  \mathfrak{S}_{i,q_y,q_y}\\ + \mathfrak{S}_{i,q_z,q_z} -  2 ( \mathfrak{S}_{i,j_x,q_x} + \mathfrak{S}_{i,j_y,q_y} + \mathfrak{S}_{i,j_z,q_z}), 
    \end{gathered}
\end{equation}
where $\begin{pmatrix}j_x, j_y, j_z\end{pmatrix} = 3(j-1) + \begin{pmatrix}1, 2, 3\end{pmatrix}$ and $(q_x, q_y, q_z)$ are also obtained similarly. 
We formulate the reprojection cost $f_{\mathrm{prj}}$ for \y{snr} and \y{qnr} as $f^{\Pi}_{i,j}$, given by:
\begin{equation}
    \begin{gathered}
    f^{\Pi}_{i,j}(\mathfrak{S}_{i}) = \|\mathbf{P}_{i,j} \times \mathfrak{d}_{i,j}\|_2^2 = \mathfrak{S}_{i,j_x,j_x} (y_{i,j}^2 \\
     +  z_{i,j}^2) + \mathfrak{S}_{i,j_y,j_y} (x_{i,j}^2 +  z_{i,j}^2) +
    \mathfrak{S}_{i,j_z,j_z} (x_{i,j}^2 +  y_{i,j}^2) \\- 2( \mathfrak{S}_{i,j_x,j_y} x_{i,j}y_{i,j} +  \mathfrak{S}_{i,j_x,j_z} x_{i,j}z_{i,j} + \mathfrak{S}_{i,j_y,j_z} y_{i,j}z_{i,j} ).
    \end{gathered}
\end{equation}
Adapting \y{mdh} to \y{pp} is non-trivial, as it requires the $Z$-component of $\mathbf{P}_{i,j}$ to be in $\mathbb{R}_+$ whilst the $X$ and $Y$ components should be left free in $\mathbb{R}$. This is not realisable on a \y{pgm} which is being used as the unknown parameter. To resolve this, we augment the \y{pgm} with an additional row and column as:
\begin{equation}\label{eqn:augmented_gram}
    \mathfrak{S}_{i}^{'} = \begin{pmatrix} 1 & \mathfrak{P}_i^{\top} \\ \mathfrak{P}_i & \mathfrak{S}_{i} \end{pmatrix},
\end{equation}
which also happens to be the moment matrix of first order \y{lmi} relaxation \cite{lasserre2000convergent}. The first column or row of $\mathfrak{S}_{i}^{'}$ allows us to compute $\langle\mathbf{P}_{i,j},\mathfrak{d}_{i,j}\rangle$, which is a signed scalar value. Hence, maximising $\langle\mathbf{P}_{i,j},\mathfrak{d}_{i,j}\rangle$ ensures non-negative $Z_{i,j}$. The \y{mdh} cost function $f_{\mathrm{mdh}}$ for \y{pp} is then given as:
\begin{equation}
    f_{i,j}^{\delta}(\mathfrak{S}_{i}^{'}) = \begin{pmatrix} \mathfrak{S}_{i,j_x',1}^{'} , \mathfrak{S}_{i,j_y',1}^{'} , \mathfrak{S}_{i,j_z',1}^{'} \end{pmatrix}^{\top}\mathfrak{d}_{i,j}.
\end{equation}
where $\begin{pmatrix}j_x', j_y', j_z'\end{pmatrix} = \begin{pmatrix}j_x, j_y, j_z\end{pmatrix} + 1$ due to equation~(\ref{eqn:augmented_gram}). With a slight abuse of notation, we consider these change of indices to be transmitted appropriately to $f^{\Pi}_{i,j}(\mathfrak{S}_{i}^{'})$ and $\mathfrak{g}_I^{P}(i,j,q)$.

In order to solve \y{snr} with \y{pp}, we {\em i)} reparameterise equation~(\ref{eqn:iso_nrsfm_prblm_formul}) with $\mathfrak{S}_{i}^{'}$, {\em ii)} introduce \y{mdh} by maximising $\langle\mathbf{P_{i,j}},\mathfrak{d}_{i,j}\rangle$, {\em iii)} anchor the scale of the scene by requiring $\mathcal{G}_{2}$ to sum to 1, and {\em iv)} relax exact rank constraint on $\mathfrak{S}_{i}^{'}$ to minimise $\Tr(\mathfrak{S}_{i}^{'})$, i.e., an exact \y{pp} counterpart of the \y{dsl} solution in equation~(\ref{eqn:sds2_1})-(\ref{eqn:sds2_2}) with an additional reprojection error. Therefore, the \y{pp} solution to \y{snr} is:
\begin{subequations}\label{eqn:snr_xyz_sol}
    \begin{gather}
        \min_{\{\mathfrak{S}_{i}^{'}\}, \mathcal{G}_{2}}  \sum_{i=1}^n \Tr(\mathfrak{S}_{i}^{'}) + \sum_{i,j=1}^{n,m} f^{\Pi}_{i,j}(\mathfrak{S}_{i}^{'}) - \sum_{i,j=1}^{n,m} f_{i,j}^{\delta}(\mathfrak{S}_{i}^{'})  \label{eqn:sxs_1}\\
        \text{s.t.} \quad \mathcal{G}_{2}(j,q) = \mathfrak{g}_I^{P}(i,j,q), \notag \\ \sum_{j,q=1}^{m,\mathcal{E}_{2}(j)} \mathcal{G}_{2}(j,q) = 1,   \mathfrak{S}_{i,1,1}^{'}  = 1, \mathfrak{S}_{i}^{'} \in \mathbf{S}_+^{m+1} \label{eqn:sxs_2}\\
        \forall \quad i \in [i,n], \quad j \in [1,m], \quad (j,q) \in \mathcal{E}_{2}. \label{eqn:sxs_3}
    \end{gather}
\end{subequations}
The \y{pp} solution for \y{qnr} is:
\begin{subequations}\label{eqn:qnr_xyz_sol}
    \begin{gather}
        \min_{\{\mathfrak{S}_{i}^{'}\}, \mathcal{G}_{2}} \bigg(  \sum_{i=1}^n\Tr(\mathfrak{S}_{i}^{'}) + \sum_{i,j,q=1}^{n,m,\mathcal{E}_{2}(j)} \Big(f^{\Pi}_{i,j}(\mathfrak{S}_{i}^{'}) + \notag \\ \lambda_I|\mathfrak{g}_I^{P}(i,j,q)  - \mathcal{G}_{2}(j,q)| -  f_{i,j}^{\delta}(\mathfrak{S}_{i}^{'}) \Big) \bigg)  \label{eqn:qxs_1}\\
        \text{s.t.} \qquad \sum_{j,q=1}^{m,\mathcal{E}_{2}(j)} \mathcal{G}_{2}(j,q) = 1, \quad \mathfrak{S}_{i,1,1}^{'} = 1,  \mathfrak{S}_{i}^{'} \in \mathbf{S}_+^{m+1} \label{eqn:qxs_2}\\
        \forall \quad i \in [i,n],\quad j \in [1,m],\quad (j,q) \in \mathcal{E}_{2}. \label{eqn:qxs_3}
    \end{gather}
\end{subequations}
Again, equations~(\ref{eqn:snr_xyz_sol}) and (\ref{eqn:qnr_xyz_sol}) represent convex problems solvable by standard \y{sdp}.

{\flushleft \textbf{Correspondence completion.}} 
In addition to dealing with noise, \y{pp} allows one to implement correspondence completion.
In the vast majority of cases, the input correspondences $\{\mathbf{x}_i\}$ have missing points. We write the visibility indicator as $\mathcal{V}$ such that $\mathcal{V}_{i,j}=1$  if $\mathbf{p}_{i,j}$ is visible and $0$ otherwise. 
Existing isometric methods handle incomplete data, but do not hallucinate the missing points. 
To all such missing points in $\mathcal{N}_i^{'}$, we assign $s$ \y{pn} which are its $s$ closest neighbours in $\mathbf{x}_1$. We guide the depth maximisation of these invisible points with the support of \y{sl} from its \y{pn}, leaving the rest of the minimisation as is for both \y{snr} and \y{qnr}. The \y{sdp} solution computes the position in $\mathbb{R}^3$ for these missing vertices, as if `hallucinating' correspondences that are not part of the inputs. We do this by substituting the scalar product $f_{i,j}^{\delta}(\mathfrak{S}_{i}^{'})$ in equations~(\ref{eqn:sxs_1}) and (\ref{eqn:qxs_1}) by a function $g^{\delta}(\mathfrak{S}_{i}^{'})$, which when evaluated for the $(i,j)$-th point is written as:
\begin{equation}\label{eq_corr_compl}
    g^{\delta}_{i,j}(\mathfrak{S}_{i}^{'}) = \begin{cases*}
                     \begin{pmatrix} \mathfrak{S}_{i,j_x',1}^{'}, \mathfrak{S}_{i,j_y',1}^{'}, \mathfrak{S}_{i,j_z',1}^{'} \end{pmatrix}^{\top}\mathfrak{d}_{i,j}  \\
                     \text{ } \qquad \qquad \qquad \qquad \qquad \qquad  \quad  \text{if~} \mathcal{V}_{i,j} = 1, \\
                     \sum_{l=1}^{\mathrm{PN}(j)}\begin{pmatrix} \mathfrak{S}_{i,j_x',1}^{'}, \mathfrak{S}_{i,j_y',1}^{'}, \mathfrak{S}_{i,j_z',1}^{'} \end{pmatrix}^{\top}\mathfrak{d}_{i,l} \\ \text{ } \qquad \qquad \qquad \text{ }\qquad \qquad \qquad \quad \text{otherwise}. 
                 \end{cases*}
\end{equation}

\subsection{Equiareal Cost: Convex Relaxations and Quasi-equiareality}\label{sec:equi_sol}
We describe our method to solve \y{hnr} by first computing $\mathfrak{g}_E^{\delta}$ and $\mathfrak{g}_E^{P}$ followed by convex relaxations to linearise these two functions.

{\flushleft \textbf{Area of 2-simplices.}} We observe that the square area enclosed by vertices $\mathbf{P}_{i,j}, \mathbf{P}_{i,q}, \mathbf{P}_{i,r} \in \mathbf{X}_i$ can be expressed as a quartic with both \y{dsl} and \y{pp} parameterisations. For $(\mathbf{P}_{i,j}, \mathbf{P}_{i,q}, \mathbf{P}_{i,r})$, the square area using \y{dsl} parameterisation is expressed by a function $\mathfrak{h}_E^{\delta}(i,j,q,r)$ which is a trivariate quartic in $\delta_{i,j}$, $\delta_{i,q}$ and $\delta_{i,r}$. Similarly expanding the square area for \y{pp} results in a function $\mathfrak{h}_E^{P}(i,j,q,r)$ which, when expanded element-wise, also results in a quartic in the 9 coordinates of $\mathbf{P}_{i,j}$, $\mathbf{P}_{i,q}$ and $\mathbf{P}_{i,r}$. The detailed expressions are given in section~(2) of supplementary materials.



{\flushleft \textbf{Sparsity-aware convex relaxations for $\mathfrak{h}_E^{\delta}$ and $\mathfrak{h}_E^{P}$.}} There exists a well-known technique for convex relaxations of multivariate polynomials, namely, \y{llr} \cite{lasserre2000convergent}. However, standard \y{llr} would require full moment matrices for second order relaxations in order to arrive at a linear expression for multivariate quartics, which are highly computationally expensive. However, we show the following: 
\begin{prop}
For the $i$-th image, $\mathfrak{h}_E^{\delta}$ and $\mathfrak{h}_E^{P}$ admits lower-degree relaxations using sub-matrices of second-order moment matrices of \y{llr}, given by $\mathfrak{T}_i \in \mathbf{S}_{+}^{\tilde{p}_{2} + m}$ and $\mathfrak{U}_i \in \mathbf{S}_{+}^{6\tilde{p}_{2} + 3m + 1}$ respectively, where $\tilde{p}_{2}$ are the number of unique edges of the 2-simplices in $\mathcal{E}_{3}$
\end{prop}
\begin{proof}
We define two vectors $\mathfrak{E}_i$ and $\mathfrak{Q}_i$ as:
\begin{equation}
    \begin{gathered}
        \mathfrak{E}_i^{\top} = \Big(\delta_{i,1},\hdots,\delta_{i,m},\hdots,\delta_{i,j}\delta_{i,q},\delta_{i,q}\delta_{i,r},\delta_{i,j}\delta_{i,r},\hdots\Big), \\ 
        \mathfrak{Q}_i^{\top} = \Big(1, \mathbf{P}_{i,1}^{\top},\hdots,\mathbf{P}_{i,m}^{\top}, \hdots, \vartheta(i,j,q)^{\top}, \vartheta(i,q,r)^{\top}, \\ \vartheta(i,j,r)^{\top},\hdots\Big), \quad \forall (j,q,r) \in \mathcal{E}_3,
    \end{gathered}
\end{equation}
such that:
\begin{equation}
\begin{gathered}
\vartheta(i,j,q)^{\top} = \Big(X_{i,j}Y_{i,q}, X_{i,j}Z_{i,q}, Y_{i,j}X_{i,q}, Y_{i,j}Z_{i,q}, \\ Z_{i,j}X_{i,q}, Z_{i,j}Y_{i,q}\Big).    
\end{gathered}
\end{equation}
This allows us to construct two augmented Gram matrices, $\mathfrak{T}_i = \mathfrak{E}_i \mathfrak{E}_i^{\top}$ and $\mathfrak{U}_i = \mathfrak{Q}_i \mathfrak{Q}_i^{\top}$. We show that $\mathfrak{T}_i$ and $\mathfrak{U}_i$ are sufficient to express the quartics in $\mathfrak{h}_E^{\delta}$ and $\mathfrak{h}_E^{P}$ respectively as a linear combination of their elements. Let $\mathfrak{g}_E^{\delta}(i,j,q,r)$ and $\mathfrak{g}_E^{P}(i,j,q,r)$ be functions that maps $(\mathbf{P}_{i,j},\mathbf{P}_{i,q},\mathbf{P}_{i,r})$ to their area using linear combinations of elements from $\mathfrak{T}_i$ and $\mathfrak{U}_i$ respectively. We define two maps: \textit{a)} $\Omega^k: \mathbb{Z}^3 \mapsto \mathbb{Z}^3$ such that $\Omega^k(j,q,r) = (a_1, a_2, a_3)$ and $\mathfrak{T}_{i,a_1,a_1} = \delta_{i,j}^2\delta_{i,q}^2$, $\mathfrak{T}_{i,a_2,a_2} = \delta_{i,q}^2\delta_{i,r}^2$, and $\mathfrak{T}_{i,a_3,a_3} = \delta_{i,j}^2\delta_{i,r}^2$, and \textit{b)} $\rho^k:\mathbb{Z}^3 \mapsto \mathbb{Z}^{18}$ such that $\rho^k(j,q,r) = (b_1, \hdots, b_{18})$ and the following set of identities apply: $\{ \mathfrak{U}_{i,b_1,b_1} = X_{i,j}^2 Y_{i,q}^2, \hdots, \mathfrak{U}_{i,b_{18},b_{18}} = Z_{i,r}^2 Y_{i,q}^2\}$. The details are given in section~(3) of supplementary materials for $\mathfrak{g}_E^{\delta}$ and $\mathfrak{g}_E^{P}$. \end{proof}

{\flushleft \textbf{Localised constraints on $\mathfrak{T}_i$ and $\mathfrak{U}_i$.}} Some of the diagonal elements of $\mathfrak{T}_i$ and $\mathfrak{U}_i$ are equal to some other squared off-diagonal elements of themselves and this mapping is obtainable from $\Omega^k$ and $\rho^k$. However, the equality between an element of a matrix with another squared element of itself is non-affine, hence non-convex. So we relax it to an inequality constraint in the convex cone of the squared term, i.e., for $(j,q) \leftrightarrow  a_1$, we change the constraint $\mathfrak{T}_{i,j,q}^2 = \mathfrak{T}_{i,a_1,a_1}$ to $\mathfrak{T}_{i,j,q}^2 \leq \mathfrak{T}_{i,a_1,a_1}$. This is combined with minimisation of the traces of $\mathfrak{T}_i$ and $\mathfrak{U}_i$ while maximising depth. With an appropriate $(\lambda_I,\lambda_E)$ (from the \y{hnr} problem statement), this relaxation finds an accurate solution that slightly violates isometry of 1-simplices while maintaining equiareality of 2-simplices. 

\subsubsection{Parameterisation with \y{dsl}}
We follow four steps to solve \y{hnr} with \y{dsl}: {\em i)} we do all the convex relaxations, reparameterisation and inclusion of \y{mdh} as in equation~(\ref{eqn:qnr_dep_sol}), {\em ii)} we relax the exact rank-1 constraint on $\mathfrak{T}_i$ to trace minimisation, {\em iii)} we minimise the $\ell^1$-norm of the difference between the areas of all 2-simplices in $\mathcal{E}_{3}$ and the estimated  areas in $\mathcal{G}_{3}(j,q,r)$, and {\em iv)} we impose inequality constraints between elements of $\mathfrak{T}_i$.
With these modifications, \y{hnr} is  expressed as:
\begin{subequations}\label{eqn:hnr_dep_sol}
    \begin{gather}
        \min_{\{\mathfrak{T}_i\},\mathcal{G}_{2},\mathcal{G}_{3}} \sum_{i=1}^n \Tr(\mathfrak{T}_i) +
         \sum_{i,j,q=1}^{n,m,\mathcal{E}_{2}(j)} \lambda_I \Big( |\mathfrak{g}_I^{\delta}(i, j, q) - \mathcal{G}_{2}(j,q)|\Big)  \notag\\ + \sum_{i,j=1}^{n,m} \Big( \frac{1}{\mathfrak{T}_{i,j,j}} + \sum_{q,r=1}^{\mathcal{E}_{2}(j),\mathcal{E}_{3}(j,q)} \lambda_E|\mathfrak{g}_E^{\delta}(i,j,q,r) - \mathcal{G}_{3}(j,q,r)|\Big),\label{eqn:hds_1}\\
        \text{s.t.~~} 
         \mathfrak{T}_{i,j,q}^2  \leq \mathfrak{T}_{i,a_1,a_1} , \mathfrak{T}_{i,q,r}  \leq  \mathfrak{T}_{i,a_2,a_2},  \mathfrak{T}_{i,j,r} \leq \mathfrak{T}_{i, a_3, a_3},  \\\sum_{j,q=1}^{m,\mathcal{E}_{2}(j)} \mathcal{G}_{2}(j,q) = 1, \mathfrak{T}_i \in \mathbf{S}^{ \tilde{p}_2 + m }_+ \quad \forall i \in [i,n], j \in [1,m], \label{eqn:hds_2}\\
        (j,q) \in \mathcal{E}_{2} , (j,q,r) \in   \mathcal{E}_{3},(a_1, a_2, a_3) \in \Omega^k(j, q, r), \lambda_I > 0.
    \end{gather}
\end{subequations}

\subsubsection{Parameterisation with \y{pp}}
We follow four steps to solve \y{hnr} with \y{pp}: {\em i)} we do all the convex relaxations, reparameterisation and inclusion of \y{mdh} as in equation~(\ref{eqn:qnr_xyz_sol}), {\em ii)} we relax the exact rank-1 constraint on $\mathfrak{U}_i$ to trace minimisation, {\em iii)} we minimise area difference, as in equation~(\ref{eqn:hds_1}), and {\em iv)} we impose inequality constraints between diagonal and off-diagonal elements of $\mathfrak{U}_i$. We obtain the following solution:
\begin{subequations}\label{eqn:hnr_xyz_sol}
    \begin{gather}
        \min_{\{\mathfrak{U}_{i}\},\mathcal{G}_{2},\mathcal{G}_{3}}   \sum_{i=1}^n\bigg( \Big(f^{\Pi}_{i,j}(\mathfrak{S}_{i}^{'}) + \Tr(\mathfrak{U}_i)\Big) + \sum_{j,q=1}^{m,\mathcal{E}_{2}(j)} \Big( \lambda_I |\mathfrak{g}_I^{P}(i,j,q)  \notag \\  - \mathcal{G}_{2}(j,q)| -  f_{i,j}^{\delta}(\mathfrak{S}_{i}^{'})   +
        \sum_{r=1}^{\mathcal{E}_{3}(j,q)}  \lambda_E|\mathfrak{g}_E^{P}(i,j,q,r) - \mathcal{G}_{3}(j,q,r)|\Big) \bigg)
        \label{eqn:hxs_1}\\
        \text{s.t.~~} \mathfrak{U}_{i,j,q}^2 \leq \mathfrak{U}_{i,b_t,b_t}, \sum_{j,q=1}^{m,\mathcal{E}_{2}(j)} \mathcal{G}_{2}(j,q) = 1, \mathfrak{U}_{i,1,1} = 1,\notag \\  \mathfrak{U}_i \in \mathbf{S}^{ 6\tilde{p}_2 + 3m + 1}_+ \quad \forall \quad b_t \in \rho^k(j,q,r), t \in [1,18], i \in [i,n], \notag \\
        j \in [1,m], (j,q) \in \mathcal{E}_{2}, (j,q, r) \in \mathcal{E}_{3}, \lambda_I > 0, \label{eqn:hxs_3}
    \end{gather}
\end{subequations}
where $\mathfrak{S}_{i}^{'}$ is the leading $(m+1)\times(m+1)$ submatrix of $\mathfrak{U}_{i}$.

{\flushleft \textbf{Acceleration with sub-relaxations.}} Clearly, the matrix $\mathfrak{U}_{i}$ grows with the number of correspondences $m$ with $\mathcal{O}(m^3)$. Although much faster than standard \y{lmi} relaxations, $\mathfrak{U}_{i}$ slows the \y{sdp} solution considerably for \y{pp}. Hence, we base our real-world solution for \y{hnr} on equation~(\ref{eqn:hnr_dep_sol}) while in section~(4) of the supplementary materials, we suggest multiple acceleration strategies for speeding up the \y{sdp} solution of equation~(\ref{eqn:hnr_xyz_sol}) using edge based relaxation techniques \cite{wang2008further} and auxiliary variables, with small sacrifice in reconstruction accuracy.

%% file: sections/experiments.tex
\glsunset{s1} \glsunset{s2} \glsunset{q1} \glsunset{q2} \glsunset{h1} \glsunset{h2}
We describe our experimental results. For our synthetic data experiments, we use a randomised \y{isg} based on~\cite{perriollat2013computational}. We also propose a technique for extending \y{isg} to \y{qsg}. \y{isg} is associated with a scalar $x_{\sigma}$, the multiplier for additive zero mean Gaussian noise added to the point correspondences. \y{qsg} is associated with the scalar $\chi_{\mathrm{E}}$ which controls the mean of the uniform random distribution signifying the extent of `equiareality' in the data. The details of \y{isg}, \y{qsg} and $\chi_{\mathrm{e}}$ are given in section~5 of the supplementary materials.  For our real data experiments, we use the following \y{nrsfm} benchmarks: {\em i)} \textit{\y{wc}} and \textit{Hulk} datasets from~\cite{chhatkuli2014non}, and {\em ii)} \textit{\y{kp}} dataset from~\cite{varol2012constrained}. We also use three dataset containing stretchable objects, obtained from the authors of~\cite{casillas2019equiareal}: {\em i)} a \textit{\y{st}} dataset, {\em ii)} a \textit{\y{slg}}, and {\em iii)} a \textit{\y{sb}}. To validate the capability of our proposed \y{nrsfm} techniques to also handle rigid objects, we use the \textit{\y{mh}} from the VGG dataset\footnote{https://www.robots.ox.ac.uk/$\sim$vgg/data/mview}. We compare our methods against \cite{chhatkuli2014non}, \cite{chhatkuli2017inextensible}, \cite{parashar2017isometric}-G and \cite{parashar2017isometric}-IP (for the general and infinitesimal planarity models resp.), \cite{dai2014simple}, \cite{ji2017maximizing}, \cite{hamsici2012learning} and \cite{gotardo2011kernel}. We name our solutions from equation~(\ref{eqn:sds2_1})-(\ref{eqn:sds2_2}), (\ref{eqn:snr_xyz_sol}), (\ref{eqn:qnr_dep_sol}), (\ref{eqn:qnr_xyz_sol}), (\ref{eqn:hnr_dep_sol}) and (\ref{eqn:hnr_xyz_sol}) as \y{s1}, \y{s2}, \y{q1}, \y{q2}, \y{h1}, and \y{h2} (resp.). We use \y{rms} and \y{med} to quantify the error between reconstruction and \y{gt}. As with all monocular techniques, our reconstruction is up to scale, which is recovered a posteriori by comparison with \y{gt}. 
All our solutions are implemented using CVX \cite{gb08} with MOSEK solver \cite{mosek} in Matlab.

{\flushleft \textbf{Strict and quasi-isometric \y{nrsfm}.}}
By sharing the same graph structure between $\mathcal{E}_2$ and \y{nng_} for  \cite{chhatkuli2017inextensible} and \cite{ji2017maximizing}, such that $|\mathcal{E}_2| = k\mathrm{N}$, the number of nearest neighbours in \y{nng_}, we demonstrate the accuracy of our methods w.r.t increasing $|\mathcal{E}_2|$ in figure~\ref{fig_q1}a. For this, we chose repeated Monte Carlo cross validation by randomly sampling 30\% correspondences from \y{st} while varying $|\mathcal{E}_2|$. At low $|\mathcal{E}_2|$ ($<5$), our proposed methods \y{s1}, \y{s2}, \y{q1} and \y{q2} are significantly better than \cite{chhatkuli2017inextensible} and \cite{ji2017maximizing}, but the gap in accuracy diminishes with increasing neighbours. Figure~\ref{fig_q1}b shows \y{rms} with increasing noise on correspondences. This is done by repeatedly drawing 25\% random correspondences from \y{wc} and by sharing the same connectivity of $|\mathcal{E}_2| = k\mathrm{N} = 7$. All the methods show comparable accuracy within a tolerance, except \cite{ji2017maximizing}, which is slightly worse. Given simulated data from \y{isg}, we demonstrate the stability of \y{s1}, \y{s2}, \y{q1} and \y{q2} against the closest zeroth-order \y{nrsfm} methods \cite{chhatkuli2017inextensible} and \cite{ji2017maximizing} with increasing number of point correspondences in figure~\ref{fig_q1}c. This is done by sharing the same connectivity of $|\mathcal{E}_2| = k\mathrm{N} = 5$. All the methods perform comparably but \cite{ji2017maximizing} is slightly worse. Our \y{s2} and \y{q2} shows slight dip in accuracy with higher correspondences. This is unsurprising, since increasing correspondences without adjusting $\mathcal{E}_2$ increases the sparsity of $\{\mathfrak{R}_i\}$ and $\{\mathfrak{S}_i^{'}\}$ and can be easily compensated back by expanding $\mathcal{E}_2$, as evident from figure~\ref{fig_q1}a. To demonstrate the efficacy of our \textit{correspondence completion} method based on equation~(\ref{eq_corr_compl}), we repeatedly draw 60\% correspondences randomly from \y{kp} and log the accuracy for observed and completed missing points by artificially `hiding' up to $\sim$20\% correspondence points. This is done with a \y{pn} of $s = 3$. The results are shown in figure~\ref{fig_q1}d. The accuracy of correspondence completion remains stable within this range. 

We compare our proposed methods with many state-of-the-art methods on hulk, \y{kp}, \y{wc} and \y{st} datasets in table~\ref{tab_iso_results}, all values in $mm$. \y{q2} is the most accurate method for hulk and \y{kp}, while \y{q2} is 1.59 $mm$ behind \cite{chhatkuli2017inextensible} in \y{rms} on the \y{wc} dataset and \y{lrsb} methods \cite{hamsici2012learning} and \cite{dai2014simple} beats \y{q1} by 0.84 and 1.17 $mm$ (resp.) of \y{rms} in \y{st}, which are very small differences in accuracy. Some qualitative plots are shown in figure~(\ref{fig_iso_quals}). Notably, in the highly stretched dataset of \y{st}, all our proposed methods on an average perform 42.8\% and 64.9\% better than comparable zeroth-order methods \cite{ji2017maximizing} and \cite{chhatkuli2017inextensible}, our methods being almost comparable to \y{lrsb} in accuracy, while also retaining the same accuracy in the perfectly isometric datasets of \y{wc}, \y{kp} and hulk, unlike the \y{lrsb} methods. We also demonstrate accurate reconstruction of rigid objects with the \y{mh} dataset. We create semi-synthetic data by randomly roto-translating the \y{gt} and projecting with arbitrary intrinsics. The results from repeated experiments are given in figure~(\ref{fig:mh_randomized}), where the strictly isometric methods show accuracies as good as \y{lrsb} methods. 

{\flushleft \textbf{Quasi-equiareal \y{nrsfm}.}} We define two error terms: $g\mathrm{E} = \frac{1}{n|\mathcal{E}_2|} \sum_{i,j,q=1}^{n,m,\mathcal{E}_2(j)} | \mathcal{G}_2(j,q) - \mathfrak{g}_I(i,j,q)|$ denoting deviation from isometry and $a\mathrm{E} = \frac{1}{n|\mathcal{E}_3|} \sum_{i,j,q,r=1}^{n,m,\mathcal{E}_2(j),\mathcal{E}_3(j,q)} | \mathcal{G}_3(j, q, r) - \mathfrak{g}_E(i,j,q,r)|$ denoting deviation from equiareality, where $\mathfrak{g}_I$ and $\mathfrak{g}_E$ are computed geometrically in $\mathbb{R}^3$ from the unscaled result of \y{h1}. In figure~\ref{fig_q1}e, we show an important result: upon repeated experiments on data from \y{qsg} with suitable $\lambda_I, \lambda_E$, we manage to control $g\mathrm{E}$ independently of $a\mathrm{E}$ in \y{h1}. Increasing $\lambda_I$ reduces both $g\mathrm{E}$ and $a\mathrm{E}$, as physics mandates. But increasing $\lambda_E$ forces \y{h1} to find a minimal $a\mathrm{E}$ solution without affecting $g\mathrm{E}$. To demonstrate how this affects \y{rms}, we demonstrate synthetic results in figure~\ref{fig_q1}f. Using \y{qsg} with increasing equiareality $\chi_{\mathrm{E}}$, we test the \y{rms} for \y{q1}, \y{h1}, \cite{ji2017maximizing} and \cite{chhatkuli2017inextensible} repeatedly ($\sim$200 times). Between \y{q1} and \y{h1}, $\lambda_I$ is identical. All four methods share same $\mathcal{E}_2$ or \y{nng_}, $\mathcal{E}_3$ is derived from the edges of $\mathcal{E}_2$. At low $\chi_{\mathrm{E}}$, \y{q1} is better than \y{h1} and similar to \cite{chhatkuli2017inextensible}. With increasing $\chi_{\mathrm{E}}$, \y{q1} becomes increasingly better than \cite{chhatkuli2017inextensible} while \y{h1} crosses over at $\chi_{\mathrm{E}} \approx 0.3$ and keeps performing better than \y{q1}. Clearly, for equiareal deformations, \y{q1} performs better than \cite{ji2017maximizing} and \cite{chhatkuli2017inextensible} while \y{h1} improves upon this accuracy and remains the most accurate method for all equiareal deformations. 
\setlength{\tabcolsep}{3.0pt}
\renewcommand{\arraystretch}{1.2}
\begin{table*}[]
\caption{Comparison of \y{s1}, \y{s2}, \y{q1} and \y{q2} with baseline methods}
\scriptsize
\begin{adjustwidth}{-0.3cm}{-0.12cm}
\begin{tabular}{ccccccccccccccccccccccccc} \toprule
& \multicolumn{2}{c}{\textbf{\y{s1}}}& \multicolumn{2}{c}{\textbf{\y{s2}}} &
\multicolumn{2}{c}{\textbf{\y{q1}}} & \multicolumn{2}{c}{\textbf{\y{q2}}} &
\multicolumn{2}{c}{\textbf{\cite{ji2017maximizing}}} &
\multicolumn{2}{c}{\textbf{\cite{chhatkuli2017inextensible}}} &
\multicolumn{2}{c}{\textbf{\cite{parashar2017isometric}-IP}} &
\multicolumn{2}{c}{\textbf{\cite{parashar2017isometric}-G}}  &
\multicolumn{2}{c}{\textbf{\cite{hamsici2012learning}}} &
\multicolumn{2}{c}{\textbf{\cite{gotardo2011kernel}}} &
\multicolumn{2}{c}{\textbf{\cite{dai2014simple}}} &
\multicolumn{2}{c}{\textbf{\cite{chhatkuli2014non}}}\\ \cline{2-25} 
& \multicolumn{1}{c}{\tiny{\y{rms}}} & \multicolumn{1}{c}{\tiny{\y{med}}} & \multicolumn{1}{c}{\tiny{\y{rms}}} & \multicolumn{1}{c}{\tiny{\y{med}}} & \multicolumn{1}{c}{\tiny{\y{rms}}} & \multicolumn{1}{c}{\tiny{\y{med}}} & \multicolumn{1}{c}{\tiny{\y{rms}}} & \multicolumn{1}{c}{\tiny{\y{med}}} & \multicolumn{1}{c}{\tiny{\y{rms}}} & \multicolumn{1}{c}{\tiny{\y{med}}} & \multicolumn{1}{c}{\tiny{\y{rms}}} & \multicolumn{1}{c}{\tiny{\y{med}}} & \multicolumn{1}{c}{\tiny{\y{rms}}} & \multicolumn{1}{c}{\tiny{\y{med}}} & \multicolumn{1}{c}{\tiny{\y{rms}}} & \multicolumn{1}{c}{\tiny{\y{med}}} & \multicolumn{1}{c}{\tiny{\y{rms}}} & \multicolumn{1}{c}{\tiny{\y{med}}} & \multicolumn{1}{c}{\tiny{\y{rms}}} & \multicolumn{1}{c}{\tiny{\y{med}}} & \multicolumn{1}{c}{\tiny{\y{rms}}} & \multicolumn{1}{c}{\tiny{\y{med}}} & \multicolumn{1}{c}{\tiny{\y{rms}}} & \multicolumn{1}{c}{\tiny{\y{med}}}  \\ \cline{2-25} 
\multicolumn{1}{c}{\textbf{Hulk}} & 2.49 & 1.94 & \textbf{2.19} & 1.68 & 2.53 & 1.98 & \textbf{2.19} & 1.68 & 2.58 & 2.10 & 2.53 & 1.87 & 11.10 & 9.64 & 11.25 & 9.81 & 19.34 & 18.33 & 34.61 & 32.18 & 16.59 & 15.00 & 16.72 & 14.93 \\ \cline{2-25} 
\multicolumn{1}{c}{\textbf{\y{kp}}}& 4.76 & 4.17 & 3.93 & 3.53 & 4.53 & 3.98 & \textbf{3.85} & 3.47 & 9.06 & 7.57 & 4.77 & 4.25 & 11.30 & 10.12 & 11.34 & 10.14 & 13.25 & 11.87 & 30.61 & 28.40 & 13.40 & 12.12 & 13.88 & 12.12 \\ \cline{2-25} 
\multicolumn{1}{c}{\textbf{\y{wc}}}& 9.39 & 7.54 & 7.45 & 5.97 & 8.78 & 7.11 & 7.39 & 5.92 & \textbf{5.8} & 4.57 & 9.72 & 7.88 & 21.45 & 18.69 & 21.54 & 18.79 & 66.28 & 62.78 & 106.89 & 95.24 & 25.7 & 22.91 & 26.54 & 23.38  \\ \cline{2-25} 
\multicolumn{1}{c}{\textbf{\y{st}}}& 8.65 & 6.79 & 8.64 & 6.68 & 8.53 & 6.71 & 8.61 & 6.66 & 24.54 & 20.71 & 15.05 & 12.39 & 11.97 & 10.85 & 11.82 & 10.72 & 7.69 & 7.14 & 13.41 & 12.01 & \textbf{7.36} & 6.79 & 9.30 & 8.69\\ \bottomrule
\label{tab_iso_results}
\end{tabular}
\end{adjustwidth}
\vspace{-0.5cm}
\end{table*}
\setlength{\tabcolsep}{2.8pt}
\renewcommand{\arraystretch}{1.2}
\begin{table}[]
\caption{\y{q1}, \y{q2} and \y{h1} against baseline methods (in $mm$)}
\small
\begin{tabular}{lccccccccc}
\toprule
\multicolumn{2}{l}{}  & \textbf{\cite{ji2017maximizing}} & \textbf{\cite{chhatkuli2017inextensible}} & \textbf{\cite{gotardo2011kernel}} & \textbf{\cite{hamsici2012learning}} & \textbf{\cite{chhatkuli2014non}} & \textbf{\y{q1}} & \textbf{\y{q2}} & \textbf{\y{h1}} \\ \cline{2-10} 
\multirow{2}{*}{\textbf{\y{sb}}}  & \textbf{\y{rms}} & 27.36  & 13.59  & 18.66  & 12.66  & 12.11 & 13.53 & 12.48 & \textbf{8.72} \\ \cline{2-10} 
 & \textbf{\y{med}} & 21.07  & 11.47  & 17.63  & 11.36  & 10.95 & 10.09 & 9.92  & \textbf{7.19} \\ \hline
\multirow{2}{*}{\textbf{\y{slg}}} & \textbf{\y{rms}} & 34.43  & 10.78  & x  & 8.92 & 10.41 & 8.81  & 8.67  & \textbf{6.43} \\ \cline{2-10} 
 & \textbf{\y{med}} & 31.53  & 8.65 & x  & 8.00 & 9.36  & 7.70  & 7.55  & \textbf{5.59} \\ \bottomrule
\end{tabular}
\label{tab_eq_comp}
\vspace{-3mm}
\end{table}
\begin{figure*}\centering
\subcaptionbox*{(a)}[4.4cm]{%
\begin{overpic}[height = 2.75cm, trim=0 0 0 0,clip]{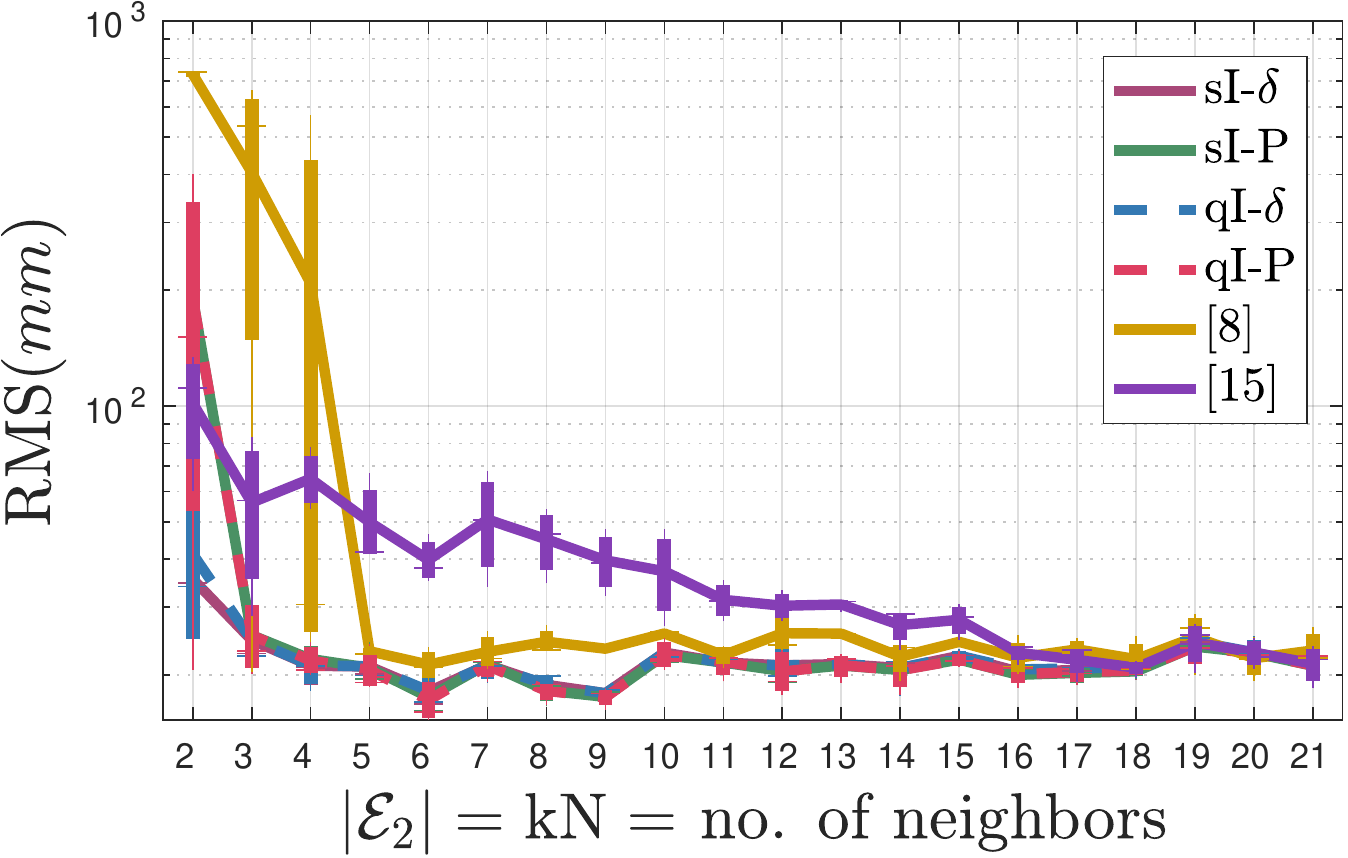}
\end{overpic}
}
\subcaptionbox*{(b)}[4.2cm]{%
\begin{overpic}[height = 2.75cm, trim=0 0 0 0,clip]{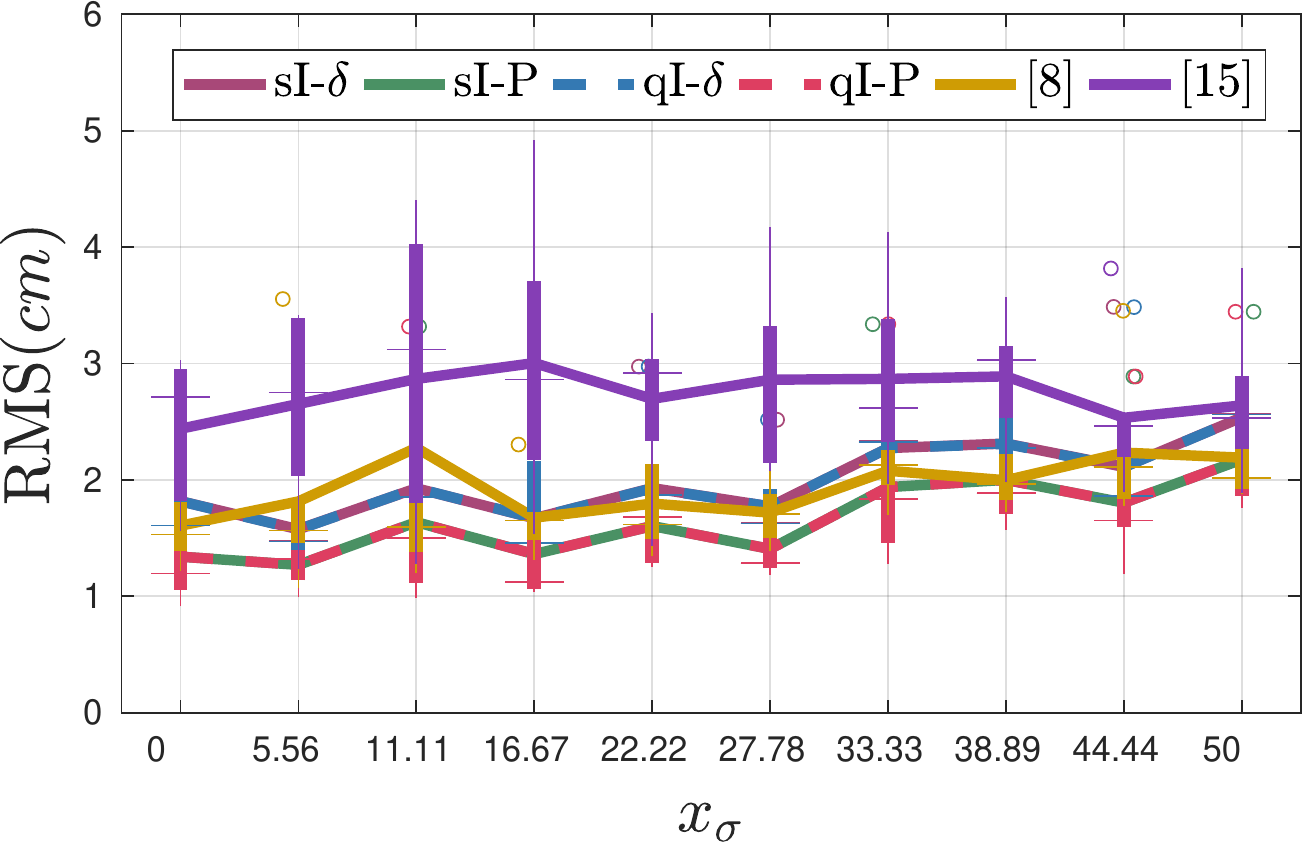}
\end{overpic}
}
\subcaptionbox*{(c)}[4.2cm]{%
\begin{overpic}[height = 2.75cm, trim=0 0 0 0,clip]{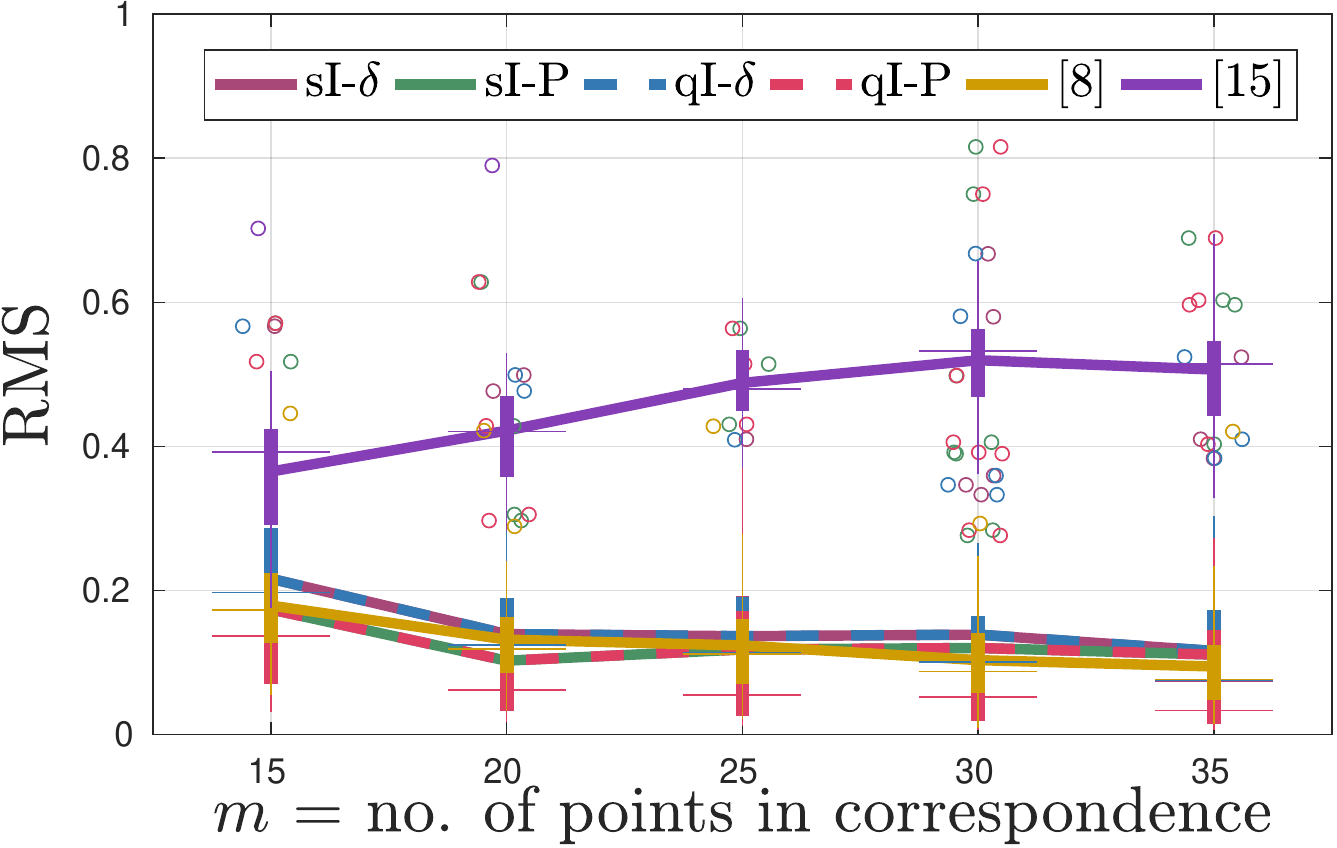}
\end{overpic}
}
\subcaptionbox*{(d)}[4.4cm]{%
\begin{overpic}[height = 2.75cm, trim=0 0 0 0,clip]{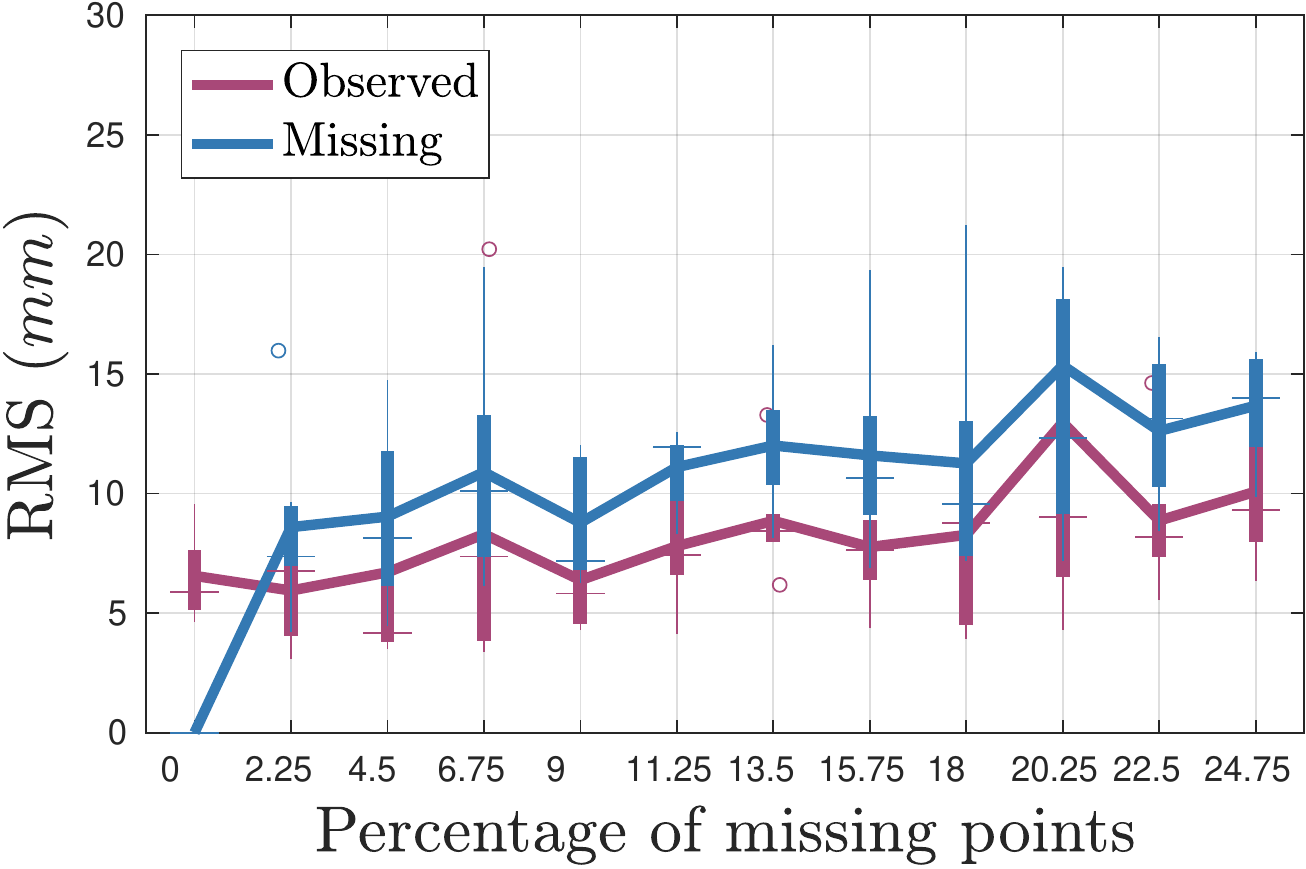}
\end{overpic}
}
\subcaptionbox*{(e)}[6cm]{%
\begin{overpic}[height = 2.75cm, trim=0 0 0 0,clip]{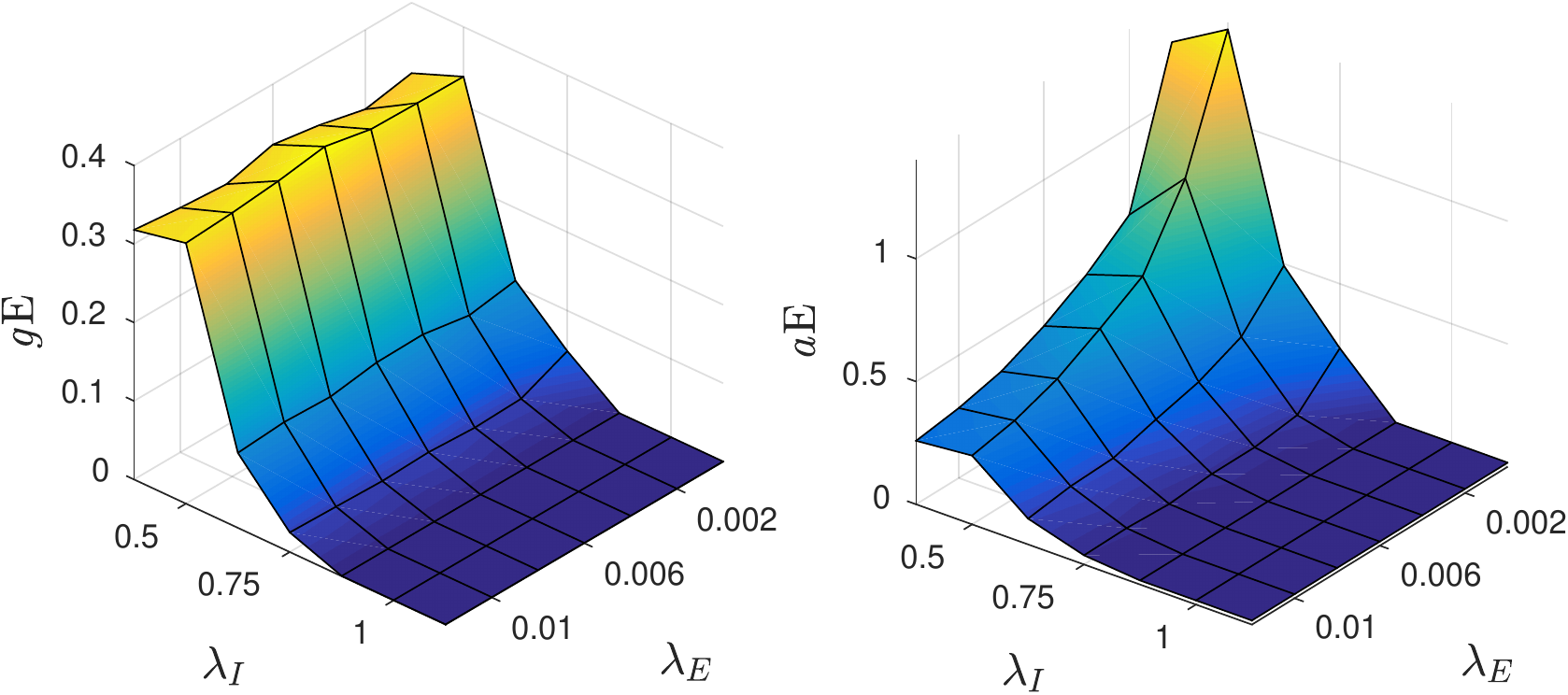}
\end{overpic}
}
\subcaptionbox*{(f)}[3.0cm]{%
\begin{overpic}[height = 2.75cm, trim=0 0 0 0,clip]{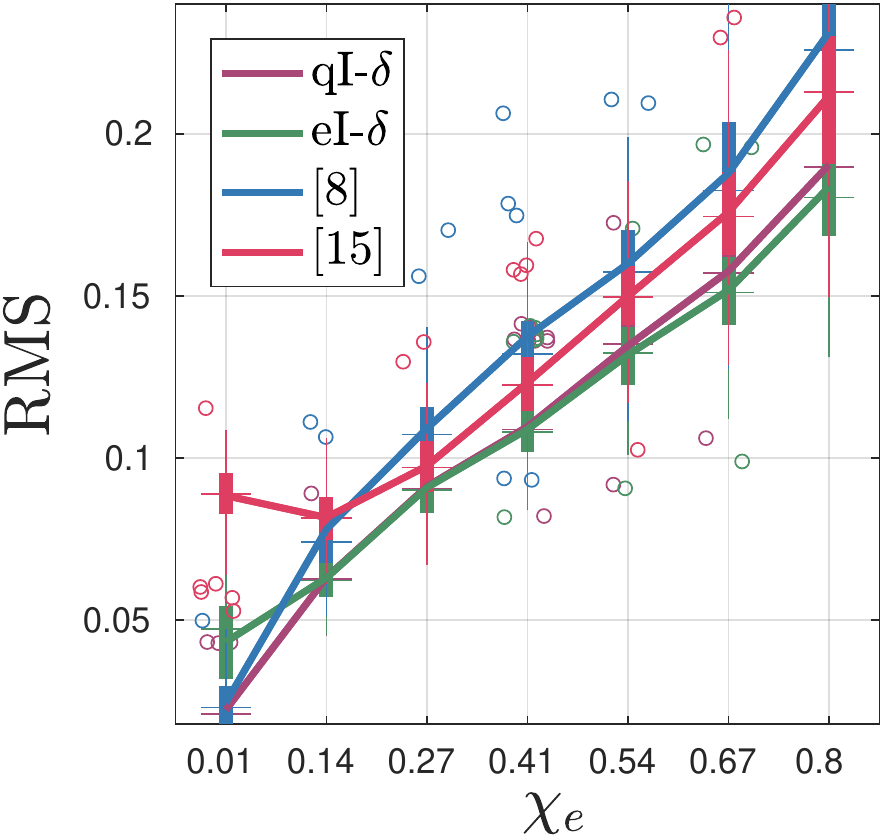}
\end{overpic}
}
\subcaptionbox*{(g)}[4.1cm]{%
\begin{overpic}[height = 2.75cm, trim=0 0 0 0,clip]{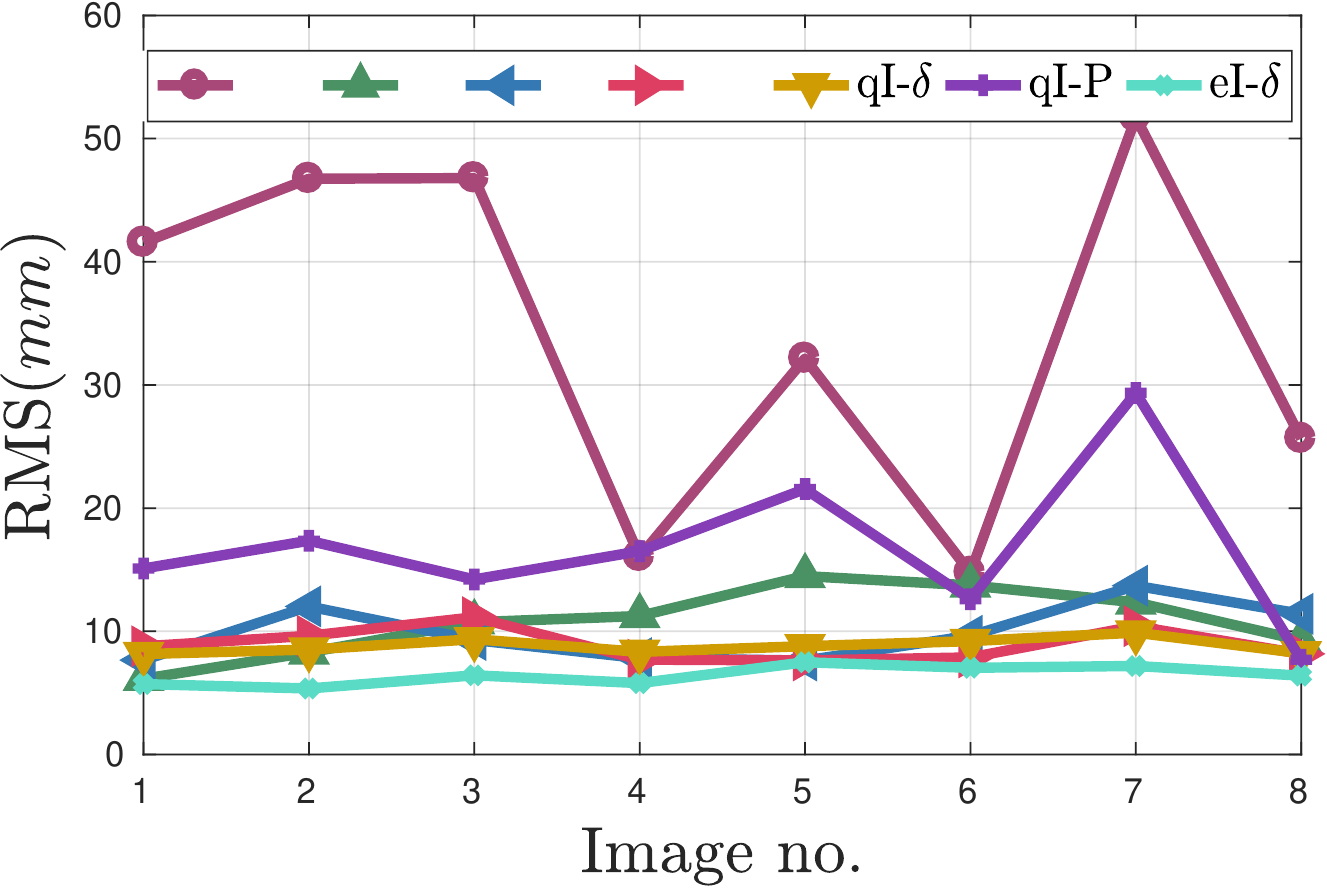}
\put(20,69.8){\scalebox{.8}{\tiny{\cite{ji2017maximizing}}}}
\put(34,69.8){\scalebox{.8}{\tiny{\cite{chhatkuli2017inextensible}}}}
\put(47,69.8){\scalebox{.8}{\tiny{\cite{chhatkuli2014non}}}}
\put(60.5,69.8){\scalebox{.8}{\tiny{\cite{hamsici2012learning}}}}
\end{overpic}
}
\subcaptionbox*{(h)}[4.0cm]{%
\begin{overpic}[height = 2.75cm, trim=0 0 0 0,clip]{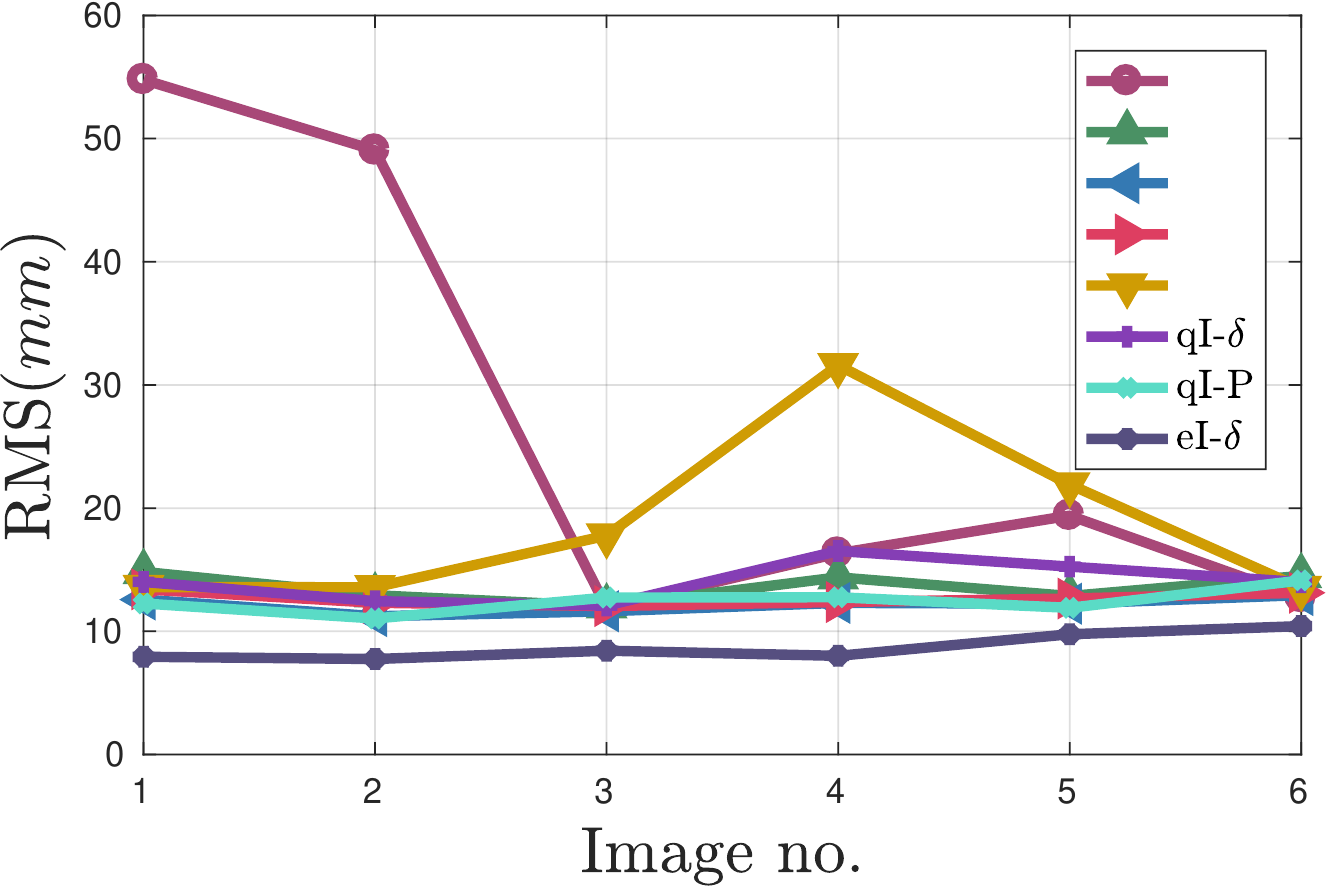}
\put(102.5,70){\scalebox{.8}{\tiny{\cite{ji2017maximizing}}}}
\put(102.5,66){\scalebox{.8}{\tiny{\cite{chhatkuli2017inextensible}}}}
\put(102.5,61){\scalebox{.8}{\tiny{\cite{chhatkuli2014non}}}}
\put(102.5,57){\scalebox{.8}{\tiny{\cite{hamsici2012learning}}}}
\put(102.5,52){\scalebox{.8}{\tiny{\cite{gotardo2011kernel}}}}
\end{overpic}
}
\vspace{-4mm}
\caption[Fig1]{We show the following: \textbf{(a)} the variation of \y{rms} w.r.t increasing neighbors in $\mathcal{E}_2$, \textbf{(b)} the variation of \y{rms} w.r.t additive pixel noise to correspondences, \textbf{(c)} the variation of \y{rms}, in \y{au}, w.r.t increasing $m$ for same $\mathcal{E}_2$ (or \y{nng_}) for all compared methods, \textbf{(d)} the variation of \y{rms} in observed and missing correspondences completed by \y{q2}, \textbf{(e)} the variation of $g\mathrm{E}$ and $a\mathrm{E}$ w.r.t changes in $\lambda_I$ and $\lambda_E$, \textbf{(f)} the variation of \y{rms} in \y{au} w.r.t increasing $\chi_{\mathrm{E}}$ in \y{qsg}, \textbf{(g)} and \textbf{(h)} \y{rms} for each frame from \y{slg} and \y{sb} respectively.}
\setlength{\belowcaptionskip}{-15pt}
\label{fig_q1}
\end{figure*}
\begin{figure}
    \centering
    \begin{overpic}[width = 6.5cm, trim=80 320 100 323,clip]{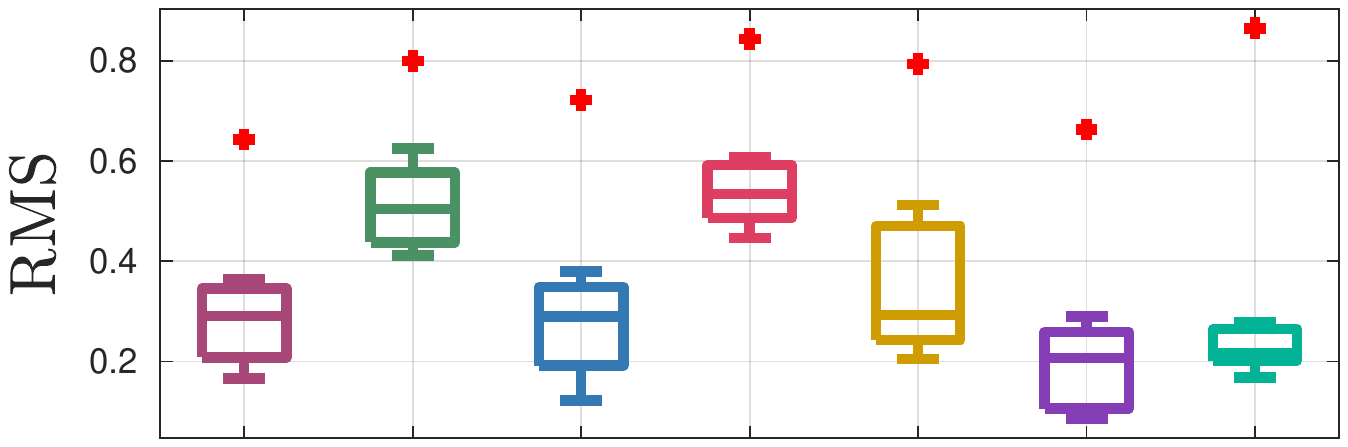}
    \put(32,0){\scalebox{.8}{\y{s1}}}
    \put(52,0){\scalebox{.8}{\y{q1}}}
    \put(72,0){\scalebox{.8}{\y{s2}}}
    \put(92,0){\scalebox{.8}{\y{q2}}}
    \put(113,0){\scalebox{.8}{\cite{chhatkuli2017inextensible}}}
    \put(132,0){\scalebox{.8}{\cite{ji2017maximizing}}}
    \put(157,0){\scalebox{.8}{\cite{dai2014simple}}}
\end{overpic}
    \caption{Results from \y{mh} rigid dataset.}
    \label{fig:mh_randomized}
    \vspace{-5mm}
\end{figure}

On real data, our experiments on \y{wc}, \y{kp}, hulk or \y{st} did not yield better results than \y{q1} or \y{q2}, unsurprisingly. Another observation (with details in section~6 of supplementary) is that with denser correspondences, we have smaller simplices and hence isometry dominates over equiareality. Hence, from the data of \cite{casillas2019equiareal}, we select two highly stretchable/squeezable objects \y{sb} and \y{sl}, and obtain sparse and randomly sampled tracks of 11 correspondences. The results for two such tracks (to be made publicly available with the code) are demonstrated in table~\ref{tab_eq_comp}, compared against all the methods that produced an output, and the accuracy of each frames are shown in figure~\ref{fig_q1}g and~\ref{fig_q1}h. \y{h1} is significantly ahead of all compared methods. Even with denser correspondences, the accuracy from these compared methods do not reach the level of \y{h1}. Check section~6 of supplementary for validation across varying resolutions for \y{sb} and \y{sl}, as well as extended analysis of all our experiments. 
\setlength{\belowcaptionskip}{-17pt}
\begin{figure}
    \centering
    \begin{overpic}[width=4.0cm, trim=0 35 260 28,clip]{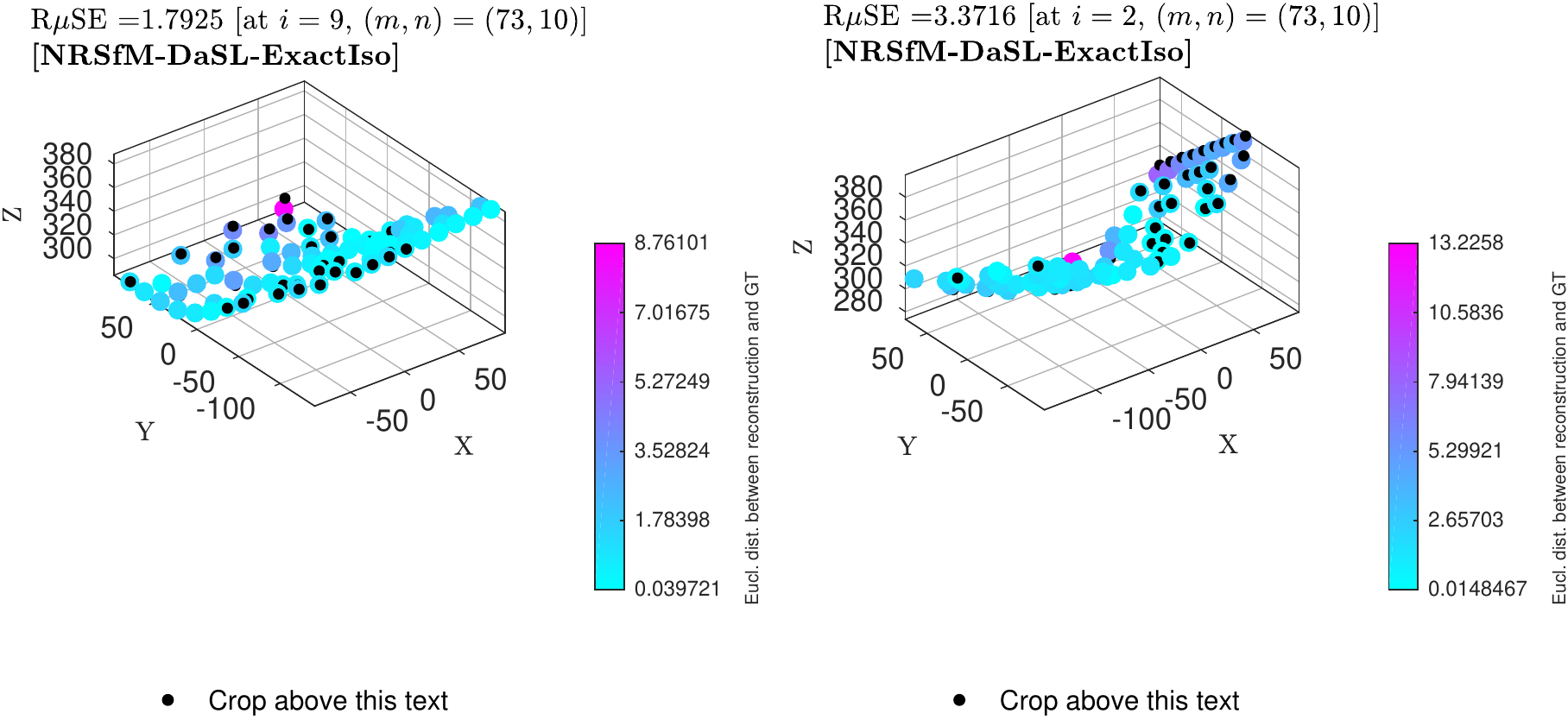}
    \put(20,15){\scalebox{.6}{Hulk ($i=9$)}}
    \end{overpic}
    \begin{overpic}[width=4.15cm, trim=0 35 260 27,clip]{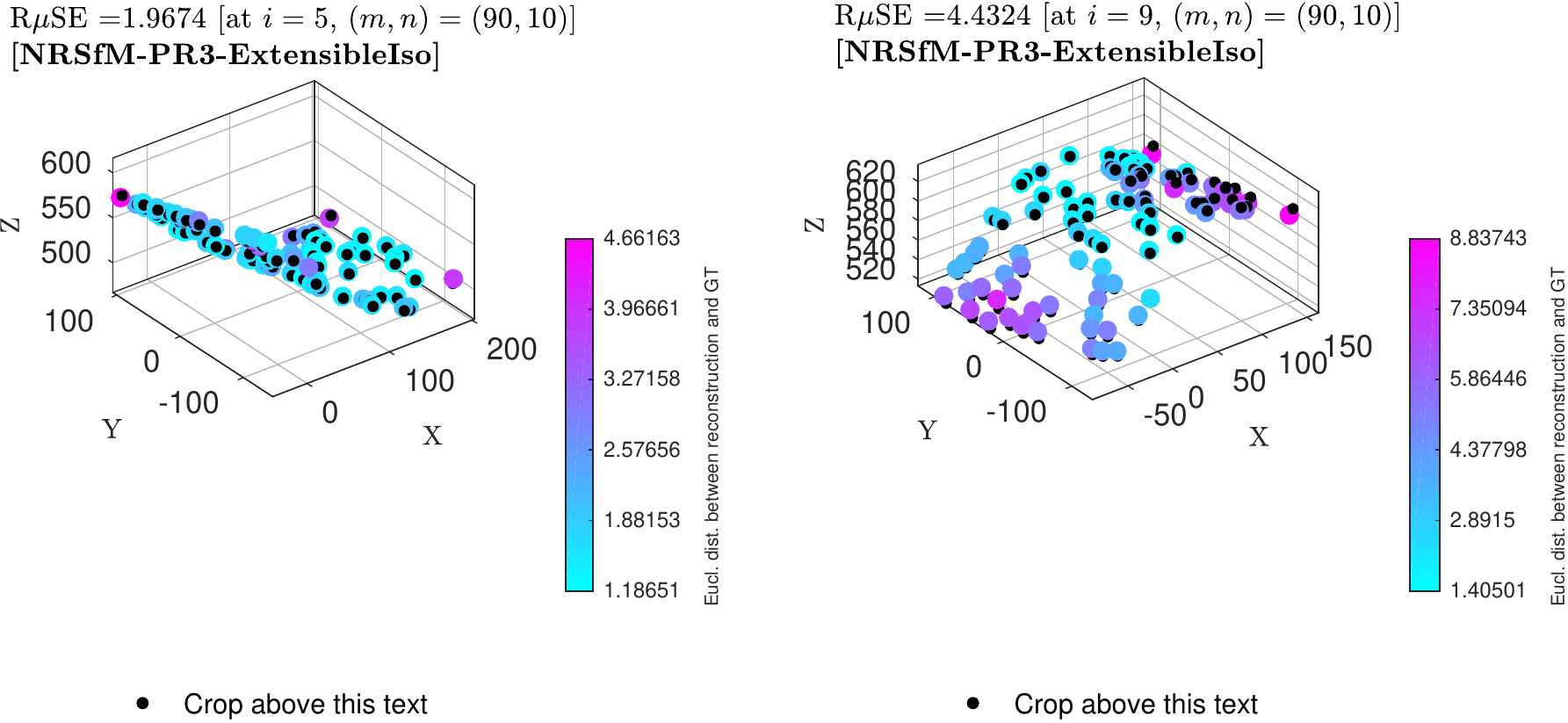}
    \put(23,20){\scalebox{.6}{\y{kp} ($i=5$)}}
    \end{overpic}
    \begin{overpic}[width=4.0cm, trim=0 35 260 26,clip]{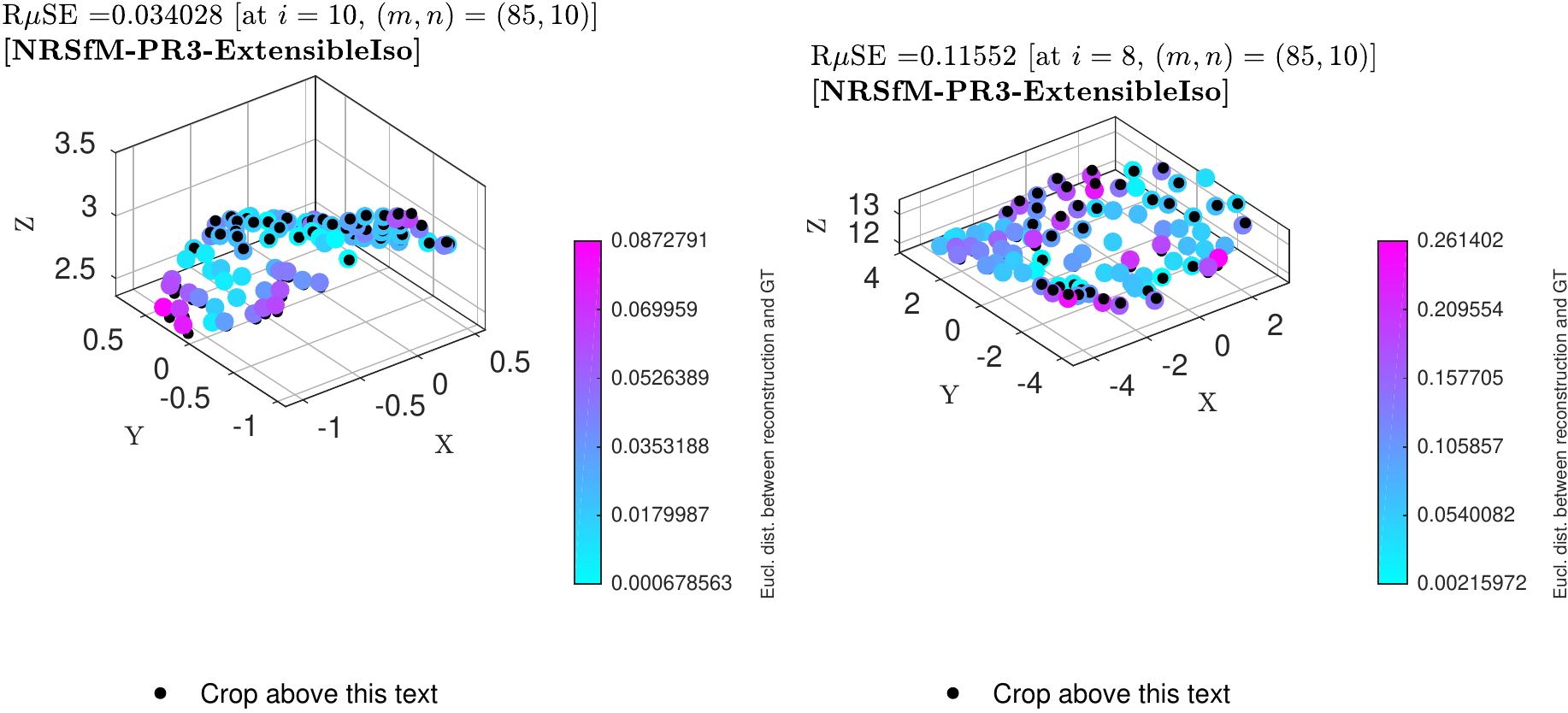}
    \put(23,15){\scalebox{.6}{\y{wc} ($i=10$)}}
    \end{overpic}
    \begin{overpic}[width=4.15cm, trim=0 35 260 36,clip]{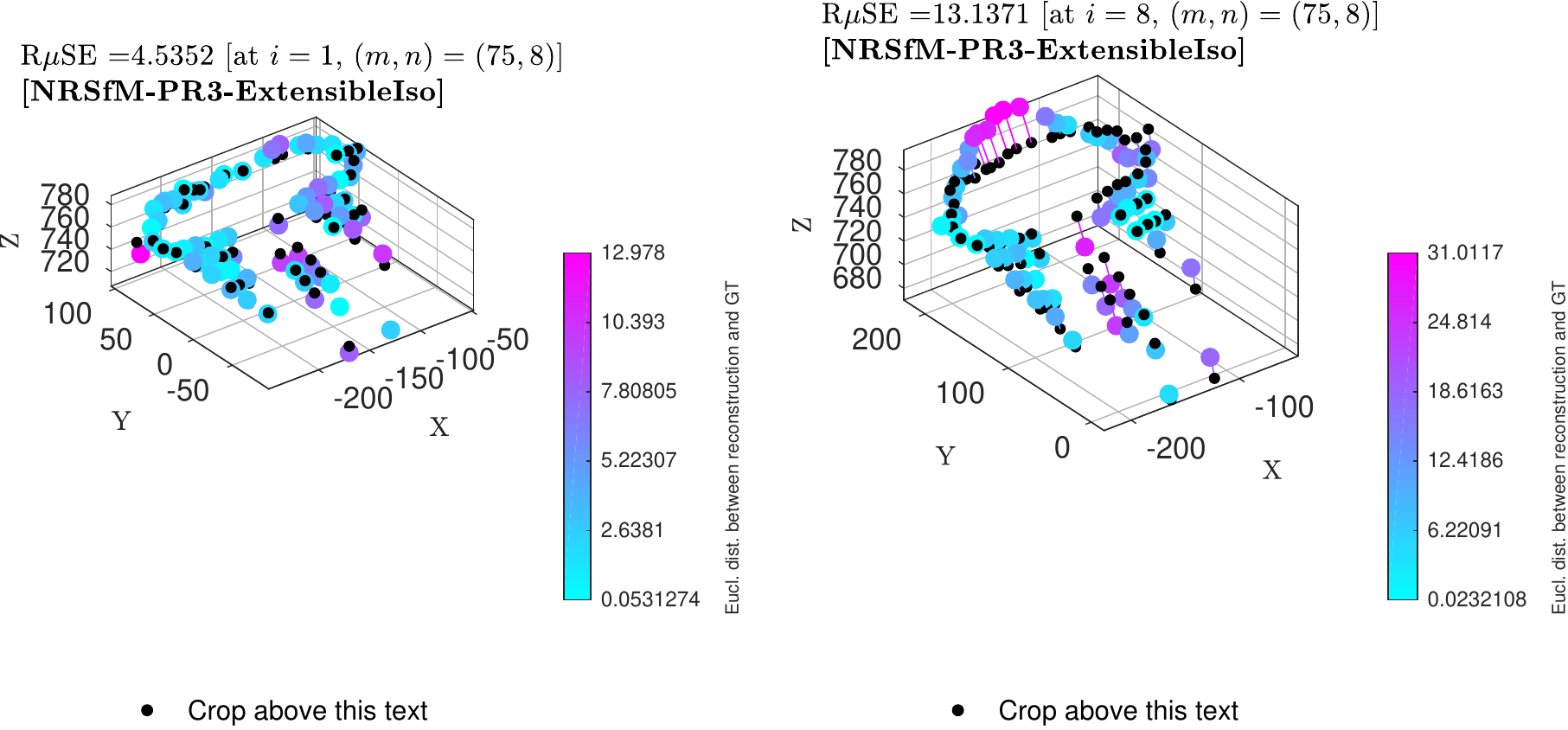}
    \put(23,20){\scalebox{.6}{\y{st} ($i=1$)}}
    \end{overpic}
    \caption{Qualitative results from \y{s2}. Black points are \y{gt}.}
    \label{fig_iso_quals}
\end{figure}
\vspace{-3mm}

In \cref{fig:colonic_real}, we show a sample reconstruction of a colonoscopic sequence from the EndoMapper dataset \cite{azagra2023endomapper} obtained by following the approach of \y{h1}. In the absence of \y{gt}, only the qualitative results have been shown.

%% file: sections/conclusion.tex
\vspace{-2mm}
We have presented a convex formulation for strict and quasi-isometric \y{nrsfm}, as well as extended the method to handle quasi-equiareality. Our solution for \y{snr} and \y{qnr} solves the restrictive isometry problem without utilising inextensible relaxations, which is a first for zeroth order \y{nrsfm} methods. Our solution for \y{hnr} presents the first method handling quasi-equiareality convexly, searching for nearly equiareal solution in the vicinity of the solution space of quasi-isometry. We demonstrate that \y{snr} and \y{qnr} achieve state-of-the-art accuracy in benchmark datasets and are more generalisable to extensible surfaces than comparable methods. We demonstrate that for highly extensible surfaces, our solution to \y{hnr} produces reconstruction of unmatched accuracy.

\begin{figure*}
    \centering    
    \includegraphics[height=3.6cm]{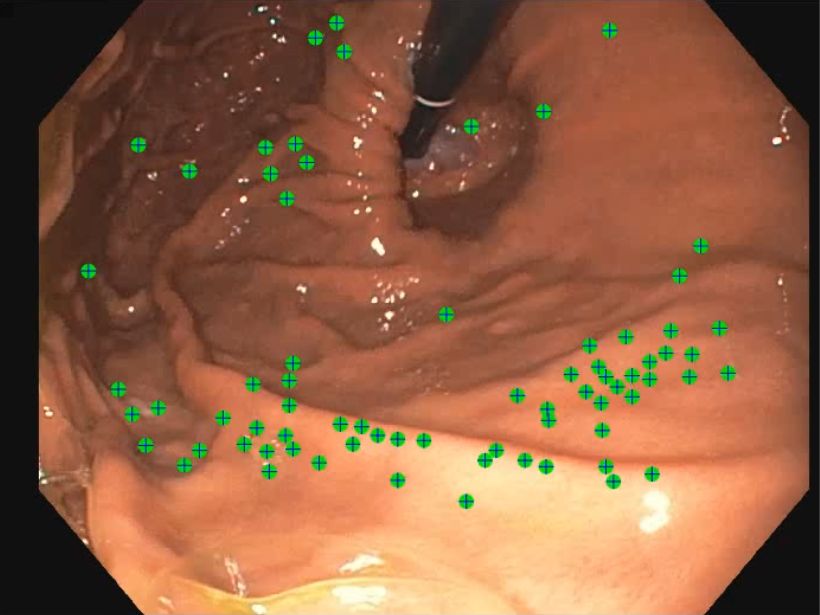}   
\includegraphics[height=3.6cm]{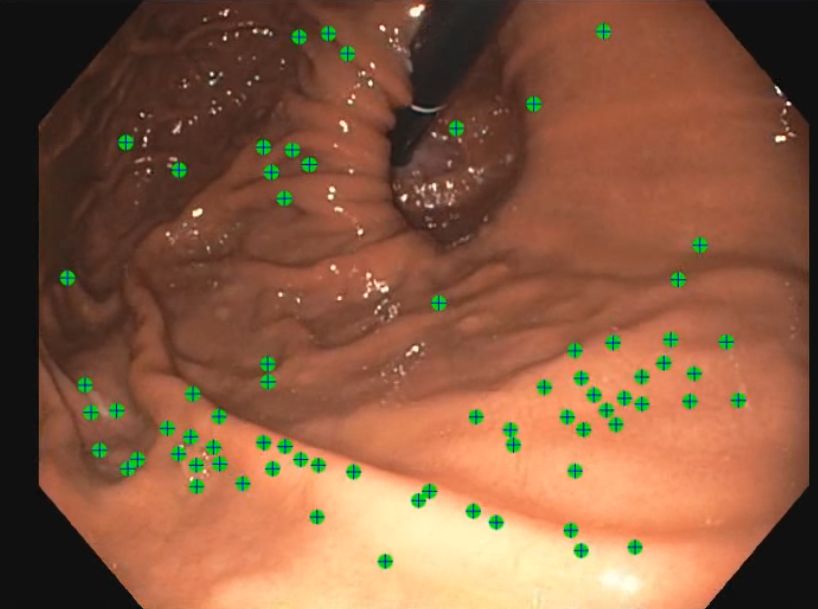}
\includegraphics[height=3.6cm]{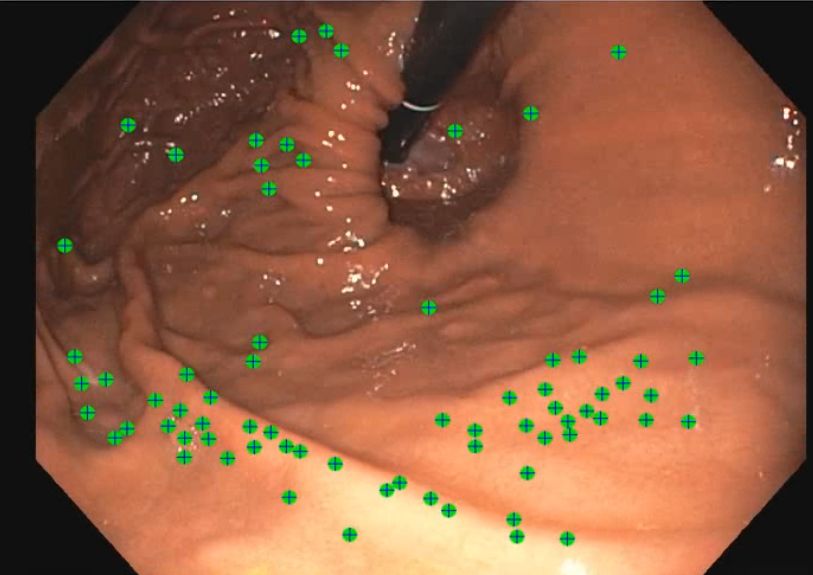}\\
    \includegraphics[height=3cm]{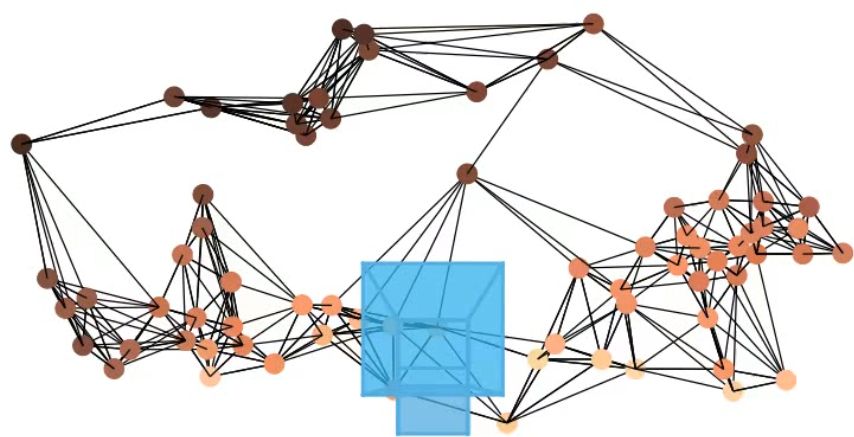}   
    \includegraphics[height=3cm]{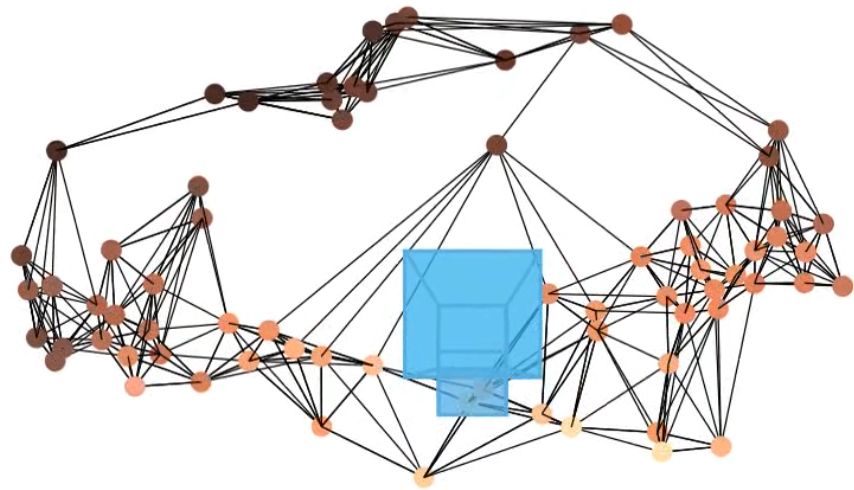}   
    \includegraphics[height=3cm]{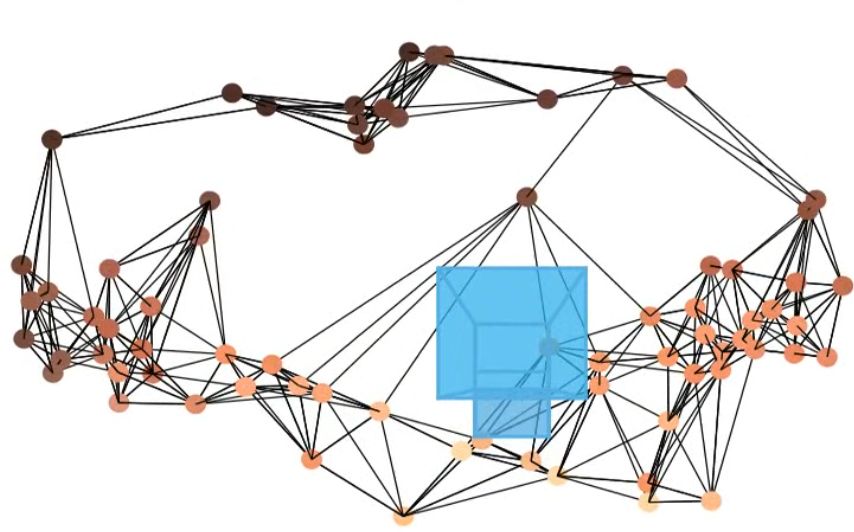}   
    \caption{Non-isometric NRSfM applied to a colonoscopy sequence of 67 images with 72 point tracks. The camera is shown in blue and the reconstructed points are colour-coded according to depth.}
    \label{fig:colonic_real}
\end{figure*}